\documentclass[sigconf]{acmart}

\AtBeginDocument{%
  }

\copyrightyear{2024}
\acmYear{2024}
\setcopyright{acmlicensed}
\acmConference[KDD '24] {Proceedings of the 30th ACM SIGKDD Conference on Knowledge Discovery and Data Mining }{August 25--29, 2024}{Barcelona, Spain.}
\acmBooktitle{Proceedings of the 30th ACM SIGKDD Conference on Knowledge Discovery and Data Mining (KDD '24), August 25--29, 2024, Barcelona, Spain}
\acmISBN{979-8-4007-0490-1/24/08}
\acmDOI{10.1145/XXXXXX.XXXXXX}

\usepackage{algorithmicx}
\usepackage[export]{adjustbox}
\usepackage[outdir=./]{epstopdf}
\usepackage{graphicx}
\usepackage{enumerate}
\usepackage{textcomp}
\usepackage{url}
\usepackage{balance}
\usepackage{mathrsfs}
\usepackage{multirow}

\usepackage{amsthm}
\usepackage{color,soul,xspace}

\theoremstyle{definition}

\usepackage{algorithm} 
\usepackage{algpseudocode}

\usepackage{amsmath,amssymb,amsfonts}

\newcommand{\inner}[2]{\langle{#1}, {#2}\rangle_{N, M}}
\def\E{\mathbb{E}}
\def\R{\mathcal{R}}
\def\Normal{\mathcal{N}}

\usepackage[tight,footnotesize]{subfigure}
\usepackage{caption}

\newcommand{\vsa}{\vspace*{-0.28cm}}
\newcommand{\vsb}{\vspace*{-0.19cm}}

\newcommand{\mourmeth}{\text{LRMDS}}
\newcommand{\ourmeth}{$\mourmeth$\xspace}

\begin{document}

\title{Low Rank Multi-Dictionary Selection at Scale}


\author{Boya Ma}
\affiliation{%
  \institution{University at Albany, State University of New York}
  \department{Department of Computer Science}
  \streetaddress{1400 Washington Ave}
  \city{Albany}
  \state{New York}
  \country{USA}
  \postcode{12203}}
\email{bma@albany.edu}

\author{Maxwell McNeil}
\affiliation{%
  \institution{University at Albany, State University of New York}
  \department{Department of Computer Science}
  \streetaddress{1400 Washington Ave}
  \city{Albany}
  \state{New York}
  \country{USA}
  \postcode{12203}}
\email{mmcneil2@albany.edu}

\author{Abram Magner}
\affiliation{%
  \institution{University at Albany, State University of New York}
  \department{Department of Computer Science}
  \streetaddress{1400 Washington Ave}
  \city{Albany}
  \state{New York}
  \country{USA}
  \postcode{12203}}
  \email{amagner@albany.edu}

\author{Petko Bogdanov}
\affiliation{%
  \institution{University at Albany, State University of New York}
  \department{Department of Computer Science}
  \streetaddress{1400 Washington Ave}
  \city{Albany}
  \state{New York}
  \country{USA}
  \postcode{12203}}
\email{pbogdanov@albany.edu}

\renewcommand{\shortauthors}{Ma et al.}

\begin{abstract}


The sparse dictionary coding framework represents signals as a linear combination of a few predefined dictionary atoms. It has been employed for images, time series, graph signals and recently for 2-way (or 2D) spatio-temporal data employing jointly temporal and spatial dictionaries. Large and over-complete dictionaries enable high-quality models, but also pose scalability challenges which are exacerbated in multi-dictionary settings. Hence, an important problem that we address in this paper is: \emph{How to scale multi-dictionary coding for large dictionaries and datasets?} 

We propose a multi-dictionary atom selection technique for low-rank sparse coding named \ourmeth. To enable scalability to large dictionaries and datasets, it progressively selects groups of row-column atom pairs based on their alignment with the data and performs convex relaxation coding via the corresponding sub-dictionaries. We demonstrate both theoretically and experimentally that when the data has a low-rank encoding with a sparse subset of the atoms, \ourmeth is able to select them with strong guarantees under mild assumptions. Furthermore, we demonstrate the scalability and quality of \ourmeth in both synthetic and real-world datasets and for a range of coding dictionaries. It achieves $3\times$ to $10\times$ speed-up compared to baselines, while obtaining up to two orders of magnitude improvement in representation quality on some of the real world datasets given a fixed target number of atoms. \vsa
\end{abstract}




\maketitle

\section{Introduction}


Sparse coding methods represent data as a linear combination of a predefined basis (called atoms) arranged in a dictionary~\cite{rubinstein2010dictionaries}. Dictionaries are either derived analytically, for example, discrete Fourier transform, Wavelets, Ramanujan periodic basis~\cite{tennetiTSP2015} or learned from data~\cite{tovsic2011dictionary}. A key assumption in sparse coding is that real-world signals are sparse (or compressive) and can be represented via a small subset of dictionary atoms. 
Sparse coding has been widely adopted in signal processing~\cite{
rubinstein2010dictionaries}, 
machine learning~\cite{maurer2013sparse},
time series analysis~\cite{zhang2021aurora}, image processing~\cite{elad2006image}, and computer vision~\cite{wright2008robust} among others. 

\begin{figure}[t]
    \vsa
   \footnotesize
    \centering
    \subfigure [2D low rank coding for user-product data]
    {
        \includegraphics[width=0.9\linewidth]{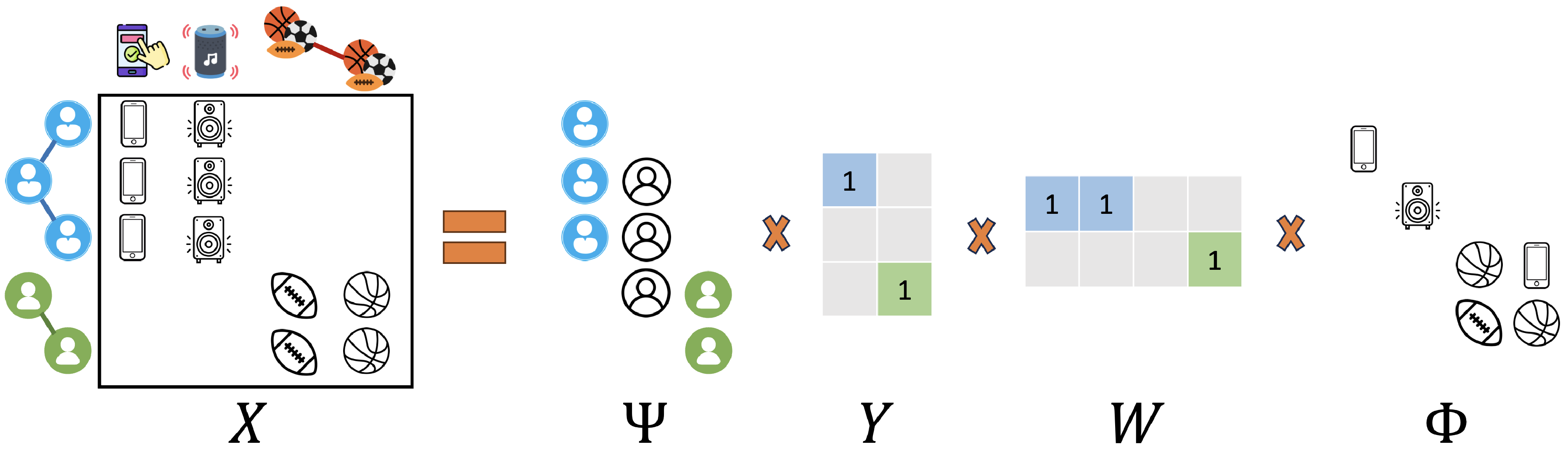}
        \label{fig:2d-coding}
    }
    \subfigure [RMSE v.s. time in road traffic data]
    {
        \includegraphics[width=0.9\linewidth]{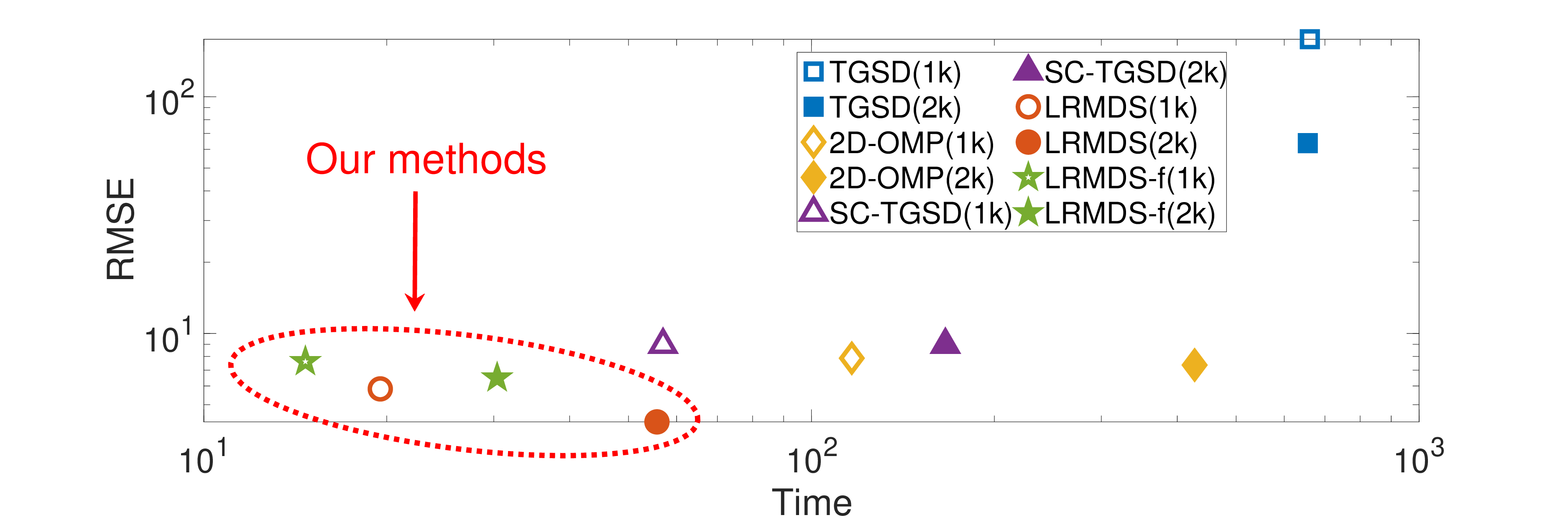}
        \label{fig:intro}
    }
    \caption{\footnotesize \subref{fig:2d-coding} 2D low rank coding example for user-product preference data. The left $\Psi$ and right $\Phi$ dictionaries are derived from user and product association graphs and the goal is to encode the data sparsely via sparse and low-rank coefficient matrices $Y,W$. Our method \ourmeth sub-selects the dictionary atoms on both sides to speed up the coding process. \subref{fig:intro} Comparison of competing techniques on a Road traffic dataset. Variants of \ourmeth outperform all baselines in both representation quality (RMSE) and running time (best regime in the lower-left corner).  
    \vsa}
    \label{fig:intro2}
\end{figure}


Existing approaches, depending on their sparsity-promoting functions, fall in three main categories~\cite{marques2018review}: convex relaxation~\cite{chen2001atomic}, 
non-convex algorithms 
~\cite{ji2008bayesian}, 
and greedy strategies based on matching pursuit~\cite{pati1993orthogonal}. 
Most existing work focuses on 1D (vector) signals such as time series~\cite{tennetiTSP2015} and graph signals~\cite{dong2020graph}. More recent approaches employ sparse coding for multi-way datasets such as images~\cite{fang20122d,grad_proj_2D,robust_2D,old_2D}, spatio-temporal data~\cite{TGSD} and higher order tensors~\cite{MDTD_full}. These 2D and higher-D methods employ separate dictionaries for the different modes (dimensions) of the data. Convex relaxation 2D approaches are typically efficient in practice, but their runtime significantly increases with the size of dictionaries and datasets and they also require careful tuning of hyper-parameters to precisely control the density of encoding coefficients~\cite{TGSD}. Greedy approaches recover a desired number of coefficients (model size), but re-estimate all coding coefficients as new ones are added, and thus do not scale to large model sizes~\cite{fang20122d}. 
Fig.~\ref{fig:2d-coding} demonstrates the 2D setting in the context of user-product purchase data by employing a separate user and product graph dictionaries (e.g., graph Fourier dictionaries~\cite{shuman2013emerging}) based on the corresponding friendship/association graphs.  

Some multi-way techniques further model the coding matrix as low-rank~\cite{2D_low_rank_completion,MDTD_full,TGSD} which enables
improved performance due to sharing of atom-specific patterns within factors. This modeling assumption is demonstrated in the toy example from Fig.~\ref{fig:2d-coding}. Given the purchase history of users, represented as one-hot encoded icons, and association graphs (e.g., user-user friendship and product similarity graphs), one can define graph-based user $\Psi$ and product $\Phi$ dictionaries which represent natural communities for each of the two data dimensions. If purchase behaviors ``conform'' to user and product communities (e.g., first three users purchase electronics while the remaining two purchase sports products), then the purchase data can be described efficiently via a few sparse factors as demonstrated in the $Y,W$ coding matrices.     
For example, the user factors in $Y$ require only two coefficients to represent the corresponding user groups, and similarly $W$ represent groups of products via three coefficients. 

While (low-rank) 2D sparse coding is advantageous and widely applicable, existing methods do not scale to large dictionaries and data. The 2D-OMP approach~\cite{fang20122d} 
composes atoms as outer products of left and right atoms and greedily selects the best aligning pairs for encoding one at a time. Low rank approaches~\cite{TGSD,MDTD_full} adopt convex relaxation based on alternating directions methods of multipliers (ADMM). Both groups suffer from poor scalability with the size of the employed dictionaries. 2D-OMP considers a quadratic number of atom combinations while low-rank approaches rely on inversion of matrices whose size depend on that of the employed dictionaries.
While this challenge of limited scalability has been addressed in the 1D sparse coding scenario via dictionary screening~\cite{xiang2016screening} and greedy dictionary selection~\cite{fujii2018fast}, these methods are not readily applicable to the 2D scenario. Furthermore, as we demonstrate experimentally, naive generalizations of dictionary screening algorithms for the 2D setting result in limited representation accuracy and scalability. 

We propose a low-rank multi-dictionary selection and coding approach for 2D data called \ourmeth. Our approach is general, scalable and theoretically justified. To scale to large dictionary sizes, it iteratively performs adaptive joint dictionary sub-selection and efficient low-rank coding based on convex optimization. \ourmeth iteratively improves the encoding by adding dictionary atoms as needed in rounds. 
We prove theoretically and demonstrate empirically that if the input data conforms to a low-rank coding model via a sparse subset of atoms, \ourmeth is guaranteed to select these atoms in noisy regimes.
An experimental snapshot showcasing the advantages of \ourmeth's variants is presented in Fig.~\ref{fig:intro} for a real-world sensor network dataset. In terms of representation error (vertical axis) and running time (horizontal axis), variants of our method (\ourmeth and \ourmeth-f) occupy the lower left corner which is the optimal regime. 
We experimentally demonstrate similar advantageous behavior on 3 other datasets detailed in the evaluation section. While in this work we focus on evaluating dictionary selection in 2-way (matrix) data, we believe that our framework can be extended to multi-dictionary settings (tensor data) , though we leave such evaluation for future investigation. Our contributions in this paper are as follows:

\noindent{\bf $\bullet$ Novelty:} To the best of our knowledge, \ourmeth is the first dictionary selection method for multi-dictionary sparse coding.

\noindent{\bf $\bullet$ Scalability:} Our approach scales better than alternatives on large real-world datasets and when employing large dictionaries. 

\noindent{\bf $\bullet$ Accuracy:} \ourmeth consistently produces solutions of lower representational error compared to the closest baselines from the literature for a fixed number of coding coefficients (models size).

\noindent{\bf $\bullet$ Theoretical guarantees:} We prove that \ourmeth's dictionary selection optimally identifies the necessary atoms for low-rank components in the input data in the presence of noise.

\section{Related work}

{\noindent \bf Sparse coding}
is widely employed in signal processing~\cite{zhang2015survey,rubinstein2010dictionaries}, image analysis~\cite{elad2006image} and computer vision~\cite{wright2008robust}. Existing methods can be grouped into three main categories: convex optimization solutions, non-convex techniques, and greedy algorithms~\cite{marques2018review}. Relaxation techniques impose sparsity on the coding coefficients via L1 regularizers~\cite{TGSD,MDTD_full}, while 
greedy algorithms select one atom at a time~\cite{wang2012generalized,de2017sparsity,lee2016sparse}. Most existing methods focus on 1D signals while our focus in this paper is on 2D signals.

{\noindent \bf 2D and multi-way coding} methods generalize the one dimensional setting by employing separate dictionaries for each dimension of the data~\cite{fang20122d,old_2D,zhang2017joint,TGSD,MDTD_full}. Some methods in this group place no assumptions on the rank of the encoding matrix~\cite{old_2D,zhang2017joint,fang20122d,2D_low_rank_completion}, while others employ a low-rank model for the encoding matrix~\cite{TGSD, MDTD_full}. Most related to \ourmeth among above the above are 2D-OMP~\cite{fang20122d} which also utilizes a greedy projection to select atoms, and  
TGSD~\cite{TGSD} as it also enforces that learned coding coefficients are low rank for 2D data. 
We elaborate further on these similarities in the following section and demonstrate a superior performance of \ourmeth over these baselines in the experimental section. 



\noindent{\bf Dictionary screening and selection.} 
Dictionary screening~\cite{xiang2016screening} is a suite of methods/bounds for ``discarding'' dictionary atoms of 0 encoding weights at a given sparsity level with the goal of 
reducing the running time of the encoding process. These techniques are limited to 1D data (i.e., only one dictionary), and are not immediately extendable to the two-dictionary setting. We include naive extensions to 2D by creating composite (pairwise) atoms as baselines and demonstrate that they do not scale well to large-dictionaries due to the quadratic space of possible composite atoms. 
There is also work on greedy atom selection for the 1D case~\cite{fujii2018fast} and our method can be viewed as a generalization of such techniques to the 2D sparse coding setting.
\section{Preliminaries} \label{sec:prelim}

Before we define our problem of low rank sub-dictionary selection, we introduce necessary preliminaries and notation. 
The goal of 1D sparse coding is to represent a signal via a single (column) dictionary $\Psi \in \mathcal{R}^{N \times I}$ optimizing the following objective:  
$$\min_y f(y)~~\text{s.t.}~~ x=\Psi y,\vsb$$
where $x \in \mathcal{R}^{N}$ represents the given signal,  $y \in \mathcal{R}^{I}$ is the learned encoding and $f(y)$ is a sparsity promoting function (often the $L_1$ norm). A popular greedy strategy to solve the problem, particularly when the dictionary forms an over-complete basis, is the orthogonal matching pursuit (OMP)~\cite{pati1993orthogonal}. The OMP algorithm maintains a residual of the signal $r$ that is not yet represented, and proceeds in greedy steps to identify the dictionary atom best aligned to the residual:
$\psi_t = {\mathrm{argmax}}_{\psi_i} 
 (r^T \psi_i)$, where $i \in [1, I]$, and $\psi_i$ is the $i$-th atom in $\Psi$. The selected atom $\psi_t$ at step $t$ is appended to the result set, the signal is re-encoded and the residual re-computed. The process continues until a desired number of atoms are employed, while satisfying the sparsity function $f(y)$.




In this paper we consider the 2D setting involving two dictionaries. The input to our problem is a real valued data matrix $X \in \mathcal{R}^{N \times M}$ which can be represented via two dictionaries: a left (column) dictionary $\Psi \in \mathcal{R}^{N \times I}$ and a right (row) dictionary $\Phi^T \in \mathcal{R}^{J \times M}$, where $I$ is the number of atoms in $\Psi$ and $J$ is the number of atoms in $\Phi^T$. It is important to note that both analytical and data-driven dictionaries can be employed ~\cite{rubinstein2010dictionaries}. The 2D problem generalizes that from the 1D case as follows:
\begin{equation}
\min_Z  f(Z)~~\text{s.t.}~~ X=\Psi Z \Phi^T, \label{eq:2dform}\vsb \end{equation}
where $Z\in \mathcal{R}^{I \times J}$ is an encoding matrix, and $f(Z)$ is the corresponding sparsity promoting function. Intuitively this decomposition facilitates a representation which aligns to dictionaries across both modes (dimensions) instead of just one. An early solution for the problem in Eq.~\ref{eq:2dform} was motivated by decomposing a 2D image via copies of the same dictionary, i.e. $\Psi = \Phi$~\cite{fang20122d}. 
It generalizes OMP to obtain a 2D-OMP algorithm by forming 2D atoms $B_{i,j}=\psi_i^T \phi_j$ as outer products of individual left $\psi_i$ and right $\phi_j$ atoms, and by selecting 2D atoms based on their alignment 
with the residual $R$ at every iteration. 
Importantly, while sparse, this solution might in general result in high-rank encoding matrix $Z$ and as we demonstrate experimentally it does not scale to large spatio-temporal datasets and large dictionaries.

A recent method called TGSD~\cite{TGSD} employs 2D sparse coding for general spatio-temporal datasets, and specifically temporal graph signals, where graph and temporal dictionaries are employed as $\Psi$ and $\Phi$ respectively. Another major difference from the 2D-OMP solution is that in order to enforce a low-rank solution, TGSD considers a model with two ``slim'' dictionary-specific encoding matrices $Y \in \mathcal{R}^{I \times r}$ and $W \in \mathcal{R}^{r \times J}$, where the middle dimension $r$ restricts the rank of the encoding $X = \Psi Y W \Phi^T$. The resulting objective is:
\begin{equation}
    \begin{aligned}
        \underset{Y,W} {\mathrm{argmin}} || X - \Psi Y W \Phi^T||_F^2 + \lambda_1 || Y ||_1 + \lambda_2 || W ||_1,
         \label{eq:TGSD}
    \end{aligned}
\end{equation}
where sparsity via an $L_1$ norm is enforced for both $Y$ and $W$.  Here, $\|M \|_{F}$ denotes the Frobenius norm of the matrix $M$.
The solution adopts an ADMM convex relaxation approach (unlike the greedy solution of 2D-OMP) producing an explainable decomposition model relating non-zero coefficients to periodic behavior and active spatial/graph domains in the data~\cite{TGSD}. However, this solution also does not scale to large coding dictionaries and datasets---a challenge we address in our solution. 

\section{Problem formulation and solutions}

\subsection{Problem formulation}
The existing methods for multi-dictionary sparse coding, TGSD and 2D-OMP, do not scale to large datasets and dictionaries for different reasons. 2D-OMP selects one atom pair from each dictionary at a time, re-encodes the data and proceeds with the data residual. While each step is initially fast, the number of atom pairs grows quadratically with the size of the dictionaries. In addition, when the data has a low-rank representation through a subset of atoms, 2D-OMP is not guaranteed to uncover it due to its formulation employing an unconstrained coding matrix $Z$ of quadratic size in the dictionary atoms. Different from that, TGSD is a low rank model, however, its optimization relies on inverting matrices whose sizes are determined by the dictionaries, hence it also does not scale with the size of the dictionaries. The scalability limitations are further exacerbated by the use of over-complete dictionaries which has been shown to produce accurate and succinct models in various signal processing and machine learning applications.  

Our goal is to enable a (i) scalable, (ii) low-rank, (iii) multi-dictionary sparse coding, accommodating large over-complete dictionaries without compromising the quality of the learned model by subs-electing the dictionary atoms. An additional goal is applicability to any 2D signals, including spatio-temporal data, graph signals evolving over time, images and others, by employing appropriate dictionaries for the corresponding data dimensions. Based on the above intuition we formalize our problem as follows:

\noindent \emph{\bf Problem definition:}
    \emph{Given a 2D signal $X$, large potentially over-complete dictionaries $\Psi$ and $\Phi$, and a desired rank $r$, fit a sparse low-rank model $X\approx \Psi_s Y W \Phi_s^T$, employing a subset of the dictionary atoms $\Psi_s,\Phi_s$ and coding matrices of $Y,W$ with inner dimension $r$.}


\subsection{\ourmeth: iterative atom selection and coding}

Both TGSD and 2D-OMP with minor modifications can be adopted for our problem formulation, however, as we discussed earlier they have limited scalability. The key idea behind our approach is to sub-select both dictionaries jointly and fit a low-rank encoding model through the reduced dictionaries. Hence, \emph{a key assumption in \ourmeth is that the data can be represented well by a low-rank encoding matrix and by employing a subset of atoms from the left and the right dictionaries}. This setting is illustrated in Fig.~\ref{fig:2d-coding} where only a subset of atoms from $\Psi$ and $\Phi$ are necessary to represent the data. There are two main steps to obtain \ourmeth's representation: (i) identify an appropriate subset of dictionary atoms which align well with the data and (ii) employ them to perform a low rank dictionary decomposition. We repeatedly perform these steps against the ``unexplained'' residual of the data after each iteration. As a result, our approach can be considered a combination of a greedy sub-dictionary identification followed by a convex encoding step. Importantly, we demonstrate that the greedy atom selection step recovers the optimal atoms to best encode a dataset with low-rank and sparse encoding in noisy regimes under mild assumptions.





To jointly sub-select atoms from both dictionaries, we consider all pairwise 2D atoms of the form $B_{i,j} = \psi_i \phi_j^T, \forall i \in [1, I], \forall j \in [1,J]$ and the magnitude of the projection of the data on them. Specifically we maintain a residual matrix $R\in \mathcal{R}^{N \times M}$ initialized as the input data $X$ and subsequently capturing the signal not yet represented by \ourmeth. The alignment scores of atom pairs are computed as:
\begin{equation}
    \label{eq:align}
    P_{i,j} = \frac{\langle R, B_{i,j} \rangle}{|| B_{i,j} ||_F}, \implies  P=\hat{\Psi}^TR\hat{\Phi},
\end{equation}
where on the left-hand-side $\langle R, B_{i,j} \rangle \triangleq \psi_i^T R \phi_j$ is the alignment of the $i$-th left atom and the $j-th$ right atom with the residual,  and $|| B_{i,j}||_F\triangleq \psi_i \phi_j^T$ is a normalization factor based on the Frobenius norm of the atoms' outer product. The right-hand-side is the equivalent to the left for all $P_{i,j}$ when using per-atom normalized dictionaries $\hat{\Psi}$ and $\hat{\Phi}$ (details in the supplement). At each iteration, our method selects the top $k$ total atoms from a combination of left or right dictionary atoms with respect to this alignment $P$.







Once atoms are selected we calculate a low-rank decomposition of the data via encoding matrices $Y \in \mathcal{R}^{I_s \times r} $ and $W  \in \mathcal{R}^{r \times J_s}$ where $r$ represents the rank of the model, while $I_s$ and $J_s$ are the number of atoms selected from $\Psi$ and $\Phi$ respectively. It is important to note, that the sparsity of the representation is ensured thanks to the atom sub-selection and the low-rank (via two encoding matrices) model, hence in this step we do not further enforce sparsity on the encoding coefficients as it is typical to convex relaxation approaches. This modeling decision enables scalable direct solutions as opposed to more complicated ADMM optimizers. The encoding problems has the following form: 
\begin{equation}
    \begin{aligned}
        \underset{Y,W} {\mathrm{argmin}} \hspace{0.1cm} & || R - \Psi_s Y W \Phi_s^T ||_F^2,
    \end{aligned} \label{eq:denseTGSD}
\end{equation}
where $\Psi_s$ and $\Phi_s$ are the subselected dictionaries and $R$ is the data residual originally initialized as $X$. We propose two alternating optimization schemes for the two variables $Y, W$ leading to two variants of \ourmeth (\ourmeth and \ourmeth-f). Both variants iteratively updated $Y$ and $W$ until convergence, however, \ourmeth-f does so faster but at the potential price of accuracy. Detailed explanation and derivations are available in the supplement.

\begin{algorithm} [!t]
\footnotesize
    \caption{Low Rank Multi-Dictionary Selection (\ourmeth)} 
	\begin{algorithmic}[1]
	    \State Input: Data $X$; dictionaries $\Psi$ and $\Phi$; atoms per iteration $k$, decomposition rank $r$
	    \State Output: Encoding matrices $Y, W$
		\State Initialize residual $R = X$
        \State \emph{ \emph{// Compute normalized dictionaries}}
        \State Compute $\hat{\Psi}:\{||\hat{\psi}_i||_2=1,\forall i\leq I\}$,  $\hat{\Phi}:\{||\hat{\phi}_j||_2=1,\forall j\leq J\}$ 
		\State Initialize sets of selected atoms: $I_s = \emptyset, J_s = \emptyset$ 
		\Repeat
            \State {\emph{// Dictionary sub-selection}}
            \State Let $P=\hat{\Psi}^TR\hat{\Phi}$ \Comment{Eq.~\ref{eq:align}}
            \State $cnt=0$
            \For {$P_{i,j}$ in descending order and while $cnt<k$}
                \State {\bf if }{$i\notin I_s$} {\bf then } {$I_s = I_s \cup i$; $cnt=cnt+1$} 
                \State {\bf if }{$j\notin J_s$} {\bf then } $J_s = J_s \cup j$; $cnt = cnt+1$
            \EndFor
			\State Sub-select dictionaries: $\Psi_s = \Psi (I_s)$, $\Phi_s = \Phi (J_s)$  
            \State{\emph{// Encoding based on $\Psi_s,\Phi_s$}}
			\State Initialize randomly $Y_{|I_s| \times r}$ and $W_{r \times |J_s|}$
			\If {\ourmeth}
                \State Pre-compute $\Psi_s^{(inv)} = \Psi_s^\dagger$ and $\Phi_s^{(inv)} = \Phi_s^\dagger$
                \Repeat
                    \State $Y = \Psi_s^{(inv)}  X (W \Phi_s)^\dagger$ 
                    \State $W = (\Psi_s Y )^\dagger   X \Phi_s^{(inv)}$ 
                \Until{$Y, W$ converge}
            \ElsIf {\ourmeth-f}
                \State Pre-compute $C = \Psi^{\dagger}  X \Phi^{\dagger}  $
                \Repeat
                    \State $Y =C  W^\dagger $ 
                    \State $W = Y^\dagger C $ 
                \Until{$Y, W$ converge}
            \EndIf
			\State $R = X - \Psi_s Y W \Phi_s^T$
		\Until{$||R||_F$ converges to $0$ or after a fixed number of iterations}
	\end{algorithmic} 
        \label{alg:LRMDS}
\end{algorithm}

\noindent{\bf The overall  \ourmeth algorithm}. The steps of the complete algorithm (corresponding to both versions of our method) are listed in Alg.~\ref{alg:LRMDS}. The inputs include the data $X$, left $\Psi$ and right $\Phi$ dictionaries and parameters for the number of atoms to select per iteration $k$ and the rank $r$ of the encoding, i.e., the inner dimension of the two output encoding matrices $Y, W$. In the initialization steps, we first compute per-atom normalized versions of the dictionaries needed for the alignment scoring (Steps 4-5) and initialize empty sets of atom indices for both dictionaries (Step 6). Dictionary sub-selection takes place in Steps 8-15. In Steps 10-14 we add the top aligned atoms $i$ and $j$ to the set of select atoms $I_s$ and $J_s$ so long as they are not already selected. We repeat until a total of $k$ new atoms are selected. Finally, we sub-select the relevant atoms from their dictionaries in Step 15 to create sub-dictionaries $\Psi_s$ and $\Phi_s$.

\begin{table*}[th]
\setlength\tabcolsep{1 pt}
\footnotesize
\centering
 \begin{tabular}{|c| c| c| c| c |c| c | c| c | c| c | c| c | c| c | c|} 
 \hline
 {\multirow{2}{*}{\bf{Dataset}}} & {\multirow{2}{*}{\bf{\#Nodes}}}  & \bf{\#Time}  & {\multirow{2}{*}{\bf{Res.}}} & {\bf Associated} &  \multicolumn{2}{|c|}{\bf{TGSD}} & \multicolumn{2}{|c|}{\bf{2D-OMP}} & \multicolumn{2}{|c|}{\bf{SC-TGSD}} & \multicolumn{2}{|c|}{\bf{\ourmeth}} & \multicolumn{2}{|c|}{\bf{\ourmeth-f}}\\
 \cline{6-15}
  &   &  \bf{Steps} &  & {\bf Graph} & \bf{RMSE} & \bf{Time} & \bf{RMSE} & \bf{Time} & \bf{RMSE} & \bf{Time} & \bf{RMSE} & \bf{Time} & \bf{RMSE} & \bf{Time}\\
  \hline
  Synthetic  & 1k-4k & 1k-8k & - & SBM & 0.06 & 146 & 0.02 & 3196 & 0.03 & 61.7 & \bf{0.009} & 31.3 & \bf{0.009} & \bf{30.1}\\ 
 \hline
 Road~\cite{LA_traffic}  & 1923   & 920 & 1h & Road network & 17.8 & 285 & 10.1 & 228 & 12.1 & 32 & \bf{5.4} & 37  & 5.8 & \bf{15} \\ 
 \hline
   Twitch ~\cite{twitch}  & 78,389   & 512 & 1h & Shared audience & 1.38  & 8,413  & 1.36  & 341,294  & 1.35 & 5,353  & \bf{1.23}  & 9655  & 1.26  & \bf{4,280} \\ 
 \hline
  Wiki~\cite{Wiki} & 999   & 792 & 1h  & Co-clicks & 15.7 & 422 & 9.7 & 1390 & 11.7 & 41 & \bf{2.8} & 52 & 3.5 & \bf{37}\\
  \hline
  Covid~\cite{covid} & 3047 & 678 & 1d  & Spatial k-NN & 31969 & 551 & 23908 & 2668 & 21320 & 267 & \bf{204} & 145 & 228 & \bf{88}\\
  \hline
\end{tabular}
 \caption{\footnotesize Statistics of the datasets used for evaluation (left sub-table) and quality and running times for competing techniques (right sub-table). All datasets have a temporal and graph mode with corresponding dictionaries. The temporal resolution of each dataset is specified in column Res while the following column lists the kind of associated graph.
RMSE and timing results of all competing methods using the same number of atoms are listed in the remaining columns. The target number of atoms are as follows: Synthetic with ground truth (GT) atom count (200,  GW+RS test in Fig.~\ref{fig:dict_size_trends});  Road, Twitch, and Covid: 40\% of total atoms; Twitch: 20\% of total atoms.}  \vspace{-0.3in}
\label{table:datasets}
\end{table*}

Steps 16-30 perform the low-rank coding by estimating $Y$ and $W$ based on the sub-dictionaries $\Psi_s$ and $\Phi_s$. We list the iterative updates of both versions \ourmeth (Steps 18-23) and \ourmeth-f (Steps 24-29) of our method. For the former, we pre-compute the pseudo inverses of the subselected dictionaries $\Psi^\dagger, \Phi^\dagger$ and iterate between close form updates of $Y$ and $W$ while in the latter we precompute the projection of the data on the pseudo inverses of the subselected dictionaries $\Psi^\dagger X \Phi^\dagger$ and perform simpler updates for the coefficients in $Y$ and $W$. Further discussion of the difference between the two versions and derivation details for the updates are available in the supplement. Finally, in Step 31 we re-calculate the residual matrix $R$. We repeat all steps until a fixed number of iterations is reached or if the norm of the residual ($||R||_F$) approaches zero. 

The overall complexity of \ourmeth is $O(t(q(N+M)r^2 + MJ_s^2 + NI_s^2))$ for \ourmeth or  $O(t(q(I_s+J_s)r^2 + MJ_s^2 + NI_s^2))$ when using \ourmeth-f where $t,q$ are the total number of iterations of the main loop and $Y,W$ updating. Our framework is designed to flexibly accommodate any dictionaries $\Phi$ and $\Psi$. Discussion of dictionaries employed for modes of different types (e.g., time series, graphs, etc.) can be found in the supplement.



\subsection{Dictionary subselection theoretical analysis}
\label{sec:theoretical-results}


In this section, we give an accuracy guarantee for the quality of our method's selection of top $k$ atoms in each step. when used to recover a low-rank signal matrix $R$ from a data matrix $\hat{R}=R+Q$ that includes Gaussian noise. Throughout, we assume without loss of generality that $M = o(\sqrt{N})$ as $N\to\infty$ (the argument applies equally well when the roles of $M$ and $N$ are switched). We describe this denoising problem setting and assumptions for our theoretical guarantee below.

\noindent{\bf Noise model: }
We will consider the recovery of a low-rank, sparse signal matrix $R$ from a noise-perturbed version 
$\hat{R} := R + Q$, where $Q$ is a matrix in $\R^{N\times M}$ with independent and identically distributed standard Gaussian entries, appropriately normalized.  Specifically, the assumption that $N \gg M$ implies that with high probability, each column of a standard Gaussian matrix has $L_2$ norm approximately $\sqrt{N}$.  Thus, we define the noise component as:
\begin{align}
    Q = \frac{\sigma}{\sqrt{NM}} \cdot \Normal(0, I_{N\times M}),
\end{align}
for some positive standard deviation $\sigma$.
With this normalization, $\|Q\|_F = \Theta(1)$, and thus has the same order of growth as $\|R\|_F$, with high probability.

\noindent{\bf Low rank and sparsity assumptions on $R$: }
We will assume that $R$ has rank $r$, which implies that $R$ has an expansion of the form
\begin{align}
    \label{expr:low-rank-expansion}
    R = \Psi  Y  W \Phi^T = \sum_{i=1}^I \sum_{j=1}^J p_{i,j} \psi_{i} \phi_{j}^T,
\end{align}
where $Y \in \R^{I\times r}$ and $W \in \R^{r \times J}$, and where the double-sum atom-based representation is under the assumption that $\Psi$ and $\Phi^T$ are not under-complete. 
Eq.~\ref{expr:low-rank-expansion} implies the following explicit $p_{i,j}$ form:
\begin{align}
    p_{i,j}
    = (Y_{i,\cdot}) (W_{\cdot, j})^T.
\end{align}
Furthermore, we will make the following sparsity assumption: there exist only $s = \Theta(1)$ dictionary coefficients $p_{i,j}$ in the expansion of $R$ that are nonzero, and these are $\Theta(1)$ uniformly in $N$ and $M$.

\noindent{\bf Assumption on approximate orthogonality of dictionary atoms: }
We next formulate an approximate orthogonality condition for the dictionary atoms.  To do so, we first recall
the definition of the $L_{\infty}$ operator norm of a matrix $M$:
\begin{align}
    \| M\|_{op,\infty} := \sup_{\|x\|_{\infty} = 1} \|Mx\|_{\infty}.
\end{align}
We will assume that the dictionaries $\Psi \in \R^{N\times I}, \Phi^{T} \in \R^{J\times M}$ are such that
$I \in [N, const \cdot N], J \in [M, const\cdot M]$ and that a subset of the atoms for each dictionary constitutes a basis for $\R^N$ and $\R^M$, respectively.
We fix an $\alpha \geq 0$, which may depend on $N$ and $M$.  We also define a matrix $\Sigma \in \R^{I\times I}$
collecting the pairwise inner products between dictionary elements in $\Psi$: namely,
$ 
    \Sigma_{i,j} := (\psi_{i})^T \cdot \psi_{j}.
$ 
We will assume that $\| \Sigma^{1/2} \|_{op,\infty} \leq \alpha = o(1)$.  We similarly define $\Gamma \in \R^{J\times J}$ for $\Phi^T$, with the same bound on $\| \Gamma^{1/2}\|_{op,\infty}$.  Intuitively, this operator norm upper bound translates to an upper bound on the sum of absolute values of inner products between dictionary atoms.  Thus, this constitutes an approximate orthogonality assumption.  Such assumptions are common in analyses of orthogonal matching pursuit, under the name of \emph{mutual incoherence} between dictionary atoms.  
See, e.g.,~\cite{Cai2010OrthogonalMP}.  We will also assume that the columns of $\Psi$ and the rows of $\Phi^T$ are normalized in $L_2$.

\noindent{\bf Statement of the accuracy guarantee: }
We finally state our main theoretical result.  We denote by $R_{reconst}$ the output of the top-$
\hat{k}$ atom selection algorithm with input data matrix $\hat{R}$ and dictionaries $\Psi, \Phi^T$.  Specifically, by this we mean that $R_{reconst}$ is the result of first choosing the $\hat{k}$ atoms (outer products of left and right dictionary elements) with the highest alignment scores with $\hat{R}$, then approximating $\hat{R}$ via a linear combination of the chosen atoms obtained by solving (\ref{eq:denseTGSD}).

\begin{theorem}[Accuracy guarantee for top-$k$ atom selection denoising]
    \label{thm:convergence-top-k}
    Let $N, M, \Psi, \Phi^T, R, Q, \hat{R}, R_{reconst}$ be as outlined above.  
    We then have that if $\hat{k} \geq s$ and $\hat{k} = \Theta(1)$, where $s$ is the sparsity parameter of the the signal matrix $R$, then 
    $ 
        \|R - R_{reconst}\|_{F}
        = o(\|R\|_{F}).
    $ 
\end{theorem}
In other words, the relative error in approximating $R$ by $R_{reconst}$ is $o(1)$ as $N\to\infty$.  That is, when the data consists of a Gaussian-noise perturbation of a low-rank signal matrix that is a sparse linear combination of dictionary atoms, and when the greedy atom selection algorithm chooses sufficiently many atoms,  
the signal matrix is recovered to within a vanishingly small relative error.
We prove Theorem~\ref{thm:convergence-top-k} in Appendix~\ref{sec:convergence-top-k-proof}.  We also extend the result to cover the case where $\hat{k} < s$ but the algorithm is run for sufficiently many iterations.

Our analysis provides a theoretical justification for the top k atom selection strategy as a recovery guarantee for a noise-perturbed low-rank and sparse signal matrix, which forms a subroutine of our method. We also demonstrate empirically that this recovery guarantee holds in Sec. \ref{sec:theor_exp}. The scalability of our technique comes from the fact that we are working with a few atoms as opposed to the complete dictionary and the theoretical analysis shows that this running time reduction affects minimally the model quality since the``true'' atoms necessary for encoding the noise-free version of the data are retained.

\section{Experimental evaluation}


Our experimental design focuses on the running time and the representation quality of competing methods with both variants of \ourmeth on synthetic and real-world datasets listed in Tbl.~\ref{table:datasets}. We also empirically confirm our theoretical results. We compare our approaches to state-of-the-art baselines for 2D sparse coding. We measure running time in seconds for execution on a dedicated Intel(R) Xeon(R) Gold 6138 CPU @ 2.00GHz and 251 GB memory server using MATLAB's R2019a 64-bit version. The representation quality is quantified as the root mean squared error (RMSE) between the data and the learned representation. It is important to note, that beyond representation quality, the low-rank sparse 2D coding model offers advantages in a number of downstream tasks as reported by baselines employing this or similar models~\cite{TGSD,fang20122d}. We focus our evaluation of scalability and representation quality as speed-up via dictionary sub-selection is the main contribution of our work. An implementation of \ourmeth is available at \url{www.cs.albany.edu/~petko/lab/code.html} and in the PySpady library \url{https://github.com/petkobogdanov/pyspady}.

\subsection{Datasets}\hfill

\noindent{\bf Synthetic data generation.} Our synthetic data is generated based on the low-rank encoding model $\Psi_s Y W \Phi_s^T + \epsilon$, where $\Psi_s$ and $\Phi_s$ are small (ground truth) randomly selected subsets of the overall dictionaries and the corresponding coding coefficients for those atoms in $Y$ and $W$ are also randomly sampled with $\epsilon$-mean random noise added to the input.
\noindent{\bf Real-world datasets.} We employ real-world datasets with temporal and spatial dimensions to evaluate competing techniques employing both time and graph dictionaries. The datasets span a variety of domains: data from content exchange within \emph{Twitch}~\cite{twitch}, web traffic data \emph{Wiki}~\cite{Wiki}, spatio-temporal disease spread over time in the \emph{Covid}~\cite{covid} dataset, and sensor network data from road traffic \emph{Road}~\cite{LA_traffic}. We provide further details on data generation and real-world evaluation datasets in the supplement.



\begin{figure*}
    \vsa\vsa
   \footnotesize
    \centering
    \subfigure [RMSE]
    {
        \includegraphics[width=0.22\linewidth]{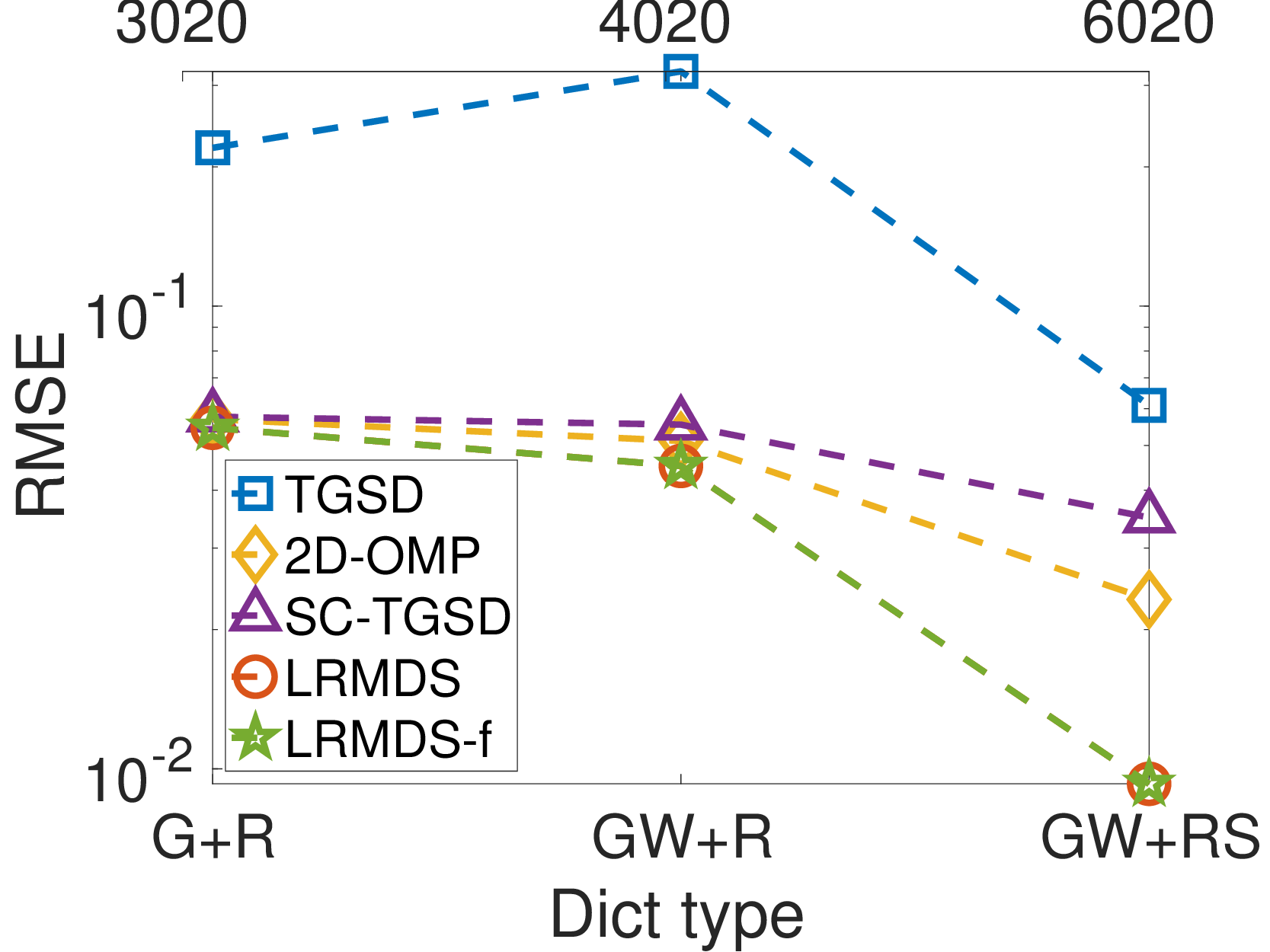}
        \label{fig:syn_size_rmse}
    }
    \subfigure [Time]
    {
        \includegraphics[width=0.22\linewidth]{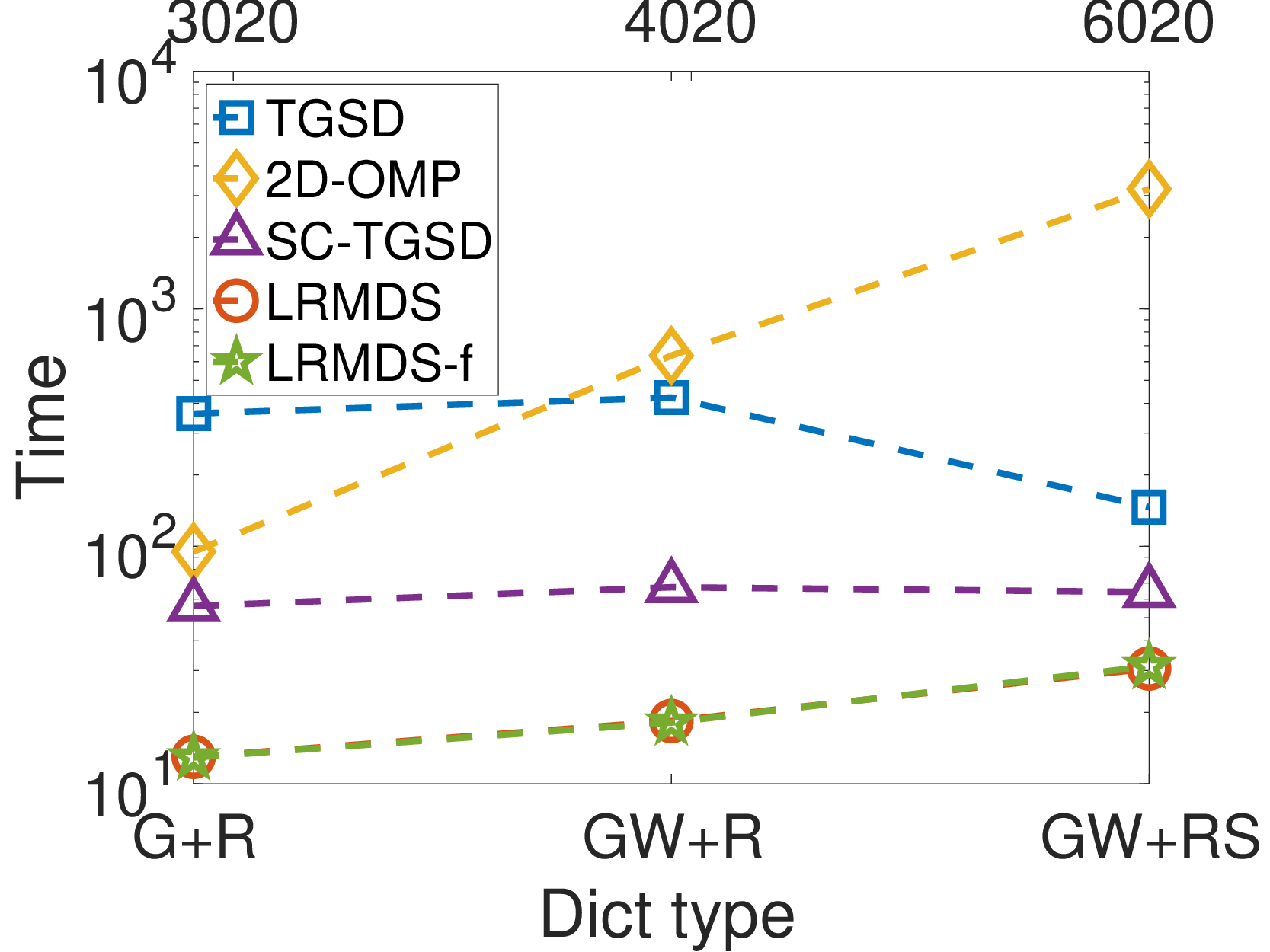}
        \label{fig:syn_size_time}
    }
    \subfigure [\scriptsize GW+RS: RMSE vs Atom\%]
    {
        \includegraphics[width=0.22\linewidth]{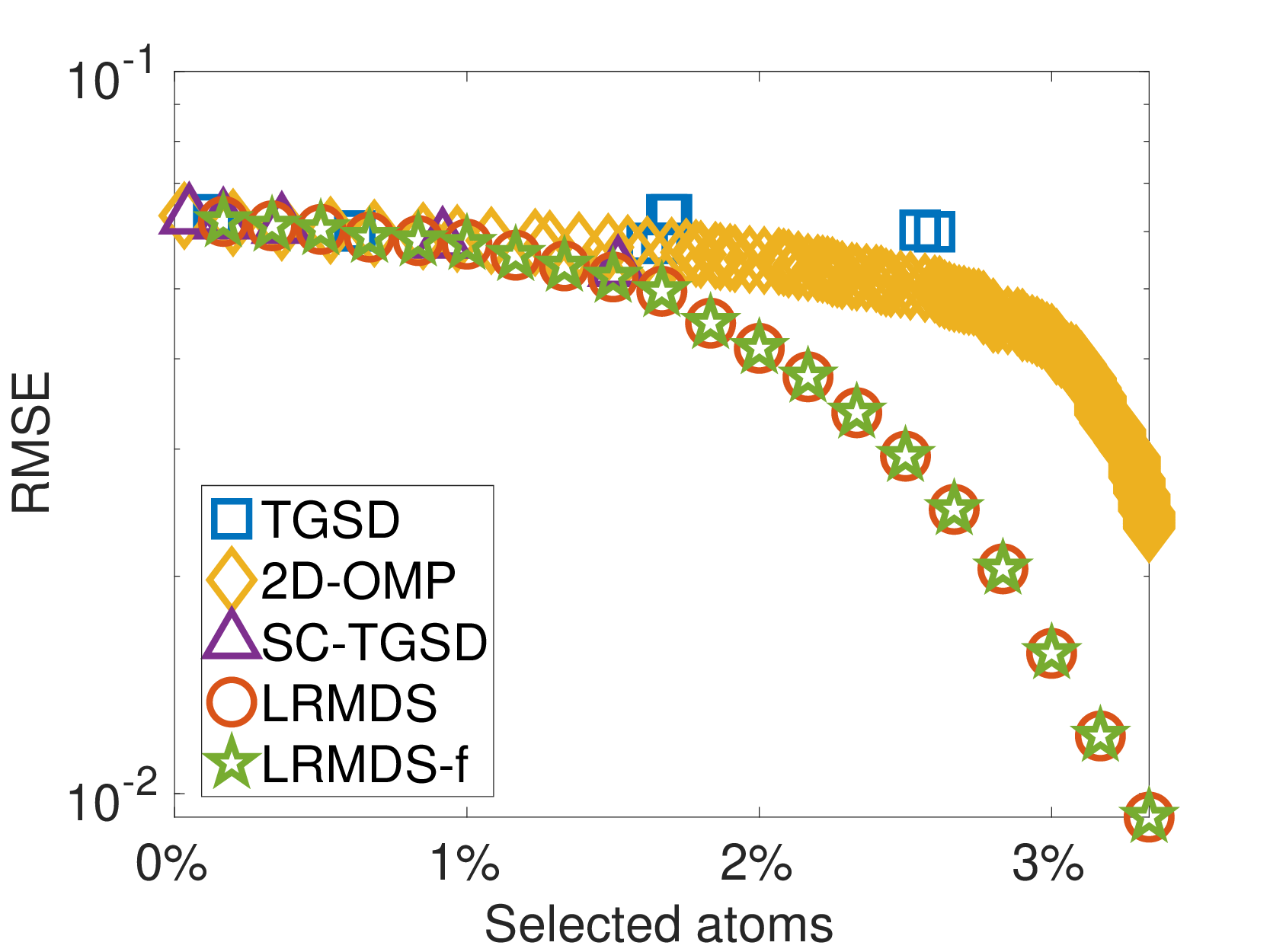}
        \label{fig:GWRS_rmse_vs_atom}
    }
    \subfigure [\scriptsize GW+RS: Time vs Atom\%]
    {
        \includegraphics[width=0.22\linewidth]{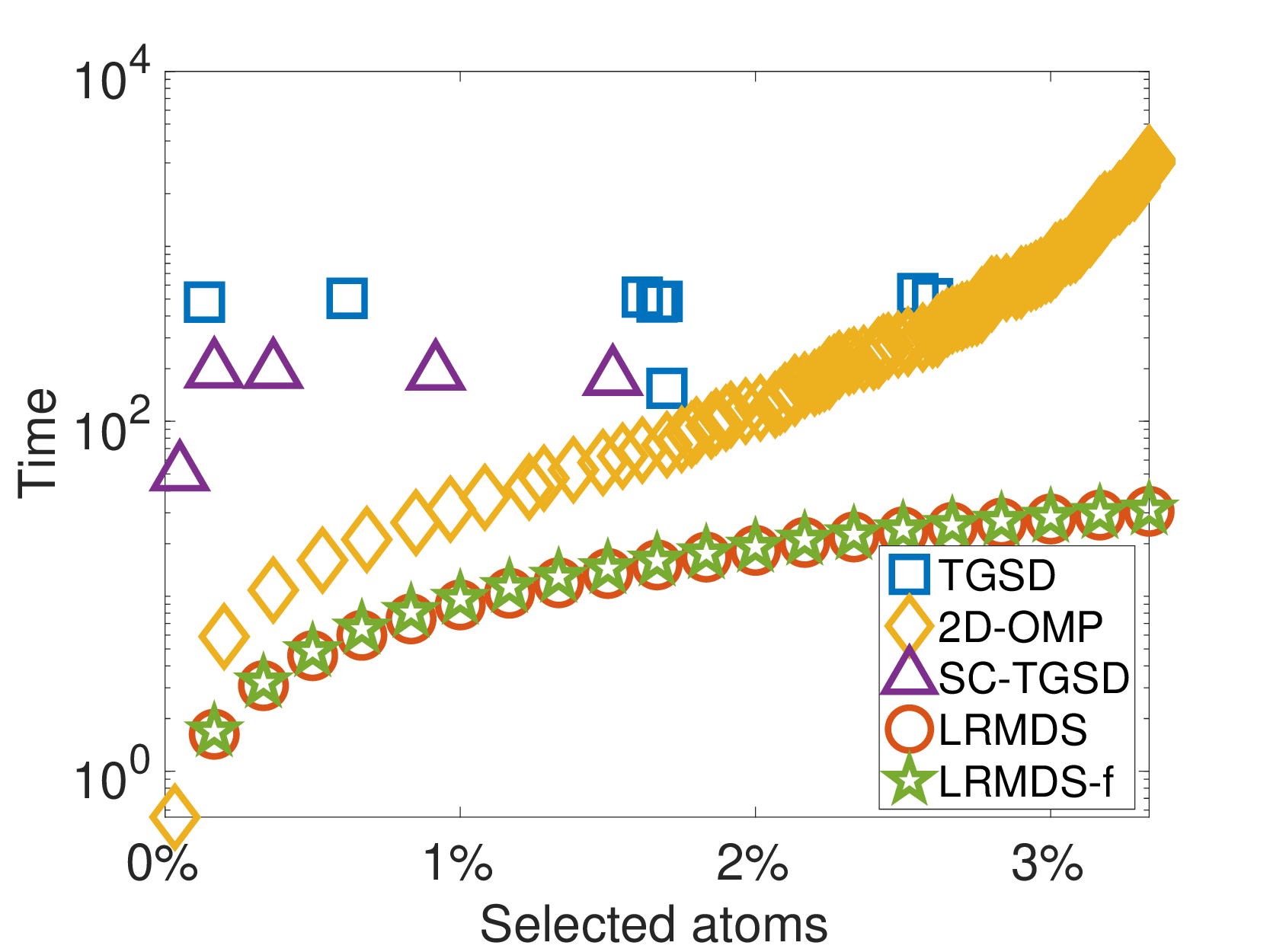}
        \label{fig:GWRS_time_vs_atom}
    }
    \vspace{-0.2in}
    \caption{\footnotesize
     Comparison of competing techniques on synthetic data. ~\subref{fig:syn_size_rmse}, ~\subref{fig:syn_size_time}: RMSE and running time for varying dictionaries available to each method (listed under the x axis). 
     The total number of (left and right dictionary) atoms is specified at the top of each figure. We stack increasing sets of dictionaries on the left and right, while the ground truth atoms are selected from the full set GW+RS.  \subref{fig:GWRS_rmse_vs_atom}: RMSE as a function of the number of selected atoms when multiple dictionaries are provided. \subref{fig:GWRS_time_vs_atom}: Run time as a function of the number of selected atoms. GW+RS stands for GFT and Graph Haar wavelets stacked together for the graph dimension and RS stands for Ramanujan and Spline dictionaries stacked for the temporal dimension (details of the dictionary definitions are available in the supplement).\vspace{-0.1in}
    }
    \label{fig:dict_size_trends}
\end{figure*}

\subsection{Experimental setup.}
\noindent{\bf Baselines.} We compare the versions of \ourmeth to the two available methods for multi-dictionary coding: (\textbf{TGSD}~\cite{TGSD} and \textbf{2D-OMP}~\cite{fang20122d}). These methods have already been discussed in detail in Sec.~\ref{sec:prelim}. As a brief summary, \textbf{TGSD}
solves the problem of low-rank encoding within an $L_1$ sparsity regularized optimization framework. {\bf 2D-OMP} selects dictionary atom pairs in a greedy manner and estimates the corresponding coding coefficients one at a time. It produces a solution which is not guaranteed to be low-rank. 

Since a key advantage of our method is its sub-selection of large dictionaries, we also seek to understand if extending 1D dictionary screening to the 2D setting results in a scalable 2D approach. To this end, we generalize a 1D dictionary screening approach~\cite{xiang2016screening} to work with multiple dictionaries and combine it with TGSD to facilitate a more thorough comparison. The resulting method \textbf{SC-TGSD} screens (removes) the worst dictionary atoms from a dictionary by calculating alignment scores between atoms and associated data. To perform the subsequent coding, we employ TGSD with only the sub-selected dictionaries. Intuitively this baseline can be thought as a 2-step combination of screening and TGSD coding. Additional details about this screening process is available in the supplement.




\noindent{\bf Metrics.}
We measure the reconstruction error of the learned representation using root mean squared error (RMSE$ = \sqrt{\frac{\sum_{i,j} (X_{i,j} - X'_{i,j})^2}{|X|}}$, where $X$ is the original signal, $X'$ is the reconstruction and $|X|$ denotes the number of elements in $X$) of the learned representation's departure from the input data. We measure the running time in seconds for all competing methods.


 \noindent{\bf Experimental design.} The goal of our experiments is to demonstrate the utility of \ourmeth in both synthetic and real world datasets. 
 In synthetic data tests, we varying different properties of the data generation process, such as SNR and the dictionary size and type. We seek to quantify the speed up and quality improvement that \ourmeth enables compared to baselines. Parameter settings and details on used dictionaries can be found in the supplement.





\vsa
\subsection{Evaluation on synthetic data.}

In our synthetic data evaluation we compare the effects of dictionary size and type on the performance of \ourmeth and its competitors. Using the ground truth number of atoms as a target, we compare the reconstruction error (RMSE) and running time (secs) for all techniques. 
We also characterize the effect of varying noise and show these results in the supplement.

\noindent{\bf Varying dictionary composition and size.} As all competitors perform a type of dictionary sub-selection (implicitly in the case of TGSD due to its regularization), a natural question to ask 
 is: \emph{How does the composition and size of the input dictionaries affect the ability of a method to quickly and accurately represent a data matrix?}
To answer this question we first utilize a set of composite dictionaries to generate the data input. In this case the left composite dictionary is a stack of a GFT (G) and a graph Wavelet (W) dictionaries and the right composite dictionary is a stack of Ramanujan (R) and Spline (S) temporal dictionaries. We use 50 randomly chosen atoms from each of the four dictionaries (200 in total) to generate the synthetic input data. We then prepare 3 test settings by varying the dictionaries compositions available to the competitors. The model input dictionaries in these 3 settings are as follows (1) $\Psi = [G], \Phi = [R]$ denoted G+R; (2) $\Psi = [G, W], \Phi = [R]$ denoted GW+R; and (3) $\Psi = [G, W], \Phi = [R, S]$ denoted GW+RS. 

We first compare the RMSE and running time for all techniques when using a fixed number of dictionary atoms and report results of this experiment in Fig.~\ref{fig:syn_size_rmse},~\ref{fig:syn_size_time}. The x-axis lists the consecutive dictionary compositions and the total number of atoms is listed on top of the figure. We report RMSE and run time of each method when employing $200$ atoms, which is also the number of ground truth (GT) atoms. Note that for SC-TGSD and TGSD, we pick the point for which the employed number of atoms is closest to the GT since they have no parameters to directly control the exact number of atoms in their methods. As larger dictionaries are being used, the RMSE of all methods improves, however \ourmeth achieves the best reconstruction quality at all points. This is due to the higher quality atom selection compared to baselines. \ourmeth variants are also the fastest to select these atoms as seen in Fig.~\ref{fig:syn_size_time}.

In Figs.~\ref{fig:GWRS_rmse_vs_atom}, ~\ref{fig:GWRS_time_vs_atom}, we further break down the performance of each method when utilizing ground truth composite dictionaries GW+RS. Explicitly, we track the RMSE and run time as a function of the percentage of selected atoms.
We additionally show that the representation quality improves as a function of the total run time (results in supplement). Similar results for other choices of dictionary combinations settings (G+R, GW+R) are also available in the supplement.

Both variants of \ourmeth obtain the most accurate representations among competitors while simultaneously taking the least amount of time. The next best approach, 2D-OMP initially selects and updates its coefficients quickly, closely trailing \ourmeth when the representation quality for both is poor. However, with further iterations the representation quality gains by 2D-OMP slow down. This is due to 2D-OMP's restriction to only select one coefficient corresponding to an atom pair per iteration and the need to re-estimate the coefficients for all previously selected pairs. This is highly inefficient when many pairs of a small subset of atoms in the optimal sparse coding are non zero. In such cases, 2D-OMP still adds pairs one at a time, while our approach allows for coding with all combinations of already selected left and right atoms.

\begin{figure*}[th]
   \footnotesize
    \centering   
    \subfigure [Twitch: RMSE vs Atom\%]
    {
        \includegraphics[width=0.25\linewidth]{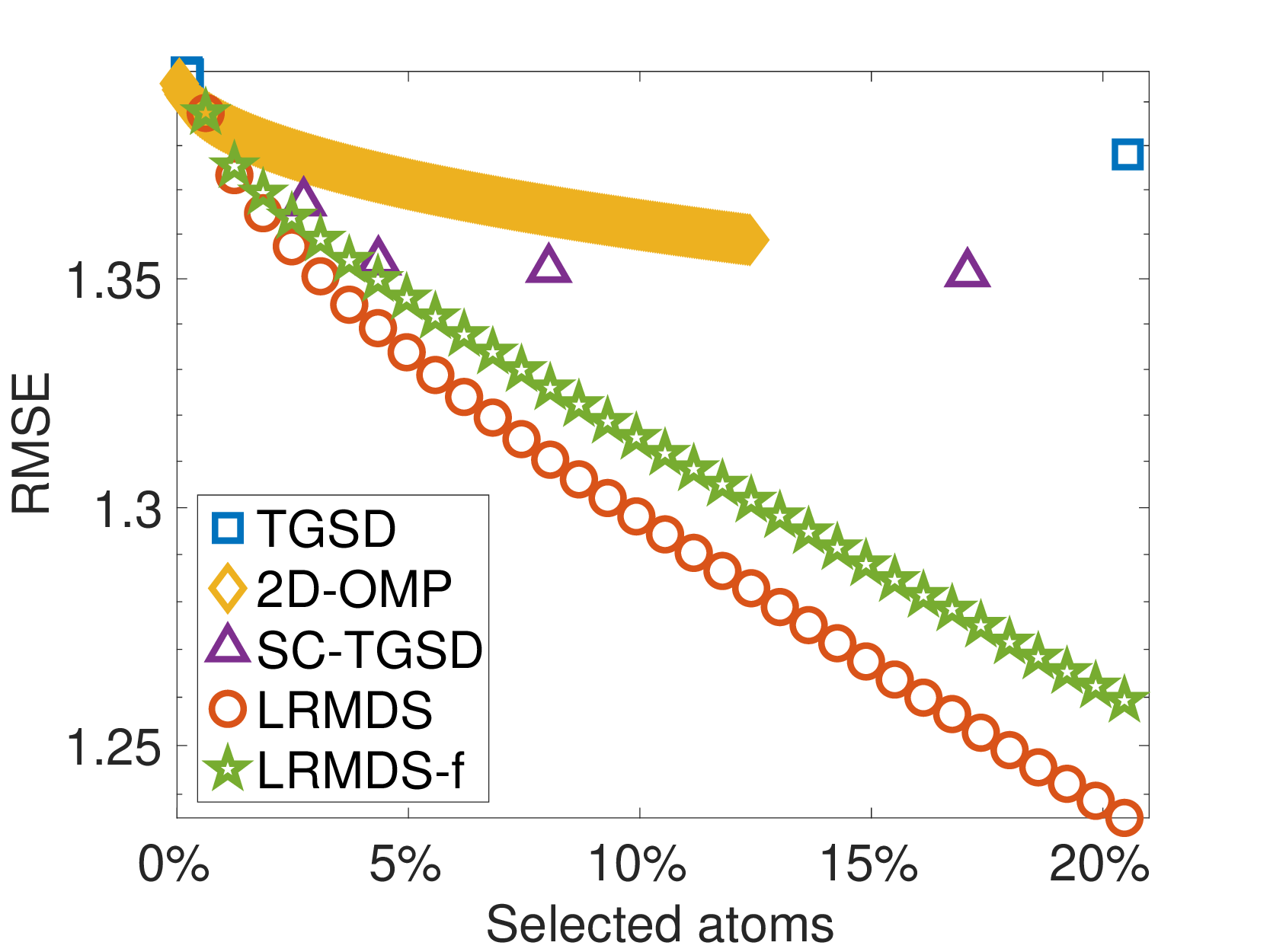}
        \label{fig:twitch_rmse_vs_atom}
    }
    \hspace{-0.2in}
    \subfigure [Road: RMSE vs Atom\%]
    {
        \includegraphics[width=0.25\linewidth]{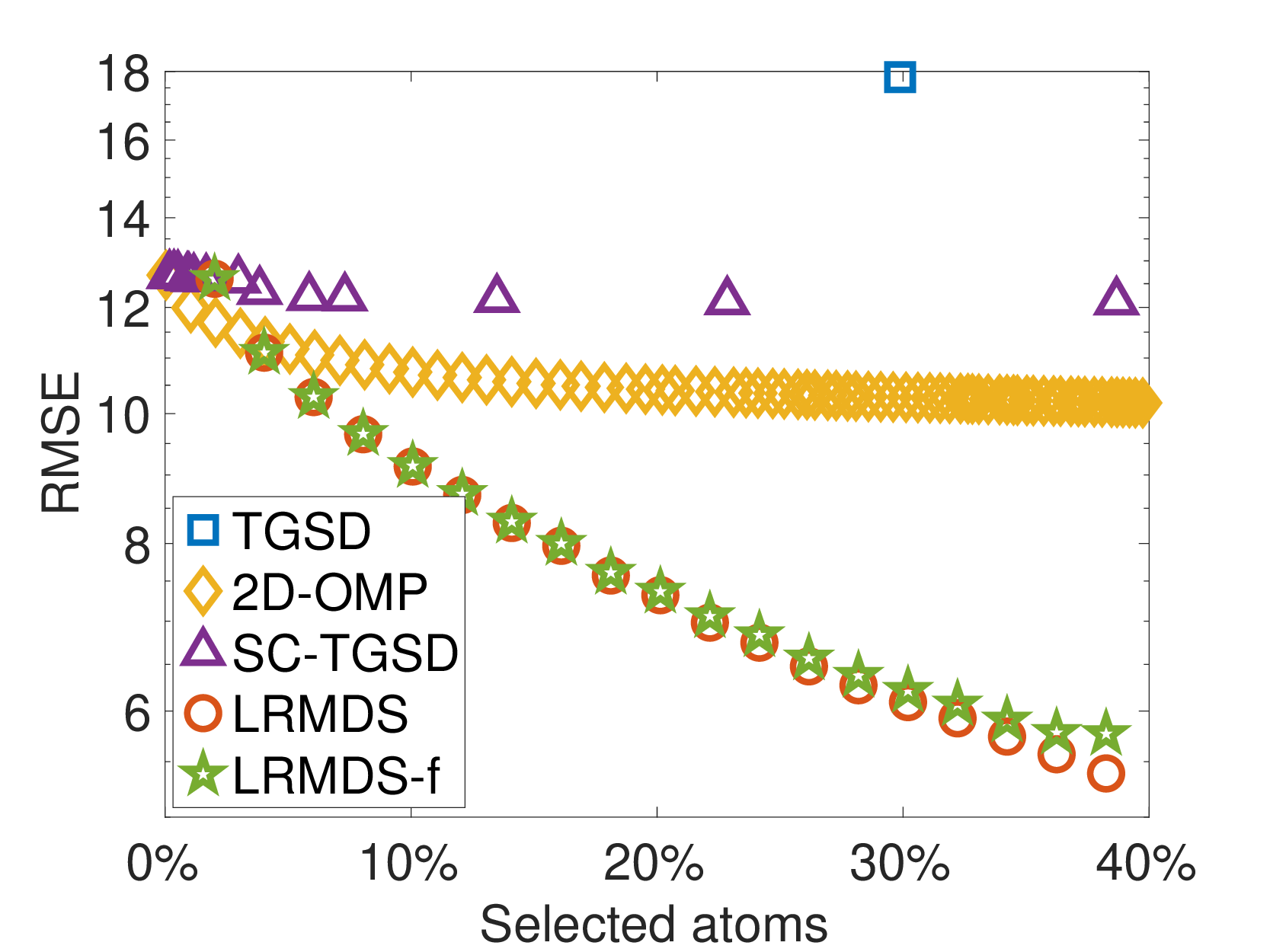}
        \label{fig:road_rmse_vs_atom}
    }
    \hspace{-0.2in}
    \subfigure [Wiki: RMSE vs Atom\%]
    {
        \includegraphics[width=0.25\linewidth]{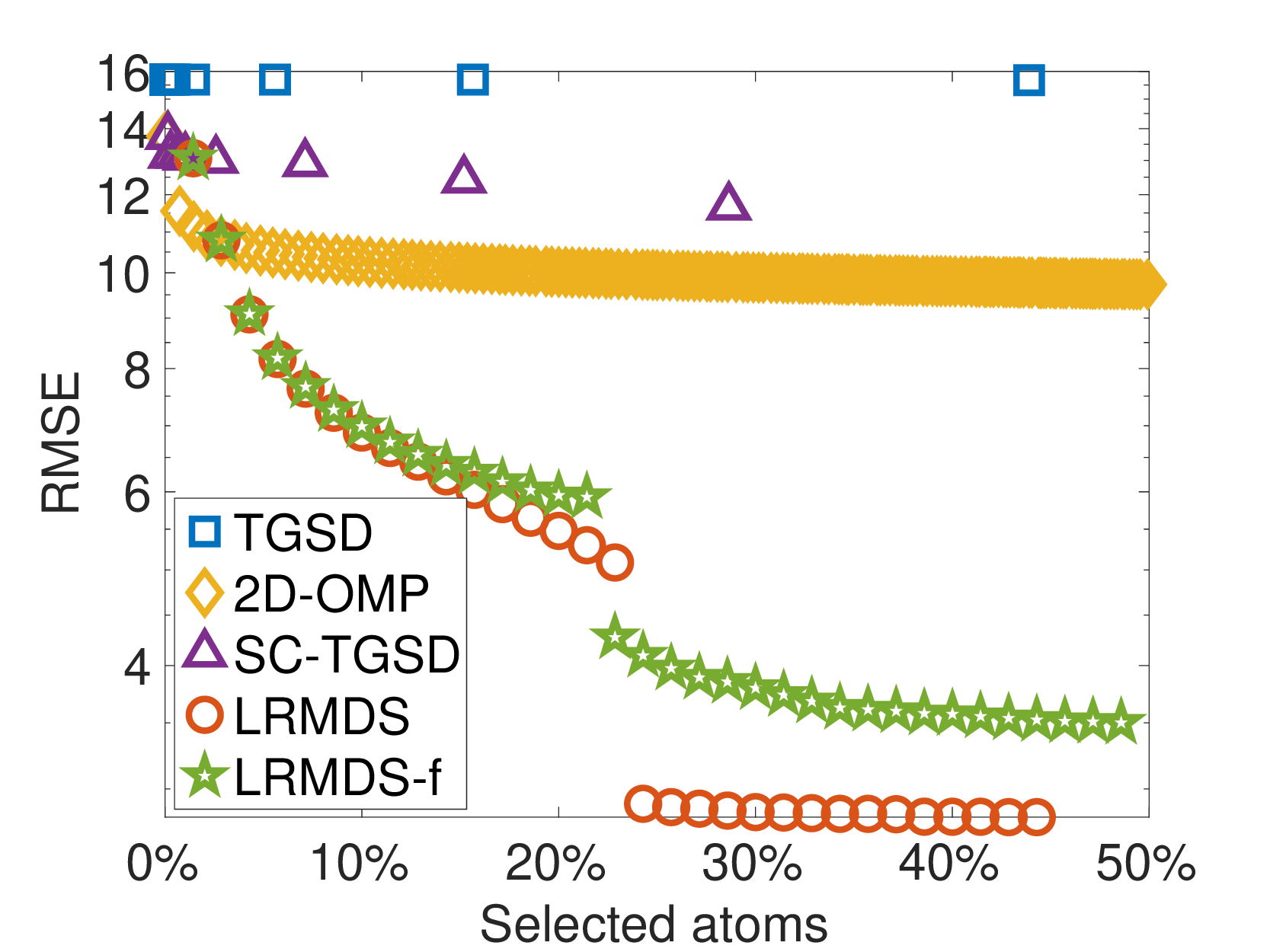}
        \label{fig:wiki_rmse_vs_atom}
    }
    \hspace{-0.2in}
    \subfigure [Covid: RMSE vs Atom\%]
    {
        \includegraphics[width=0.25\linewidth]{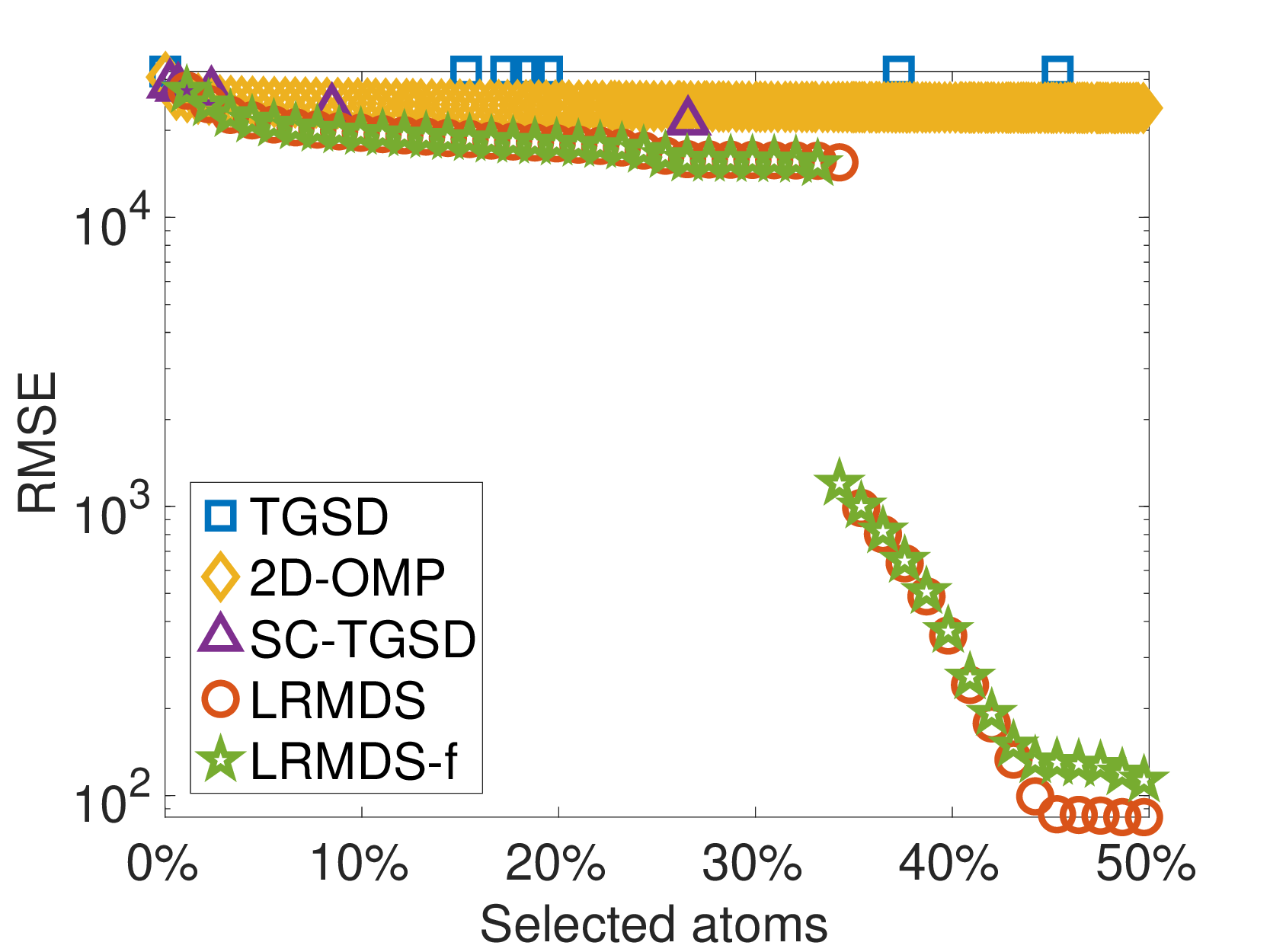}
        \label{fig:covid_rmse_vs_atom}
    }
    \hspace{-0.2in}
    \subfigure [Twitch: Time vs Atom\%]
    {
        \includegraphics[width=0.25\linewidth]{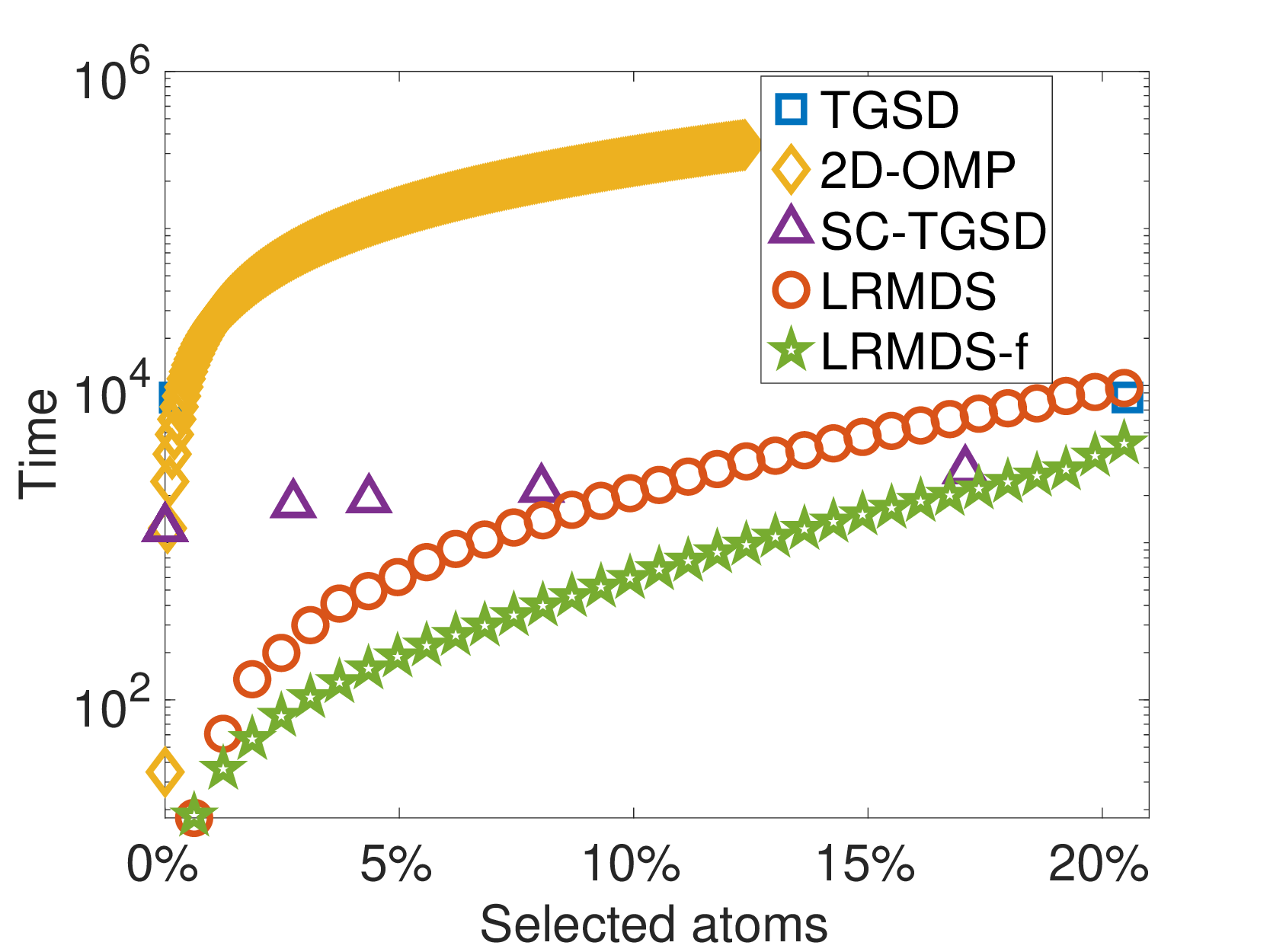}
        \label{fig:twitch_time_vs_atom}
    }
    \hspace{-0.2in}
    \subfigure [Road: Time vs Atom\%]
    {
        \includegraphics[width=0.25\linewidth]{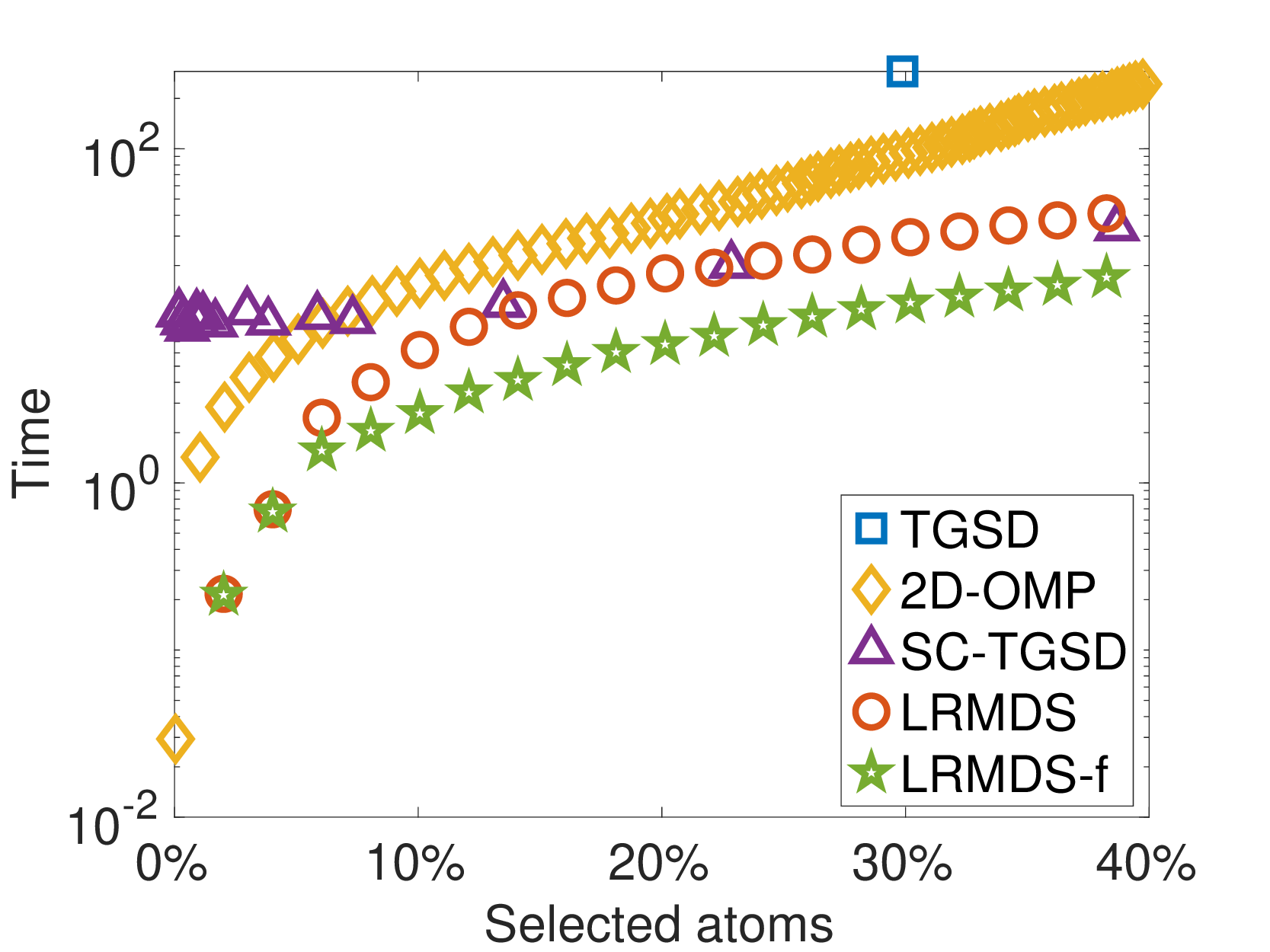}
        \label{fig:road_time_vs_atom}
    }
    \hspace{-0.2in}
    \subfigure [Wiki: Time vs Atom\%]
    {
        \includegraphics[width=0.25\linewidth]{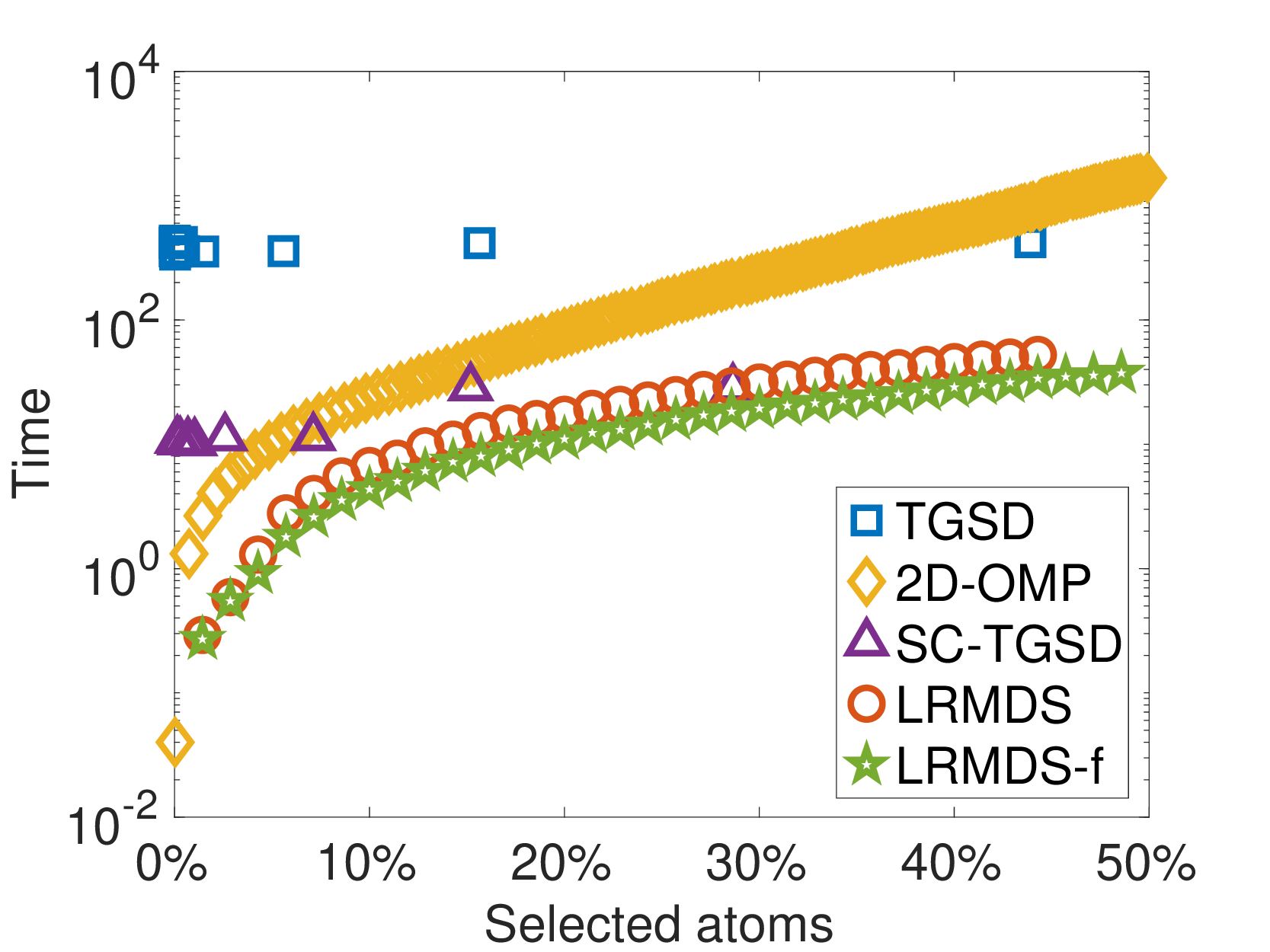}
        \label{fig:wiki_time_vs_atom}
    }
    \hspace{-0.2in}
    \subfigure [Covid: Time vs Atom\%]
    {
        \includegraphics[width=0.25\linewidth]{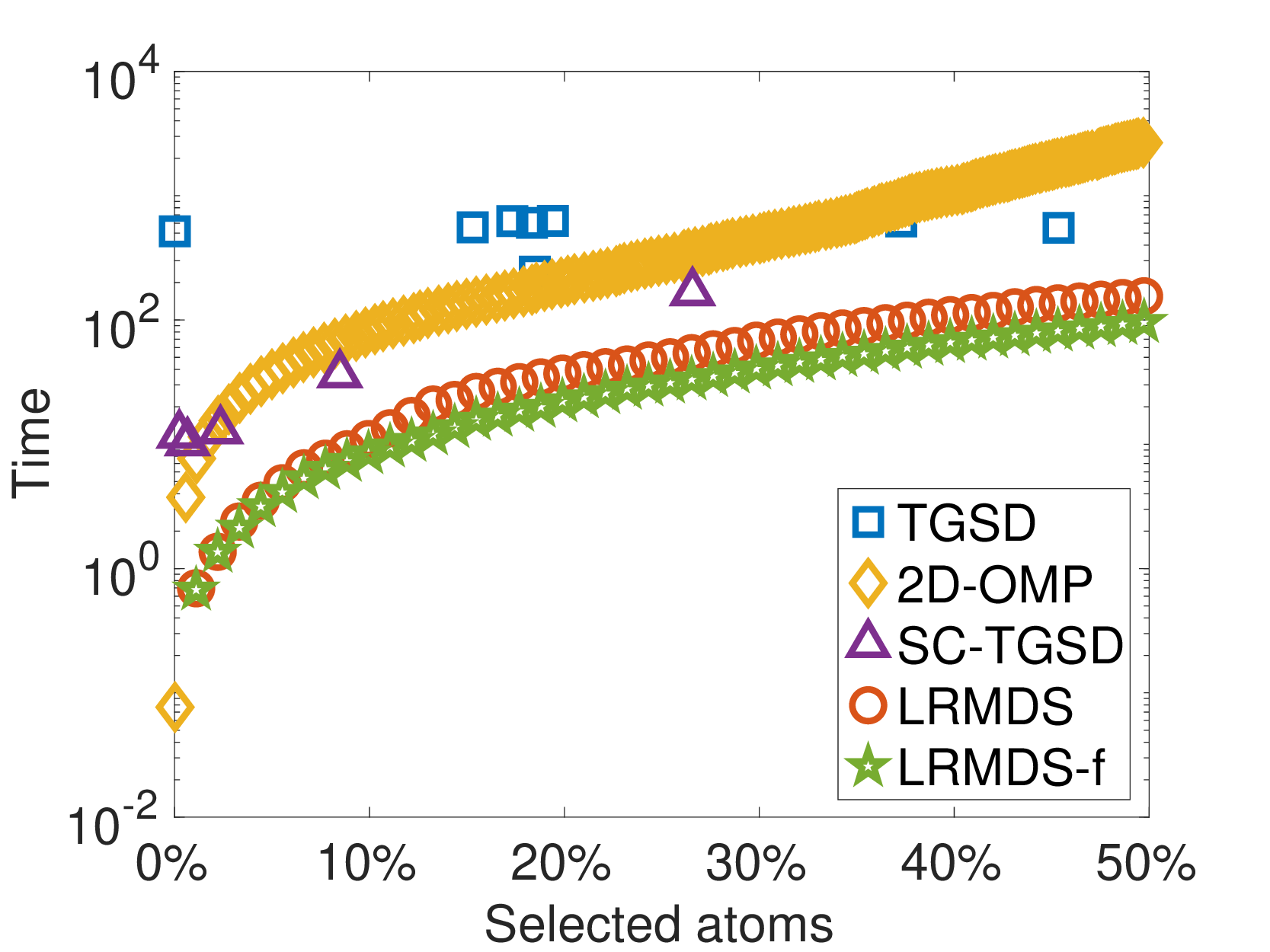}
        \label{fig:covid_time_vs_atom}
    }\vspace{-0.2in}
    \caption{\footnotesize Comparison between competitors of representation quality as a function of the percentage of selected atoms Figs.\subref{fig:twitch_rmse_vs_atom}-\subref{fig:covid_rmse_vs_atom}, and runtime as a function of the percentage of selected atoms Figs.\subref{fig:twitch_time_vs_atom}-\subref{fig:covid_time_vs_atom}. All methods use a GFT for $\Psi$ and a Ramanujan periodic dictionary for $\Phi$. The dimensions of the utilized dictionaries are as follows: Twitch: $\Psi\in \mathcal{R}^{78389 \times 78389}$, $\Phi\in \mathcal{R}^{512 \times 2230}$; Wiki: $\Psi \in \mathcal{R}^{999 \times 999}$,$\Phi \in \mathcal{R}^{792 \times 6000}$; Road: $\Psi \in \mathcal{R}^{1923 \times 1923}$,$\Phi \in \mathcal{R}^{720 \times 3044}$; Covid: $\Psi \in \mathcal{R}^{3047 \times 3047}$,$\Phi \in \mathcal{R}^{678 \times 6000}$. Note: 2D-OMP's trace on the Twitch dataset is truncated early as it does not scale (fails to complete in $72$ hours) when selecting more than $13\%$ of the atoms.}
    \label{fig:real-test}
\end{figure*}

\subsection{Evaluation on real-word datasets.}
We next evaluate all techniques on the real-world datasets and report the RMSE and running time at set 
percentages of total available atoms selected. Results from all datasets for a fixed percentages of atoms are listed in Tbl.~\ref{table:datasets}. This high-level comparison demonstrates that given a fixed number of target atoms, \ourmeth produces the most accurate representations, while \ourmeth-f is the most scalable at the cost of slight deteriorating in RMSE compared to \ourmeth. Note that \ourmeth-f is still the most accurate among baselines from the literature. 

More detailed results on real-world datasets are presented in Fig.~\ref{fig:real-test}. We employ a graph Fourier dictionary (GFT) for $\Psi$ and Ramanujan periodic dictionary for $\Phi$ for all datasets. The sizes of these dictionaries are listed in the caption of Fig.~\ref{fig:real-test}. The detailed analysis also demonstrates that variants of \ourmeth dominate based on both accuracy and running time across a wide variety of settings and datasets. We show the representation error as a function of the percentage of selected atoms in Figs~\ref{fig:twitch_rmse_vs_atom}-\ref{fig:covid_rmse_vs_atom} and the run time necessary to obtain said percentages in Figs.~\ref{fig:twitch_time_vs_atom}-\ref{fig:covid_time_vs_atom}. Together these plots show both the quality of representation and the time necessary to obtain it for a varying percentage of selected atoms (Fig.~\ref{fig:rmse_vs_time} in the supplement explicitly shows this relationship). For each dataset, 2D-OMP selects highly representative atoms at first due to its greedy strategy, however, its trend is quickly overtaken by those of our methods as more atoms are allowed for selection. \ourmeth matches or outperforms \ourmeth-f in terms of representation quality given the same number (percentage) of atoms, however, \ourmeth-f selects new atoms faster demonstrating the trade-off between running time and quality between the two. Both of our methods are as fast or faster than all baselines at selecting atoms with the exception of the 2D-OMP in its first several iterations. The only method matching the speed of \ourmeth is TGSD-SC  (\ourmeth-f is always faster). Although fast, TGSD-CS exhibits poor representation quality rendering it not useful in settings in which the representation quality is critical. For the Twitch datasets when a large percentage of atoms are selected, TGSD is able to obtain similar running time to \ourmeth but at a far worse representation quality. TGSD is not well suited for dictionary sub-selection as sparsity is only implicitly encouraged through $L_1$ regularization over all possible coefficients and it has no direct control on which atoms are used. Thus, even a single nonzero coefficient corresponding to an otherwise poorly selected atom may cause the atom to be ``selected'' by TGSD. 

An interesting finding is that there is an abrupt drop in RMSE in Wiki and Covid data for both variants of \ourmeth. This indicates that the learned representation is initially missing some crucial atoms that \ourmeth is able to eventually detect and incorporate into the selected dictionaries. Competitors omit these crucial atoms in their representations leading to poorer RMSE.  This drop also corresponds to a setting where the difference in quality between \ourmeth and competitors is most striking. For example, in Fig.~\ref{fig:covid_rmse_vs_atom} \ourmeth obtains a roughly two orders of magnitude reduction in RMSE when $50\%$ of the available atoms are selected.




\begin{figure*}[h]
\vsa
   \footnotesize
    \centering 
    \subfigure [\scriptsize clean vs. noise]
    {
        \includegraphics[width=0.23\linewidth]{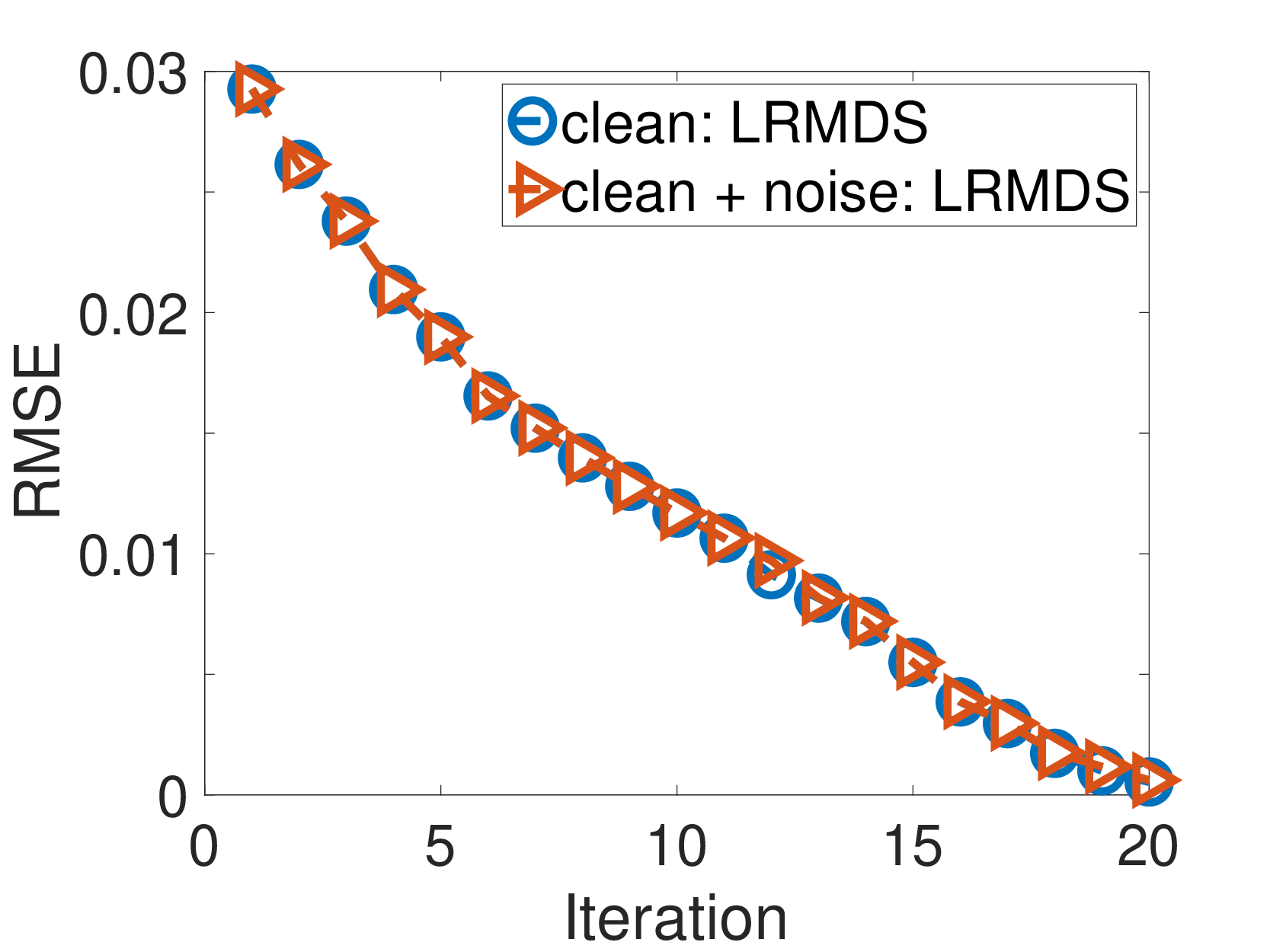}
        \label{fig:convergence_rmse}
    }
    \subfigure [\scriptsize Coef. diff distribution]
    {
        \includegraphics[width=0.23\linewidth]{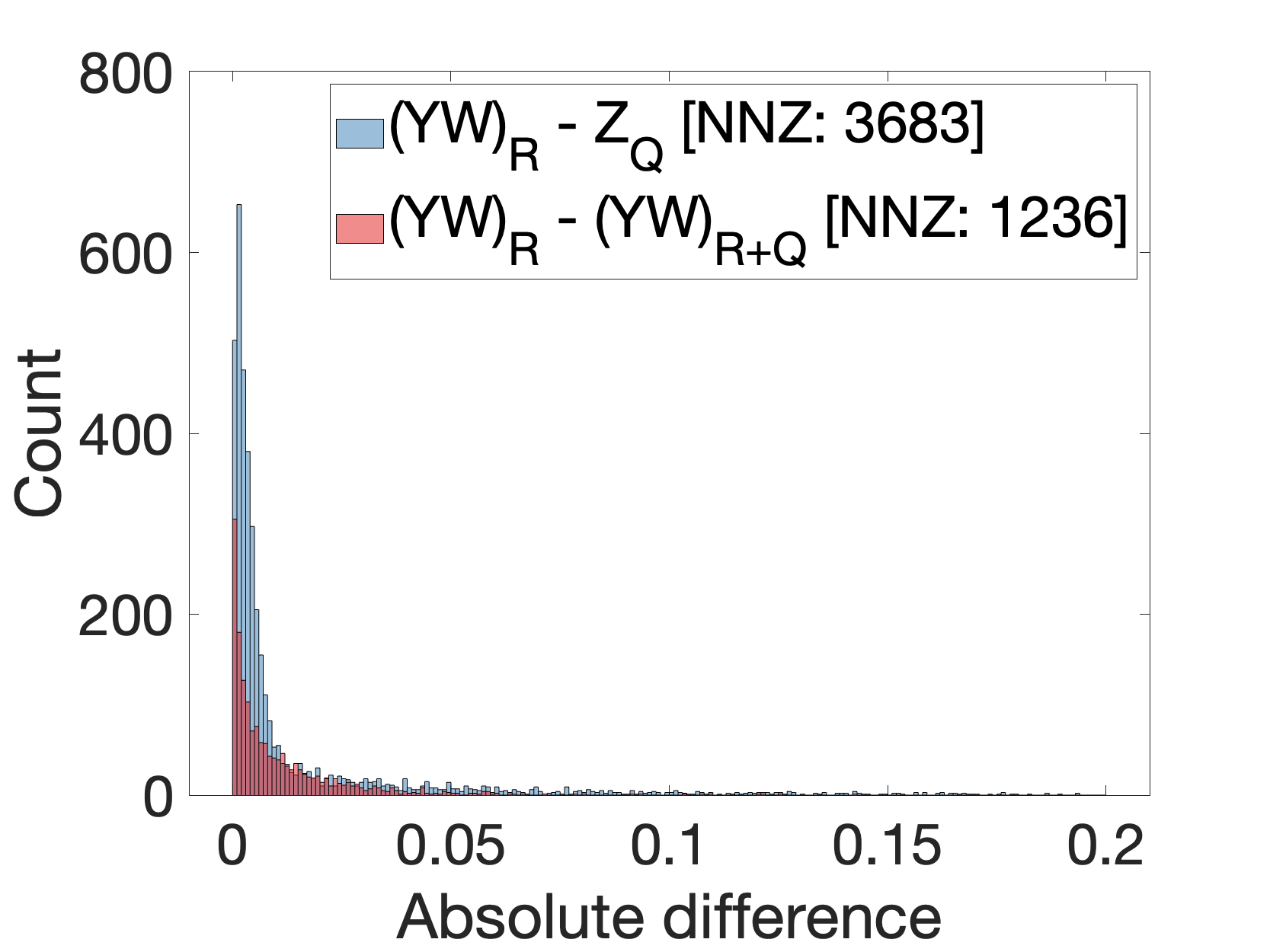}
        \label{fig:hist_YW}
    }
    \subfigure [\scriptsize Ablation: RMSE vs Atom\%]
    {
        \includegraphics[width=0.23\linewidth]{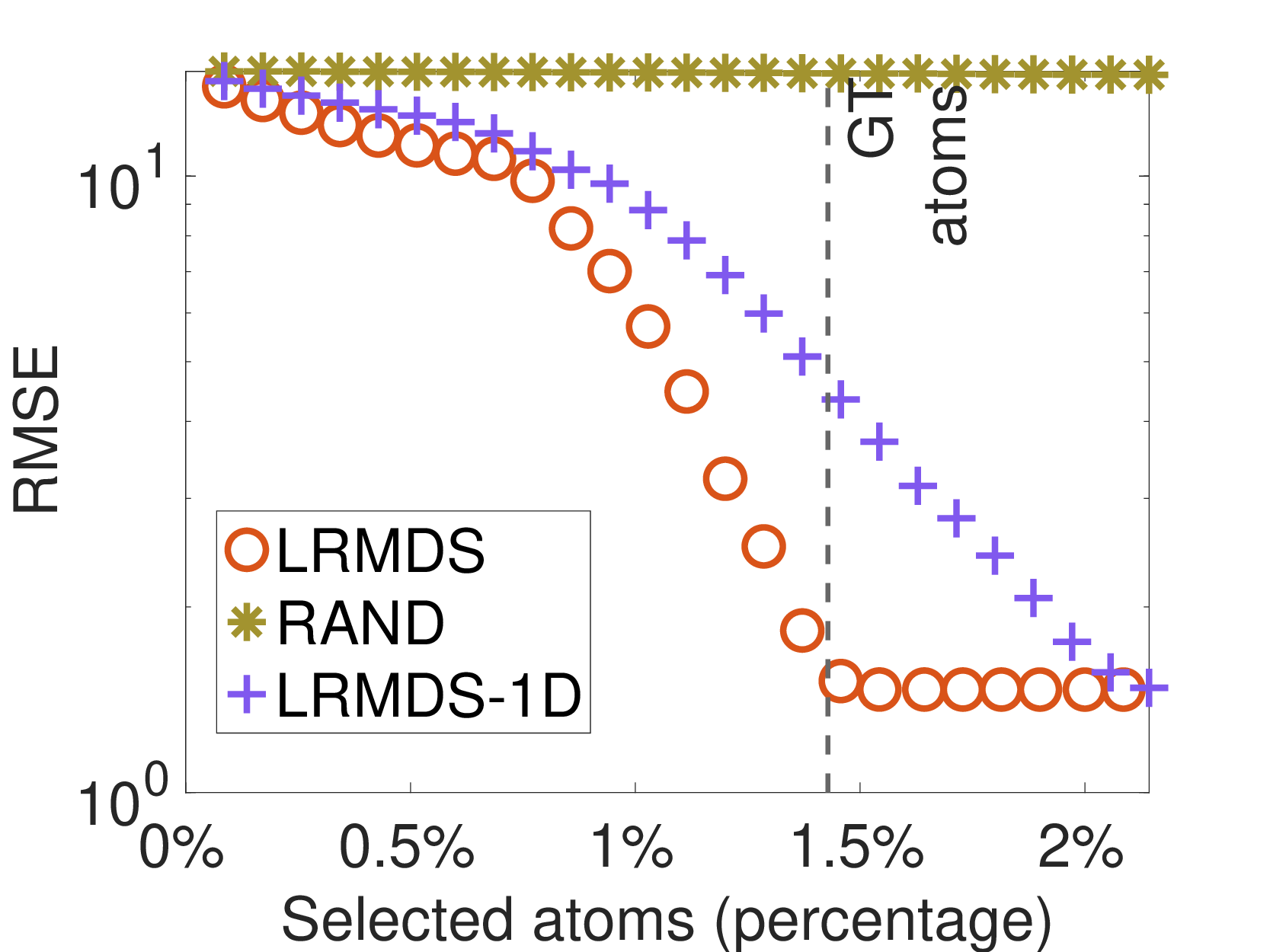}
        \label{fig:ab_rmse_vs_atom}
    }
    \subfigure [\scriptsize Ablation: Time vs Atoms\%]
    {
        \includegraphics[width=0.23\linewidth]{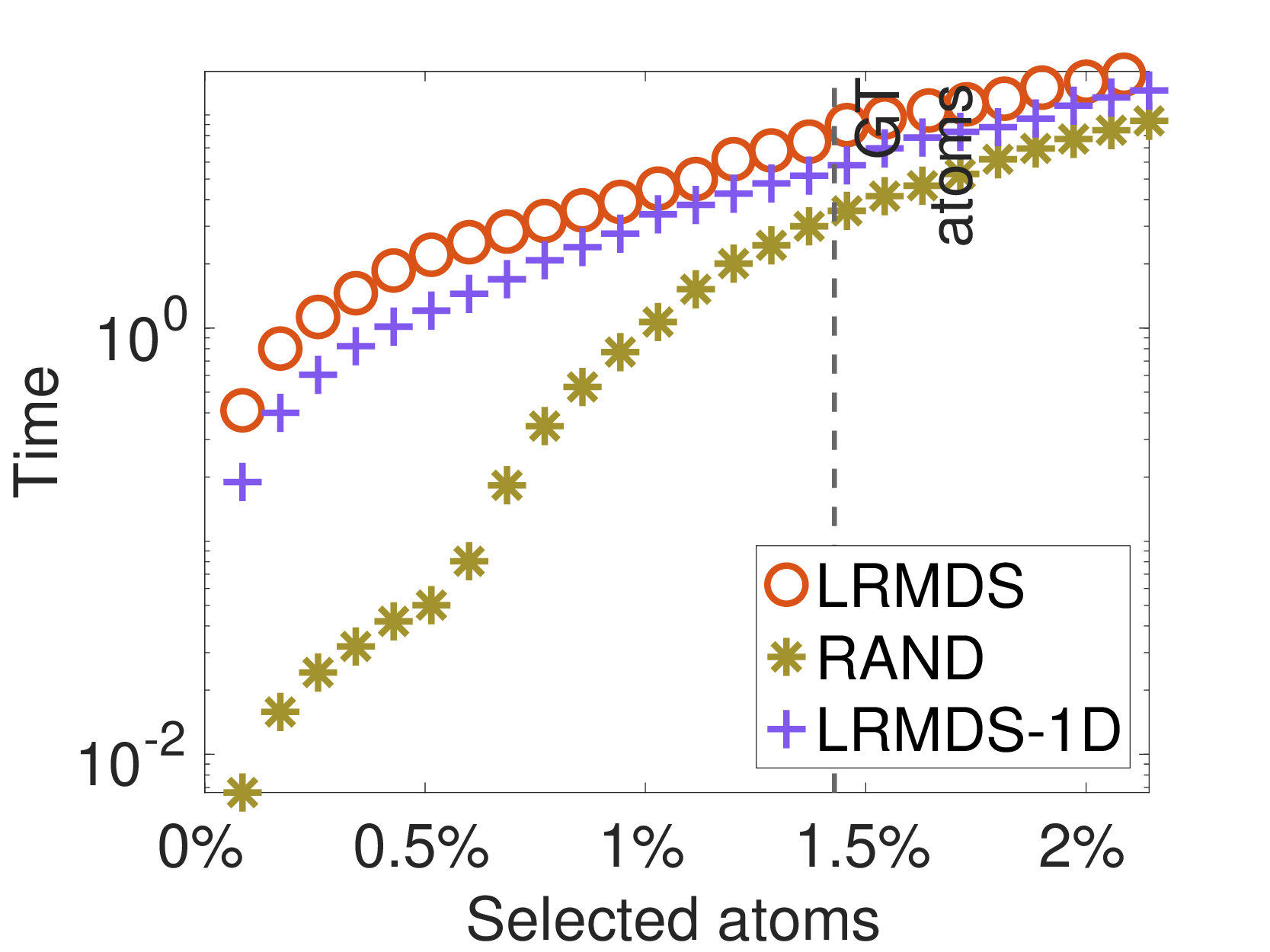}
        \label{fig:ab_time_vs_atom}
    }\vspace{-0.2in}
    \caption{\footnotesize \subref{fig:convergence_rmse}-\subref{fig:hist_YW}: Empirical demonstration of the theoretical guarantee on \ourmeth's ability to denoise a signal. \subref{fig:convergence_rmse}: ``clean: LRMDS'' operates on the clean matrix $R$ whereas ``clean + noise'' operates on the noisy signal $\hat{R}=R+Q$. The RMSE for both methods is measured with respect to the clean data $R$. \subref{fig:hist_YW} The absolute difference between the learned coefficient matrices for the clean data $(YW)_R$, noisy data $(YW)_{R+Q}$, and pure noise $Z_Q$.
    \subref{fig:ab_rmse_vs_atom}\subref{fig:ab_time_vs_atom}: Ablation study demonstrating the importance of joint selection of atoms from both dictionaries. We compare \ourmeth to variants in which atoms are selected from the left and right dictionaries independently (\ourmeth-1D) or randomly (RAND). We measure RMSE~\subref{fig:ab_rmse_vs_atom} and runtime~\subref{fig:ab_time_vs_atom} as a function of the percentage of selected atoms.}
    \label{fig:convergence_test}
\end{figure*}

\subsection{Theoretical guarantees validation (Thm~\ref{thm:convergence-top-k})} \label{sec:theor_exp}

Here, we study the performance of the \ourmeth for denoising which serves as empirical validation of Thm.~\ref{thm:convergence-top-k}. Specifically, we demonstrate that \ourmeth is able to recover the underlying ``clean'' signal $R$ from a noisy signal $\hat{R}=R+Q$ (Fig.\ref{fig:convergence_rmse}). The experimental setup is as follows: $N, M, I, J$ are set to $500, 10, 1000, 20$, respectively. The rank $r$ of the signal is set to $3$. For each dictionary, the first half of its atoms are almost orthogonal (generated as an orthogonal basis with Gaussian noise added to the atoms at SNR=20), and the atoms in the second half of the dictionary are generated as Gaussian $\mathcal{N}(0,1)$ random. The sparsity parameter $s$ is set to $10\%$ of the total number of the almost-orthogonal atoms. The GT atoms for the signal matrix $R$ are chosen uniformly at random from the first half of the atoms (almost-orthogonal), and the atom coefficients are selected randomly ($\mathcal{N}(0,1)$). This constitutes the clean signal matrix $R$. We also create a pure independent Gaussian noise matrix $Q$ by first calculating the standard deviation $\sigma$ of $R$, and set $Q = \mathcal{N}(0, \frac{\sigma}{20})$. Finally, we set $\hat{R} = R + Q$.

To demonstrate \ourmeth's ability to denoise input data $\hat{R}$, we run \ourmeth on both $R$ and $\hat{R}$ producing two sets of coefficients $(YW)_R$ and $(YW)_{R+Q}$. We then track the RMSE of the reconstruction for both variants against the clean data $R$ (i.e., RMSE($R-\Psi (YW)_R \Phi^T $) and RMSE($R-\Psi (YW)_{R+Q} \Phi^T $)). 
Results from this analysis are presented in Fig.~\ref{fig:convergence_rmse}. The curves are nearly identical regardless of the input, demonstrating that \ourmeth successfully extracts the underlying signal while ignoring the noise. To further investigate this property we compare the three different sets of dictionary coefficients corresponding to $R$, $\hat{R}$, and $Q$: $(YW)_R$, $(YW)_{R+Q}$ (as above) and $Z_{Q}$. $Z_{Q}$ contains coefficients computed via 2D-OMP of the noise matrix. We utilize this instead of \ourmeth as due to its low rank constraint it is not capable of well representing an arbitrary noise matrix (as demonstrated above). All methods are run until they converge for their respective inputs. We then calculate the absolute difference in the learned coefficients. Explicitly, $|(YW)_R -(YW)_{R+Q}|$ and $|(YW)_R -(Z)_{Q}|$  and plot the histograms of the nonzero difference values in Fig.~\ref{fig:hist_YW}. While the noise $Z_Q$ and clean data $(YW)_R$ differ significantly (3683 non-zero differences between the two), the fits of the noisy $(YW)_{R+Q}$ and clean $(YW)_R$ data align much better (1236 non-zero differences). The histograms of these differences also indicate that the addition of noise does not significantly impact the coefficients learned by \ourmeth. \vsa

\subsection{Ablation study: Is joint selection critical?}
\ourmeth uses the projection of the residual onto left-right atom pairs (i.e., $P=\hat{\Psi}^TR\hat{\Phi}$) to select atoms. This opens a natural question on the necessity of this technique: 
\emph{Can we select atoms from each of the dictionaries independently employing 1D approaches directly on the left and right dictionary? In other words, is joint selection based on the projection we employ critical?} To answer these questions, we implement two variants of \ourmeth: i) \ourmeth-1D selects atoms from one dictionary at a time via 1D projection, while ii) RAND chooses 2D atoms randomly. We then evaluate their performance on a version of our synthetic dataset with an equal number of ground truth atoms in $\Psi$ and $\Phi$. More details on the implementation and setting for this experiment are available in the supplement. 


In Fig.~\ref{fig:ab_rmse_vs_atom} we plot RMSE of the three variants of our method as a function of the number of selected atoms. \ourmeth approaches its optimal fit (smallest RMSE) when using the ground truth number of atoms. \ourmeth-1D requires more atoms to achieve the same level of RMSE, demonstrating that the joint atom selection is essential for identifying good representative atoms from both dictionaries. The RAND method (random 2D atom selection) is unlikely to select atoms aligned with the data leading to its poor performance. The running time of \ourmeth and \ourmeth-1D are similar with \ourmeth-1D having a slight advantage due to its cheaper selection mechanism and residual re-calculation (Fig.~\ref{fig:ab_time_vs_atom}). For \ourmeth, the projection requires multiplication of complexity $O(min(INM+IMJ$, $INJ+NMJ$), whereas the projection in \ourmeth-1D 
has a complexity of $O(INM+NMJ$). Note that $M<J$ in this experiment, explaining the runtime advantage of \ourmeth-1D. RAND runs much faster at the beginning as there is no projection to select atoms, however, when more atoms are added this advantage shrinks dramatically. This is because the computational complexity quickly becomes dominated by the coefficient updates which take similar time regardless of which atoms are selected.

Another significant weakness of \ourmeth-1D and RAND not highlighted by this experiment is their inability to adaptively select different number of atoms from the right and left dictionaries. The user must specify how many atoms should be selected from each dictionary manually. In contrast, \ourmeth can dynamically select the best atoms from either dictionary in a data-driven manner. Thus, this experiment represents an ideal scenario where a user has correctly identified the proportion of atoms need from $\Psi$ and $\Phi$.

\section{Acknowledgement}
This research was funded by the NSF SC\&C grant CMMI-1831547. AM was funded by NSF CCF grants CIF-2212327 and CIF-2338855.
\vsa
\section{Conclusion}
In this paper we introduced \ourmeth, a scalable and accurate method for sparse multi-dictionary coding of 2D datasets. Our approach sub-selects dictionary atoms and employs convex optimization to encode the data using the selected atoms. We provided a theoretical guarantee for the quality of the atom sub-selection for the task of denoising the data. We also demonstrated the quality and scalability of \ourmeth on several real-world datasets and by employing multiple analytical dictionaries. It outperformed state-of-the-art 2D sparse coding baselines by up to $1$ order of magnitude in terms of running time and up to $2$ orders of magnitude in representation quality on some of the real-world datasets. As a future direction, we plan to extend our dictionary selection approach to multi-way data (i.e., tensors) making the core idea applicable to a wider range of problem settings.

\bibliographystyle{ACM-Reference-Format}
\bibliography{references}

\appendix

\newpage

\section*{Supplemental Material}

In this supplement we include material that could not be included in the main text due to space constraints including proofs of the theoretical results, additional experiments and information supporting reproducibility. The supplement is divided into three main sections: A) proof of Theorem 4.1, B) derivations of the algorithmic steps and experimental details to facilitate reproducibility, and C) additional experimental results. Specifically, we first prove Theorem 4.1 in Sec.~\ref{sec:top-k-proof}. In Sec.~\ref{sec:reproducibility} we add additional implementation details; solutions for \ourmeth, SC-TGSD, \ourmeth-1D, and RAND; detailed description of the datasets and synthetic data protocols; and last but not least, we describe the hyper-parameter tuning procedures and grid search ranges for all methods. Finally, we conclude with more experimental results and figures in Sec~\ref{sec:sup_add_exp}.  


\section{Proofs and supporting numerical experiments.}
\label{sec:top-k-proof}

We first details the proof of our main theoretical result in Sec.~\ref{sec:convergence-top-k-proof} followed by additional numerical experiments supporting the lemmas the overall proof relies on in Sec.~\ref{sec:proof-exp} 

\subsection{Proof of Theorem~\ref{thm:convergence-top-k}}\label{sec:convergence-top-k-proof}

Here we give details for the proof of Theorem~\ref{thm:convergence-top-k}.  Intuitively, the task boils down to showing that the coefficients in any dictionary expansion of the noise matrix $Q$ are uniformly $o(1)$, which allows us to recover the dictionary atoms that contribute to the signal matrix $R$.
As a reminder $Q \in \R^{N\times M}$ with independent and identically distributed standard Gaussian entries.


We start with several lemmas.  
In essence, the first lemma allows us to focus on upper bounding the inner product of the columns of $Q$ with those of $\Psi$ in order to upper bound the coefficients in any dictionary expansion of $Q$.  Before stating it, we note that $\R^{N\times M}$ is an inner product space with inner product $\inner{A}{B} := \sum_{i=1}^N \sum_{j=1}^M A_{i,j} B_{i,j}$.  It is a matter of simple algebra to show the following formula for $\inner{Q}{\psi_{i}\phi^T_{j}}$:
\begin{align}
    \inner{Q}{\psi_{i}\phi^T_{j}}
    = \sum_{k=1}^{M} \phi_{j,k} \cdot \inner{ Q_{\cdot,k} }{ \psi_{i}}.
\end{align}
This implies the following upper bound, since the rows of $\Phi^{T}$ are normalized in $L_2$:
\begin{align}
    \label{expr:inner-product-matrix-bound}
    |\inner{Q}{\psi_{i}\phi^T_{j}}|
    \leq 
    \sqrt{M} \cdot \max_{k \in [M]} | \inner{Q_{\cdot,k}}{\psi_{i}} |,
\end{align}
using the fact that for any vector $x \in \R^d$, $\|x\|_1 \leq \sqrt{d}\|x\|_2$.  In other words, to upper bound the inner product of $Q$ with any dictionary element, it suffices to upper bound the inner product of any column of $Q$ with any element of the left-hand dictionary.


\begin{lemma}[Comparison of inner products with dictionary coefficients]
    \label{lemma:inner-products-suffice}
    Under the boundedness assumptions on dictionary atoms, if an $O(1)$-norm vector $Q \in \R^{N\times M}$ has an expansion
    $Q = \sum_{i=1,j=1}^{I,J} c_{i,j} \cdot \psi_{i} \phi^T_{j}$ for $c_{i,j} \in \R$, then we may upper bound $\max_{i \in [I], j \in [J]} |c_{i,j}|$
    by upper bounding the inner product \\ 
    $\max_{i,j} \inner{Q}{\psi_{i} \phi^T_{j}}$.  Specifically, for all
    $(i, j) \in [I]\times [J]$,
    \begin{align}
        |\inner{Q}{\psi_{i} \phi^T_{j}} - c_{i,j}| = O(M\alpha^2) = o(1).
    \end{align}
\end{lemma}
\begin{proof}
    We note that by linearity of the inner product,
    \begin{align}
        \inner{Q}{\psi_{i} \phi^T_{j}}
        = c_{i,j} + \sum_{(k,\ell) \neq (i, j)} c_{k,\ell} \inner{ \psi_{k}\phi^T_{\ell} }{ \psi_{i}\phi^T_{j} }.
    \end{align}
    The coefficients $c_{k,\ell}$ are uniformly $O(1)$ by virtue of $Q$ having norm $O(1)$, so this simplifies
    to
    \begin{align}
        \inner{Q}{\psi_{i} \phi^T_{j}} - c_{i,j} 
        &= O(1) \cdot \sum_{(k,\ell) \neq (i, j)} \inner{ \psi_{k}\phi^T_{\ell} }{ \psi_{i}\phi^T_{j} } \\
        &= O(1) \sum_{(k,\ell) \neq (i, j)} \inner{ \psi_{k}}{\psi_{i}} \inner{\phi^T_{\ell} }{\phi^T_{j}} \\
        &\leq O(1)\sum_{(k,\ell) \neq (i, j)} \inner{ \psi_{k}}{\psi_{i}} \\
        &\leq O(M\alpha^2).
    \end{align}

\end{proof}
Lemma~\ref{lemma:inner-products-suffice} implies that we may upper bound the coefficients of an expansion of a matrix $Q$ using the inner products of $Q$ with dictionary elements.  The upper bound (\ref{expr:inner-product-matrix-bound}) allows us to further upper bound $\inner{Q}{\psi_{i} \phi^T_{j}}$, reducing our problem
to upper bounding the entries of a multivariate Gaussian random variable.  Specifically, if we denote by $K \in \R^{M\times I}$ the matrix $K = Q^T \Psi$, then $K_{i,j} = \inner{Q_{\cdot, i}}{\psi_{j}}$.  We then have that
\begin{align}
    \max_{i,j} |c_{i,j}| 
    \leq  \sqrt{M} \max_{i,j} |K_{i,j}| + o(1).
\end{align}

The next lemma gives us a tool to upper bound $\max_{i,j} |K_{i,j}|$ by using the fact that the covariance of
$K_{i,j}$ and $K_{i,\ell}$, for $\ell \neq j$, is equal to $\inner{\psi_{j}}{\psi_{\ell}}/\sqrt{NM} = \Sigma_{j,\ell}/\sqrt{NM}$.  In other words, the vector $\hat{K}$ obtained by appending the columns of $K^T$ into a column vector of dimension $M\cdot I$ has distribution $\Normal(0, \hat{\Sigma})$, where $\hat{\Sigma} \in \R^{MI\times MI}$ and satisfies
$\|\hat{\Sigma}^{1/2} \|_{op,\infty} = \frac{1}{\sqrt{N}} \cdot \| \Sigma^{1/2}\|_{op,\infty}$.

\begin{lemma}[Upper bound on the maximum of correlated Gaussians]
    \label{lemma:aux-max-of-gaussians}
    Let $X \sim \Normal(0, \Sigma)$ be a Gaussian vector in $\R^n$ with covariance matrix $\Sigma \in \R^{n\times n}$. Then we have that
    \begin{align}
        \E[\|X\|_{\infty}] = O(\|\Sigma^{1/2}\|_{op,\infty} \cdot \sqrt{\log n}).
    \end{align}
\end{lemma}
\begin{proof}
    This is a consequence of a well-known upper bound on the maximum of independent and identically distributed standard normal random variables, along with the fact that, for an isotropic, mean $0$ Gaussian vector $Z$, $\Sigma^{1/2}Z$ has covariance matrix $\Sigma$.
\end{proof}

\begin{lemma}[Upper bound on the maximum inner product between a noise vector and a dictionary atom]
    \label{lemma:gaussian-dictionary-inner-product-bound}
    Consider the matrix $K = Q^T \cdot \Psi \in \R^{M\times I}$ whose $(i,j)$th entry is the inner product
    of the $i$th column of $Q$ with the $j$th dictionary element of $\Psi$.  We have that with high probability,
    \begin{align}
        \max_{i,j} |K_{i,j}| = O(\sqrt{\log(MI)}/\sqrt{N}).
    \end{align}
\end{lemma}
\begin{proof}
    We start with Lemma~\ref{lemma:aux-max-of-gaussians} applied to $\hat{K}$.  This yields
    \begin{align}
        \E[\|\hat{K}\|_{\infty}] 
        &= O(\|\hat{\Sigma}^{1/2}\|_{op,\infty} \cdot \frac{\sqrt{\log(MI)}}{\sqrt{NM}}) \\
        &= O(\frac{\sqrt{\log(MI)}}{\sqrt{N}} \cdot \| \Sigma^{1/2}\|_{op,\infty}).
    \end{align} 
    The proof is finished by applying Markov's inequality.
\end{proof}
A corollary of Lemma~\ref{lemma:gaussian-dictionary-inner-product-bound} is that the maximum inner product between $Q$ and the dictionary elements is $o(1)$ with high probability.  This implies, by Lemma~\ref{lemma:inner-products-suffice}, that the coefficients of $Q$ in any of its dictionary expansions are uniformly $o(1)$.  Because of the sparsity assumption on the coefficients of $R$, the data matrix $\hat{R}$ has $s$ coefficients that are $\Theta(1)$, while the rest are $o(1)$.  Thus, provided that the dictionary atom selection procedure selects $\hat{k} \geq s$ atoms with $\hat{k} = \Theta(1)$, those atoms for which $R$ has nonzero coefficients will be among those selected.  As a result, the reconstructed matrix $R_{reconst}$ differs from the signal matrix $R$ by only $o(\|R\|_{F})$.
This completes the proof of Theorem~\ref{thm:convergence-top-k}.

We note that only a minor tweak of the above proof is needed to extend to the case where dictionary selection is iteratively applied for a fixed number $t \geq s/\hat{k}$ of steps, each time to the residual of the previous step.  Specifically, in order to formulate this, we need more notation.  Suppose, as before, that $R$ is a linear combination of $s$ atoms, each with coefficient uniformly $\Theta(1)$.  Suppose that $\hat{k} < s$.
Let $R_{\leq \hat{k}}$ denote the truncation of $R$ to its top $\hat{k}$ atoms (i.e., those with the largest coefficients in absolute value), and let $R_{> \hat{k}} := R - R_{\leq \hat{k}}$.  That is, $R_{> \hat{k}}$ is the residual of the signal matrix after subtracting $R_{\leq \hat{k}}$.  Finally, we define $R_{reconst,\leq \hat{k}}$ to be the output of the algorithm after a single iteration.
The proof of our theorem so far showed that $R_{reconst,\leq \hat{k}} - R_{\leq \hat{k}} = o(R_{\leq \hat{k}})$.  This implies the following:
\begin{align}
    \hat{R} - R_{reconst,\leq \hat{k}}
    &= (Q + R_{\leq \hat{k}} + R_{> \hat{k}}) - R_{reconst,\leq \hat{k}} \\
    &= (R_{> \hat{k}} + Q) + (R_{\leq \hat{k}} - R_{reconst,\leq \hat{k}}) \\
    &= (R_{> \hat{k}} + Q) + o(R_{\leq \hat{k}}).
\end{align}
We note that $\hat{R} - R_{reconst,\leq \hat{k}}$ is the residual after applying a single iteration of the top-$\hat{k}$ atom selection algorithm.  After at least $t-1$ applications of the algorithm to the residual matrix of the previous step, the final residual matrix consists of fewer than $\hat{k}$ nonzero atoms, plus Gaussian noise.  This satisfies the hypotheses of our theorem statement.  Since the sparsity parameter $s$ is assumed to be $\Theta(1)$, the total accumulated error over all steps of the algorithm is $o(R_{\leq \hat{k}})$, which is $o(R)$ in the Frobenius norm.

\subsection{Lemma-supporting numerical experiments.}\label{sec:proof-exp}

To empirically demonstrate the bound in Lemma~\ref{lemma:gaussian-dictionary-inner-product-bound} we generate random Gaussian noise matrices $Q$ of size $N \times M$ for increasing $N$ and while keeping $M$ set to $1000$. For each generated matrix we learn an encoding matrix $Z$ via 2D-OMP as in Sec.~\ref{sec:theor_exp}. We then plot the max coefficient for each $Z$ in Fig.~\ref{fig:noise_max_coef}. We can clearly see from the figure that as $N$ grows the maximum coefficient shrinks, thus empirically confirming the bound from Lemma~\ref{lemma:gaussian-dictionary-inner-product-bound}.

\begin{figure}[h]
   \footnotesize
    \centering 
    \includegraphics[width=0.4\linewidth]{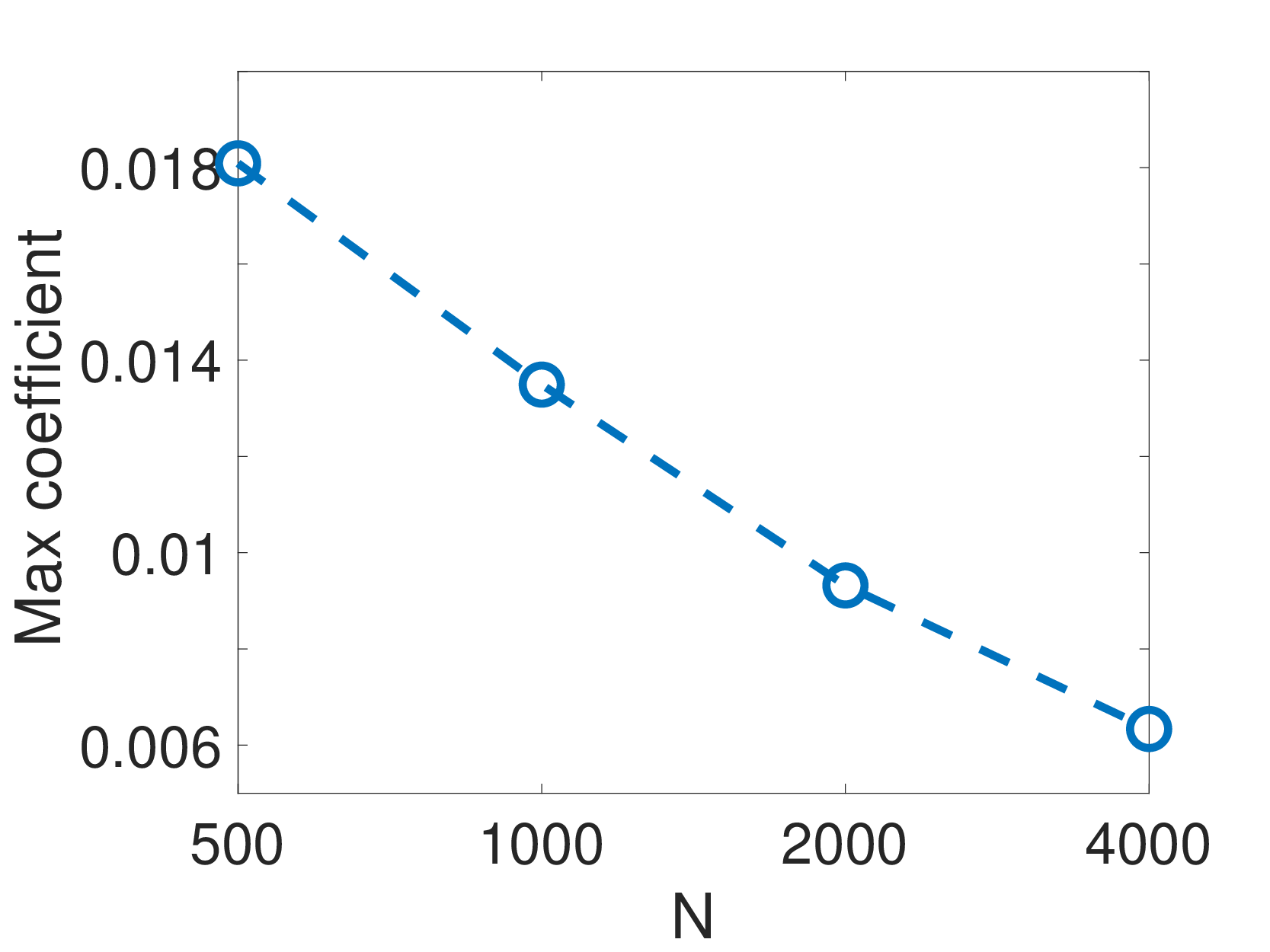}
    \caption{\footnotesize Max coefficient of the learned coefficient matrix of the noise data while N increases.}
       \label{fig:noise_max_coef}
\end{figure}

\section{Reproducibility} \label{sec:reproducibility}

To facilitate reproducibility we discuss further details of the derivation of \ourmeth in Sec.~\ref{sec:sup_LRMDS_sol}, and baseline methods in Sec.~\ref{sec:baseline_sol_det}. We also add details on our synthetic data generation and pre-processing performed to real-worlds dataset in Sec.~\ref{sec:sup_data}. Finally, we describe how parameters were tuned for all methods in Sec.~\ref{sec:sup_param_set}. 

\subsection{\ourmeth Solution Details}\label{sec:sup_LRMDS_sol}

\ourmeth has two key steps: atom selection and encoding.
To perform atom selection we quantify the alignment of each 2D atom with the current residual via projection. We can compute $R$'s projection as follows:
\begin{equation}
    \label{eq:align}
    P_{i,j} = \frac{\langle R, B_{i,j} \rangle}{|| B_{i,j} ||_F},
\end{equation}
where $\langle R, B_{i,j} \rangle \triangleq \psi_i^T R \phi_j$ is the alignment, and $|| B_{i,j}||_F$ is a normalization based on the Frobenius norm of the 2D atom product. Intuitively atoms of good alignment will be advantageous for encoding the data. Instead of utilizing Eq.~\ref{eq:align} in the algorithm for \ourmeth we perform the functionally equivalent projection via $P=\hat{\Psi}^TR\hat{\Phi}$ employing the normalized dictionaries $\hat{\Psi}$ and $\hat{\Phi}$.
Due to the normalization, the denominator from Eq.~\ref{eq:align} can be omitted since:
$$||B_{i,j}||_F =|| \hat{ \psi_i}||_2 \cdot || \hat{\phi_j}||_2=  1  \hspace{.1pt} , \forall i,j.$$ 
This in turn allows to us utilize simply matrix multiplication $\hat{\Psi}^TR\hat{\Phi}$ to obtain alignment scores for atoms.

We then  select the top $k$ total atoms from a combination of left or right dictionary atoms with respect to this alignment. To illustrate the methodology, suppose $k = 3$, and the top alignments correspond to $P_{2,3},P_{3,3}$ in descending order. We would then add the atoms $\psi_2$, $\phi_3$, $\psi_3$ and in that order to our sub-dictionaries we call $\Psi_s\in \mathcal{R}^{N \times I_s} $ and $\Phi_s\in \mathcal{R}^{M \times J_s}$, where $I_s, J_s$ are the number of selected atoms from the left and the right dictionaries.  It is important to note that we only add atoms if they don't already exists in our selected sub-dictionary. Importantly this may result in uneven selection from the dictionaries (i.e $I_s \neq J_s$). This is desirable as there may be significantly more complexity in one of $X$'s modes, necessitating more atoms from the corresponding dictionary for good representation. Intuitively, we let the data guide the selection on both sides. Ties between atoms are resolved arbitrarily.

Once we have sub-selected the dictionary via these chosen atoms we need to solve for the encoding coefficients in $Y$ and $W$ by solving the following:

\begin{equation}
    \begin{aligned}
        \underset{Y,W} {\mathrm{argmin}} \hspace{0.1cm} & || R - \Psi_s Y W \Phi_s^T ||_F^2,
    \end{aligned} 
\end{equation}

To achieve this we iteratively alternate through solving for $Y$ and $W$ while the other is fixed.
The updates in each case can be derived by taking the gradients with respect to the non-fixed variable, setting them to $0$, and solving. This results in the following update rules:

(1) Given $W$, the update rule for $Y$ is as follows: 
\begin{equation}
    \begin{aligned}
    \Psi_s ^\dagger \Psi_s  Y  W \Phi_s^T  (W \Phi_s^T )^\dagger &=
      \Psi_s ^\dagger   R (W \Phi_s^T )^\dagger \\
      Y &= \Psi_s ^\dagger   R (W \Phi_s^T )^\dagger,
    \end{aligned}
    \label{eq:solve_Y}
\end{equation}

where $\dagger$ denotes the pseudo-inverse of the corresponding matrix.

(2) Given $Y$, the update rule for 
$W$ is:
\begin{equation}
    \begin{aligned}
      (\Psi_s  Y)^\dagger \Psi_s  Y  W \Phi_s^T (\Phi_s^\dagger)^T   &= (\Psi_s  Y )^\dagger  R (\Phi_s^\dagger)^T   \\
      W &= (\Psi_s  Y )^\dagger   R (\Phi_s^\dagger)^T
    \end{aligned}
    \label{eq:solve_W}
\end{equation}

Note that at every iteration, the update rules for the two variables require four pseudo-inversions solved via singular value decomposition with per-iteration complexity of $O(min(mn^2,m^2n))$, where $m, n$ are the size of the target matrix. The dictionary inversions $\Psi_s^\dagger$ and $\Phi_s^\dagger$ can be computed only once per decomposition as they are fixed with respect to $Y$ and $W$. 
Thus, the overall complexity of these steps assuming  $N>I_s$, and $M>J_s$ is 
$O(NI_s^2)$ for $\Psi_s^\dagger$ and   $O(MJ_s^2)$ for $\Phi_s^\dagger$. We need to compute  $(\Psi_s  Y )^\dagger$ and  $(\Phi_s  W )^\dagger$ for 
every iteration, thus assuming $q$ iterations to convergence the total complexity of the coding step is $O(q(N+M)r^2 + MJ_s^2 + NI_s^2)$ assuming the selected decomposition rank is lower than the corresponding data dimensions, i.e., $N>r$ and $M>r$.

We can further optimize the run time based on the assumption that the inversions of both products $(W \Phi_s^T )^\dagger$ and $(\Psi_s  Y )^\dagger$ involve matrices of full column (left matrix in the product) and row (right matrix) rank. Then we can separate the inversions and open more opportunities for savings by using the following matrix product inversion rule due to~\cite{greville1966note}:

\begin{equation}
    \begin{aligned}
        (A B)^\dagger = B^\dagger A^\dagger
    \end{aligned}
    \label{eq:inverse}
\end{equation}

Updates from Eq.~\ref{eq:solve_Y} and Eq.~\ref{eq:solve_W} can then be rewritten as:

\begin{equation}
    \begin{aligned}
      Y = \Psi_s ^{\dagger}  X  (\Phi_s ^\dagger)^T W^\dagger
    \end{aligned}
    \label{eq:solve_Y_fast}
\end{equation}

\begin{equation}
    \begin{aligned}
        W = Y^\dagger   \Psi_s ^\dagger X (\Phi_s^\dagger)^T, 
    \end{aligned}
    \label{eq:solve_W_fast}
\end{equation}
where $\Psi_s ^{\dagger}  X  (\Phi^\dagger)^T$ is a common term that can be pre-computed outside of the iterative updates. This enables us to only need to compute the pseudo-inversion of $Y$ and $W$ within the inner-loop. $Y^\dagger$ and $W^\dagger$ have complexity $O(r^2I_s)$ and $O(r^2J_s)$ respectively. Reducing the overall complexity is to $O(q(I_s+J_s)r^2 + MJ_s^2 + NI_s^2)$ which is linear with respect to the size of input matrix. We name this faster \ourmeth variation method \ourmeth-f.
Note that when the conditions for Eq.~\ref{eq:inverse} are not met, our encoding will be not as accurate in this variant, but we demonstrate experimentally that this alternative solver offers a good runtime-quality trade-off.  




\noindent \textbf{Dictionaries for \ourmeth:} 
Our framework is designed to flexibly accommodate any dictionaries $\Phi$ and $\Psi$ for encoding.
In our experimental evaluation we utilize a variety of commonly used dictionaries for spatio-temporal datasets following the ones adopted in~\cite{TGSD}. Namely, the graph Fourier transform (GFT) \cite{shuman2013emerging}, and Graph-Haar Wavelets \cite{crovella2003graph} for modes corresponding to nodes in a network. For modes corresponding to temporal samples we utilize two alternative temporal dictionaries: the Ramanujan periodic dictionary~\cite{tennetiTSP2015} and a Spline dictionary~\cite{Goepp2018Spline}. For a concise summary of these dictionaries we refer the reader to~\cite{TGSD}. 


\subsection{Baseline Solution Details} \label{sec:baseline_sol_det}
For TGSD and 2D-OMP we employ the implementation provided by the authors. However, SC-TGSD which combines 1D dictionary screening and TGSD and the ablation study variants of our method: \ourmeth-1D, and RAND are novel formulations which require changes in the implementation. The details of the latter implementations are described next.

\noindent \textbf{SC-TGSD screening process:} SC-TGSD screens (removes) the worst dictionary atoms from a dictionary by calculating alignment scores between atoms and associated data. It then removes atoms whose alignment falls below a preset threshold. Formally the screening process produces a subselected dictionary:  $\Psi_s= \Psi \setminus{ 
 \{ (x^T \psi_i) \leq \lambda, \forall i \in I)\}} $. This is effectively the Sphere Test 2 from \cite{xiang2016screening}. The screening process in SC-TGSD is similar to OMP in how it calculates its alignment scores.  To extend the screening to 2D data we simply vectorize the input $X$ and all pairwise 2D atoms $\psi_i \phi_j^T$ and use the original screening method from~\cite{xiang2016screening} with the resulting vectors. To reduce running time, similar to our approach, instead of computing all inner products, we use the equivalent alignment computed as $\hat{\Psi}^T X \hat{\Phi}$, and select the atoms from either dictionary which don't fall below the set threshold.

\noindent \textbf{\ourmeth-1D and RAND. }
The implementations of \ourmeth-1D and RAND are essentially identical to \ourmeth with the exception of the sub-dictionary selection steps.  Specifically, in \ourmeth-1D to choose the atoms of sub-dictionary $\Psi_s$ the residual is projected on only $\Psi$ ($P_1 = \Psi^T R$). Then rows of $P_1$ are ranked by total energy, the top $k_1$ indices are determined and the associate atoms added to the set of included atoms in sub-dictionary $\Psi_s$. A similar process is used to find $\Phi_s$: Calculate $P_2 = R \Phi^T$, select top $k_2$ column indices from  $P_2$, and add the associated atoms to sub-dictionary $\Phi_s$. This approach is also similar to performing generalized OMP atom selection for each dictionary separately~\cite{tropp2005simultaneous,wang2012generalized}. 
For the RAND method in each iteration, we randomly select $k_1$ atoms from $\Psi$, and randomly select $k_2$ atoms from $\Phi$.  Once the sub-dictionaries are selected in this fashion the method proceeds with encoding in the same manner as \ourmeth.

\subsection{Datasets Generation and Pre-processing}
\label{sec:sup_data}

\noindent\textbf{Synthetic data generation:}
Unless otherwise noted in specific experiments, the variables in our model for synthetic data are set to the following: $\Psi_s$ corresponds to $20$ randomly chosen atoms from a GFT dictionary which itself is generated from a Stochastic Block Model (SBM) graph with $3$ blocks of equal size and $1000$ total nodes and internal and cross-block edge probabilities set to $0.2$ and $0.02$ respectively. $\Phi_s$ contains $20$ randomly selected atoms from a Ramanujan periodic dictionary. The entries of $Y$ and $W$ are set to uniformly random numbers between $0$ and $1$ and the rank $r$ is set to $3$. Finally, $\epsilon$ is Gaussian white noise with magnitude ensuring an overall $SNR=10$ for the signal.

\noindent\textbf{Real-world dataset:} 
The original Twitch dataset specifies active viewers over time and the streams that they are viewing. We create a graph among viewers, where an edge between a pair of viewers exists if they viewed the same stream at least 3 times over a period of $512$ hours (which is the temporal dimension of the dataset). The largest connected component of this co-viewing graph involve $78,389$ viewers. Each entry of the data matrix $X\in\mathcal{R}^{78389 \times 512}$ from Twitch represents the number of minutes in any given hour that the viewer spent viewing streams on the platform. The Wiki dataset captures hourly number of views of Wikipedia articles for a total of $792$ hours. We construct a graph among the articles by placing edges between articles with at least $10$ pairwise (clicked by the same IPs) click events within a day. Furthermore, we pick a starting node (the Wikipedia article on China) and construct a breadth-first-search (snowball) subgraph of $1000$ nodes around it. We removed an article that was not sufficiently active during the observed period resulting in $999$ total nodes. The Covid dataset tracks daily confirmed COVID cases for $3047$ counties in the US for $678$ days. We use a k-nearest neighbor ($k=5$) spatial graph connecting counties to their closest neighbors. The Road dataset consists of $1923$ highway speed sensors in the LA area, we use the hourly average speed for $30$ days as our signal matrix, and the graph is based on connected road segments.



\begin{table*}[ht]
\footnotesize
\setlength\tabcolsep{3 pt}
\centering
 \begin{tabular}{|c|c|c|c|c|c|c|c|c|c|} 
 \hline
 Method & Parameters &  Range & Synthetic & Convergence & Ablation & Twitch & Road & Wiki & Covid\\
 \hline
 TGSD  & $\lambda_1, \lambda_2$ & $[10^{-3}, 10^{-2}, \cdots, 10^{3}]$ & Vary & Vary & Vary & Vary & Vary & Vary & Vary\\
 \hline
2D-OMP  & $T_0$ & $3 - 100\% ~\# \text{atoms}$ & $3.5\%$ & $100\%$ & NA & $13\%$ & $40\%$ & $50\%$ & $50\%$\\
 \hline
 SC-TGSD  & $\frac{\lambda_0}{\lambda_{0_{max}}}, \lambda_1, \lambda_2$ & $(0.1:0.01:0.9);[10^{-3}, 10^{-2}, \cdots, 10^{3}]$ & Vary & NA & NA & Vary & Vary & Vary & Vary\\
 \hline
 \ourmeth  & $k$ & $[5, 6, 10, 100, 500]$ & 5 & 10 & 6 & 500 & 100 & 100 & 100\\
 \hline
 \ourmeth-f  & $k$ & $[5, 6, 10, 100, 500]$ & 5 & 10 & 6 & 500 & 100 & 100 & 100\\
 \hline
\end{tabular}
\caption{\footnotesize Parameters for competing methods where $\lambda_1, \lambda_2$ are sparsity parameters for TGSD; $T_0$ is the targeting number of coefficients for 2D-OMP;  $\frac{\lambda_0}{\lambda_{0_{max}}}$ is the regularizer for SC; $k$ is the number of atoms selected per iteration for \ourmeth and \ourmeth-f. Some methods are not included in the convergence and ablation experiments and the corresponding cells are marked as NA for Not Applicable. Ranges for tested values are listed in the Range column.}
\label{table:method_params}
\end{table*}

\subsection{Hyper-parameter Settings} \label{sec:sup_param_set}

The parameter settings for all competing techniques unless otherwise specified are as follows. We set the rank for low rank decomposition methods (\ourmeth, \ourmeth-f, TGSD, SC-TGSD, \ourmeth-1D, RAND ) to be $r=3$ in synthetic (equal to the ground truth) and $r=50$ in real-world datasets. For all real-world datasets, we set the number of atoms per iteration for \ourmeth and \ourmeth-f to be $k=100$ for all experiments except for Twitch in which $k=500$ since this dataset is large and the input dictionaries have in total close to $80,000$ atoms. 
In both synthetic and real world datasets, we vary SC-TGSD's screening parameter $\frac{\lambda_0}{\lambda_{0_{max}}}$ in the range of $0.1$ to $0.9$ with a step size of $0.01$ to create a set of regimes of selected sub-dictionaries and associated RMSE/running time. 

Controlling the number of selected atoms exactly for all competitors is not trivial as TGSD and SC-TGSD employ sparsity regularizers ($\lambda_1, \lambda_2$) that do not offer explicit control over the number of used atoms. In order to facilitate direct but fair comparison we use the ground truth number of selected atoms in synthetic datasets as targets for 2D-OMP and \ourmeth, and report results for TGSD and SC-TGSD using sparsity levels resulting in atom ``selection'' closest to (but exceeding) the ground truth number. 

In our ablation study we set $k_1,k_2=3$ for both \ourmeth-1D and RAND, and $k=6$ for \ourmeth. We do this so that the total number of atoms selected by each method is $6$ at each iteration to facilitate fair comparison. We re-run RAND $10$ times and report its average performance in terms of RMSE and running time. 

Complete details on how parameters were searched (i.e., ranges) and set for each dataset can are listed in Tbl.~ \ref{table:method_params}.

\begin{figure}[t]
   \footnotesize
    \centering
    \subfigure [RMSE ]
    {
        \includegraphics[width=0.45\linewidth]{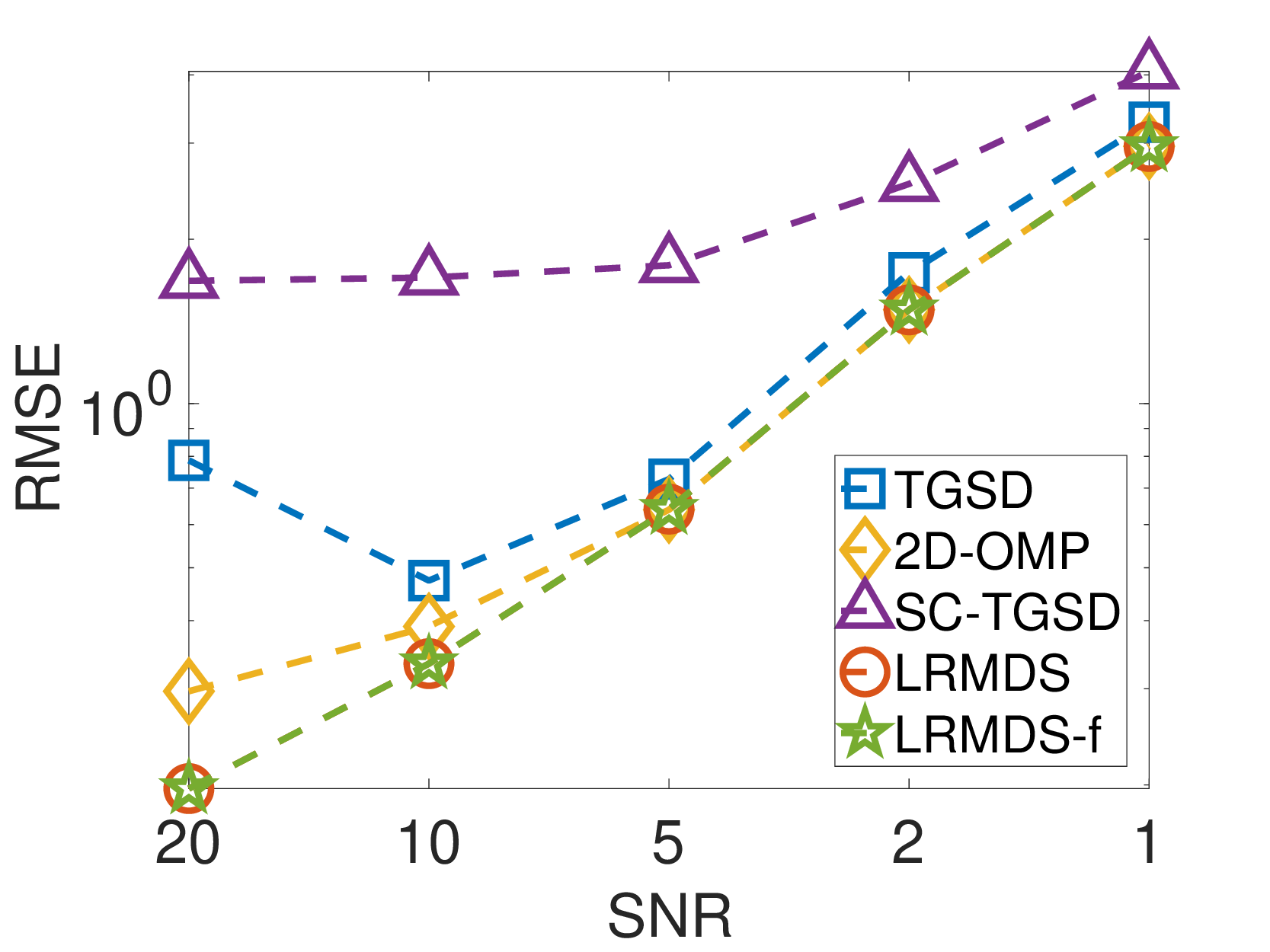}
        \label{fig:syn_snr_rmse}
    }
    \subfigure [Time]
    {
        \includegraphics[width=0.45\linewidth]{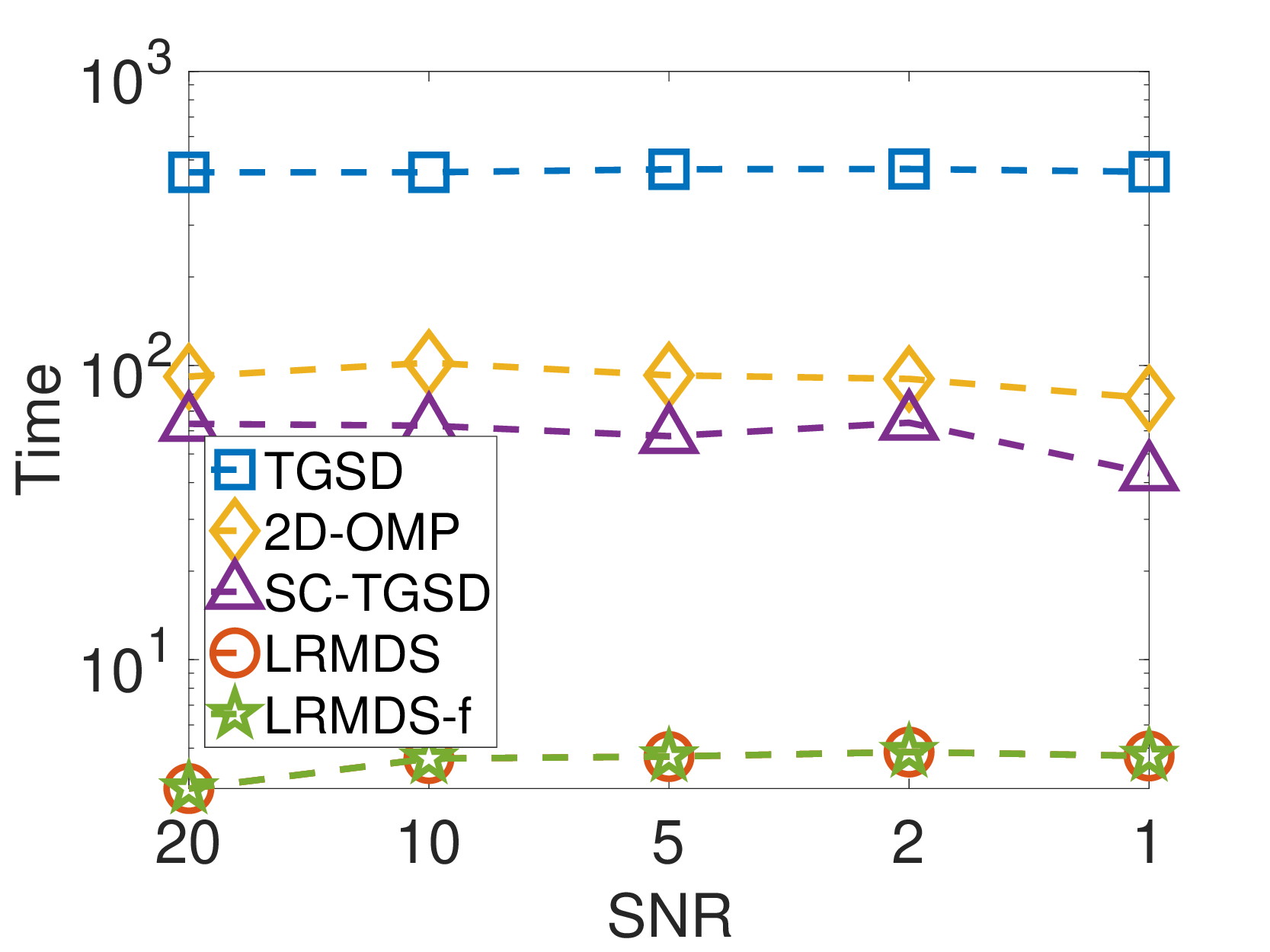}
        \label{fig:syn_snr_time}
    }
    \vsa
    \caption{\footnotesize Comparison of competing techiques' RMSE~\subref{fig:syn_snr_rmse} and running time~\subref{fig:syn_snr_time} for varying signal-to-noise-ratio (SNR) on synthetic data.}
    \label{fig:snr_trend}
\end{figure}

\begin{figure}
   \footnotesize
    \centering 
  
    \subfigure [\scriptsize G+R: RMSE vs Atom\%]
    {
        \includegraphics[width=0.45\linewidth]{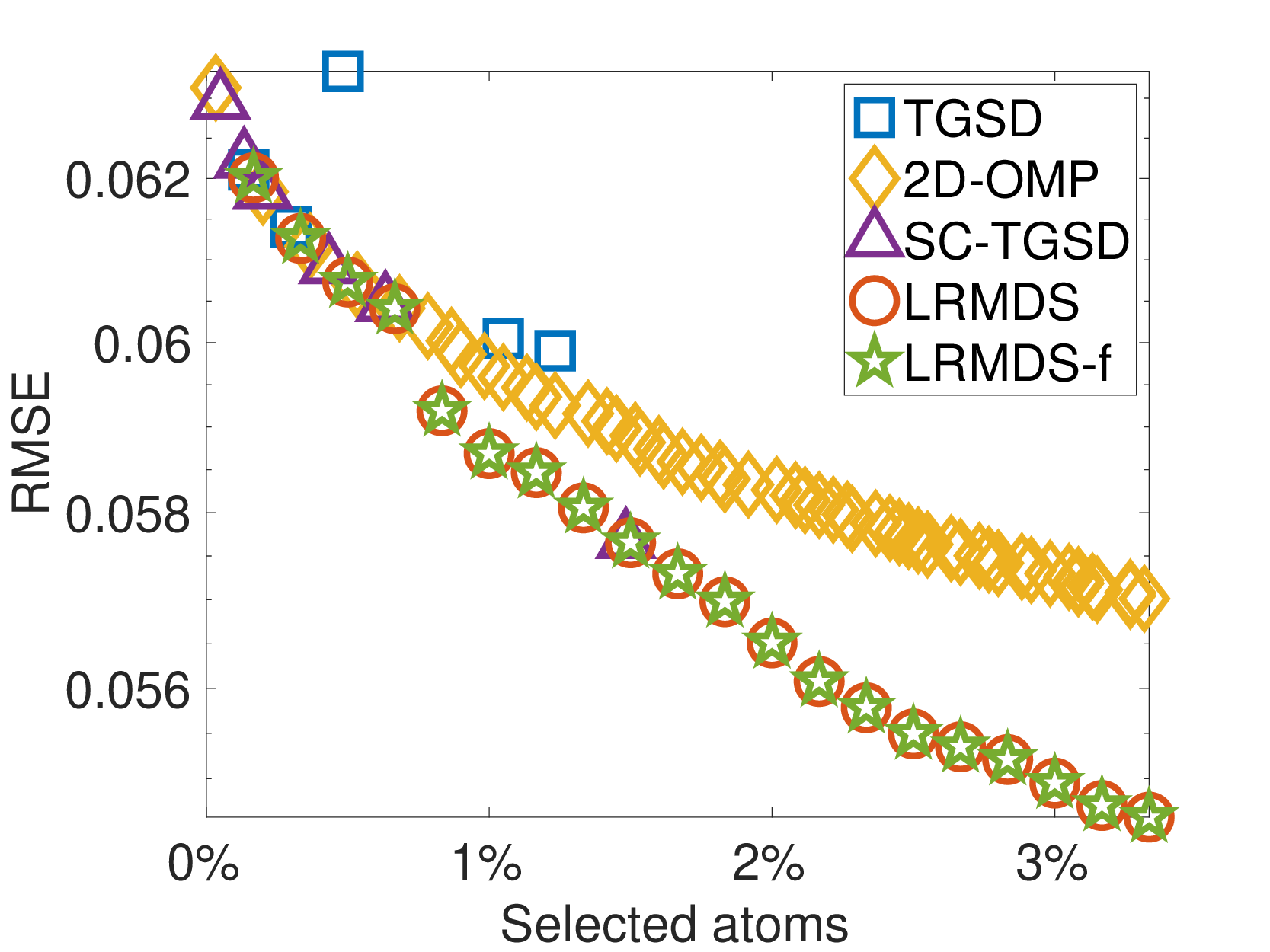}
        \label{fig:GR_rmse_vs_atom}
    }
    \subfigure [\scriptsize GW+R: RMSE vs Atom\%]
    {
        \includegraphics[width=0.45\linewidth]{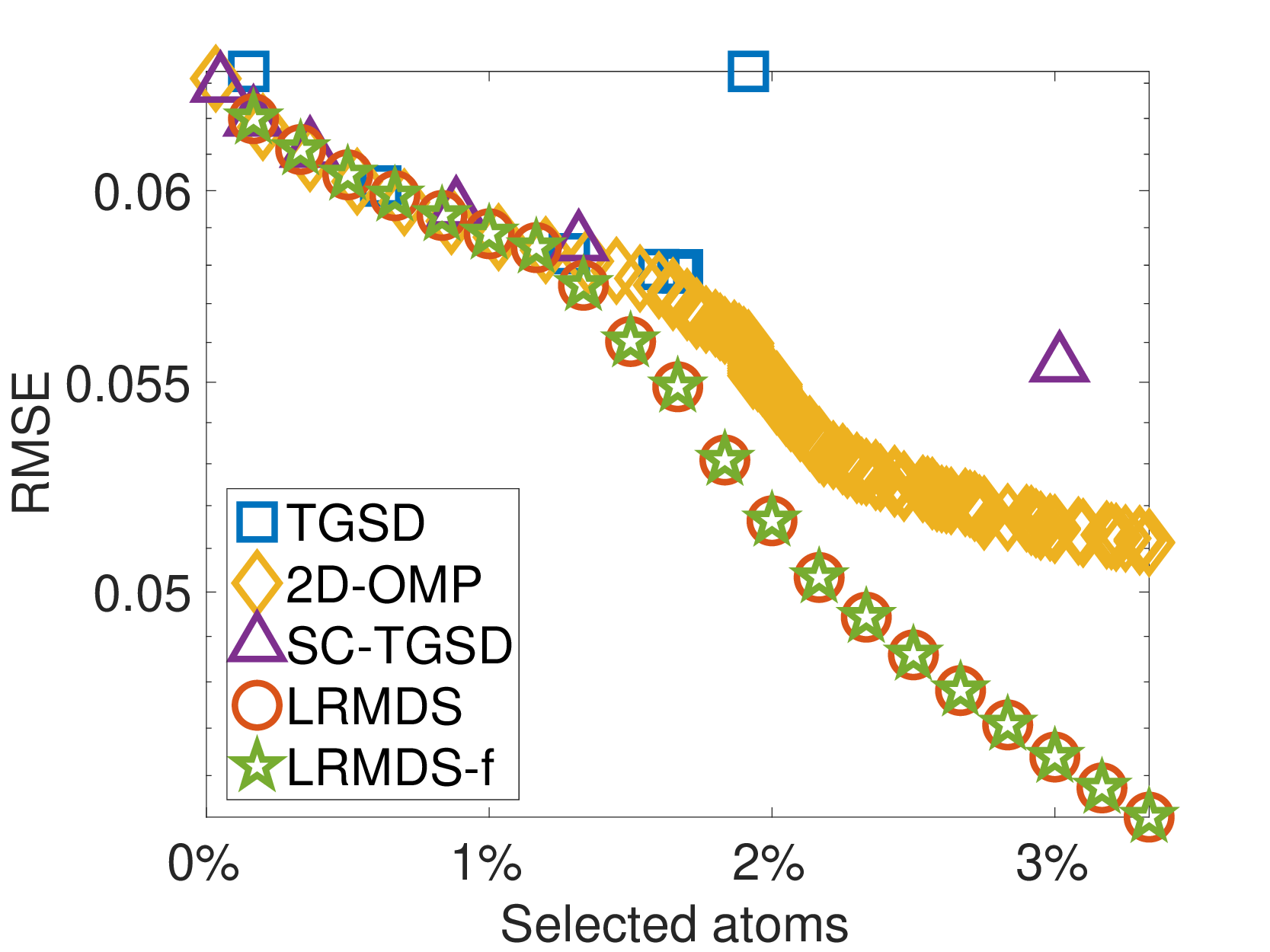}
        \label{fig:GWR_rmse_vs_atom}
    }
    \subfigure [G+R: Time vs Atom\%]
    {
        \includegraphics[width=0.45\linewidth]{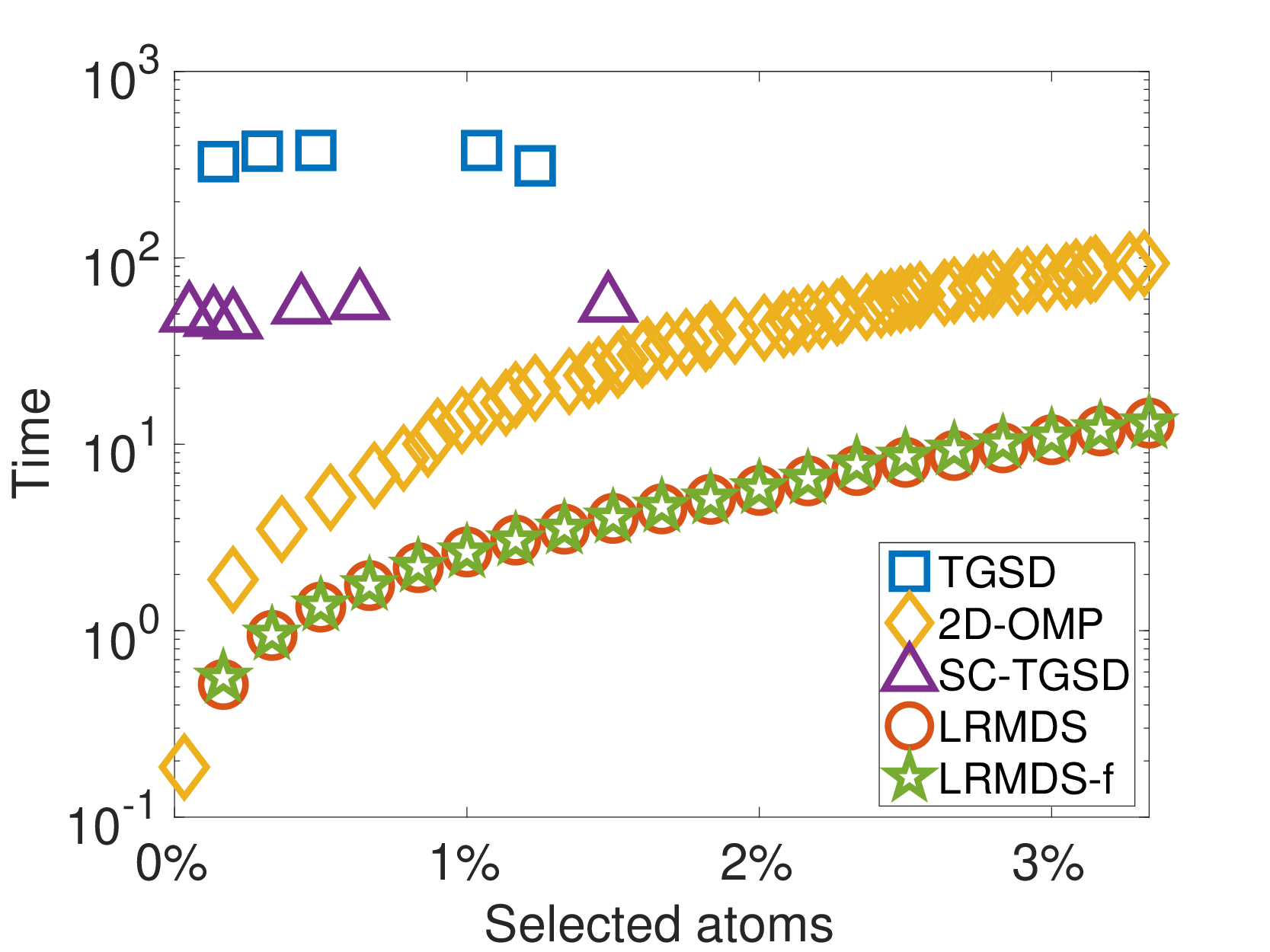}
        \label{fig:GR_time_vs_atom}
    }
    \subfigure [GW+R: Time vs Atom\%]
    {
        \includegraphics[width=0.45\linewidth]{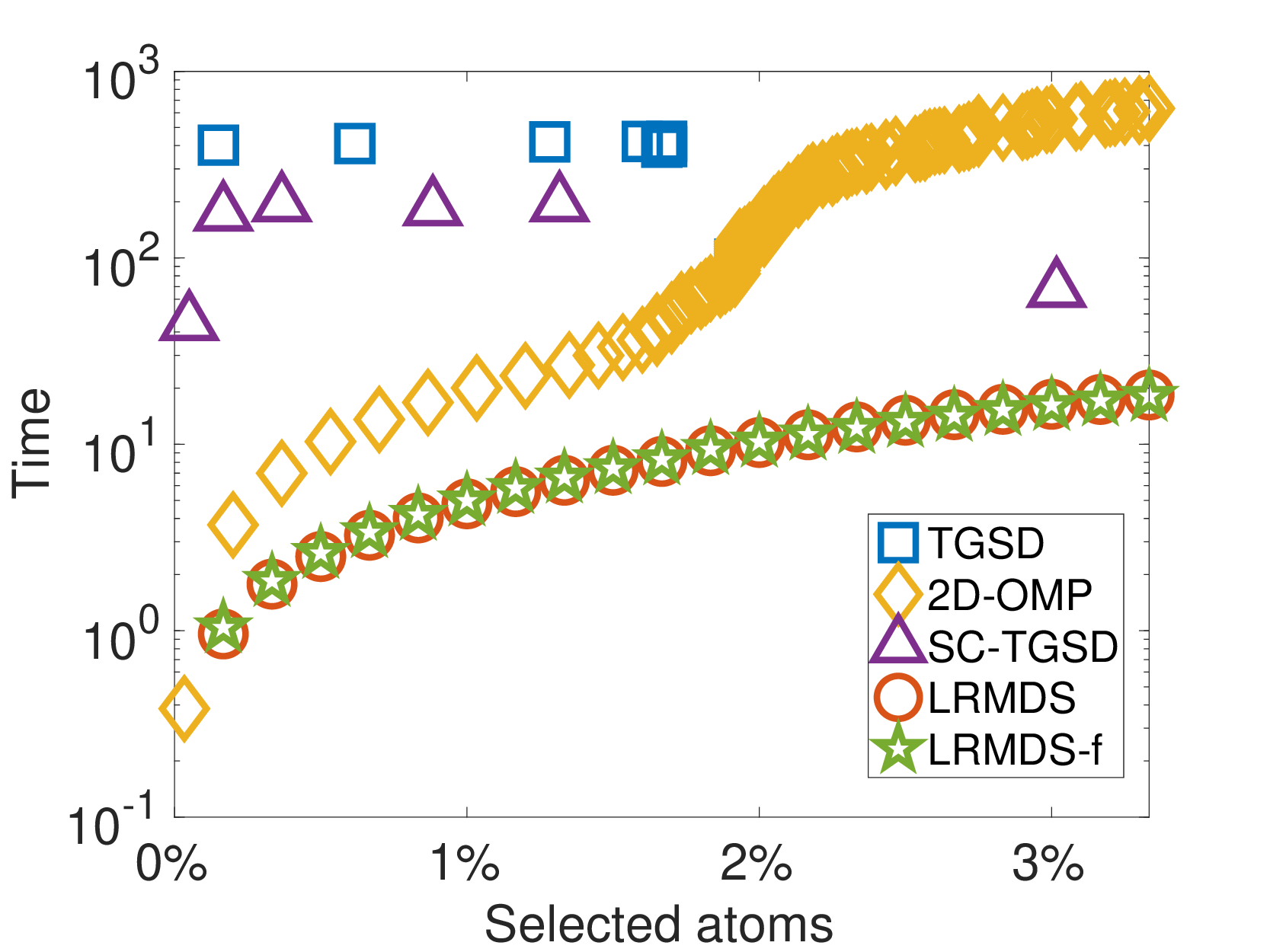}
        \label{fig:GWR_time_vs_atom}
    }

    \caption{\footnotesize Comparison between competing methods measuring, representation quality per percentage of atoms selected Fig.\subref{fig:GR_rmse_vs_atom}\subref{fig:GWR_rmse_vs_atom}, and run time per percentage of atoms selected Fig.\subref{fig:GR_time_vs_atom}\subref{fig:GWR_time_vs_atom}.}
    \label{fig:dict_types_sup}
\end{figure}

\begin{figure*}[th]
   \footnotesize
    \centering 
  \subfigure [G+R: RMSE vs Time]
    {
        \includegraphics[width=0.25\linewidth]{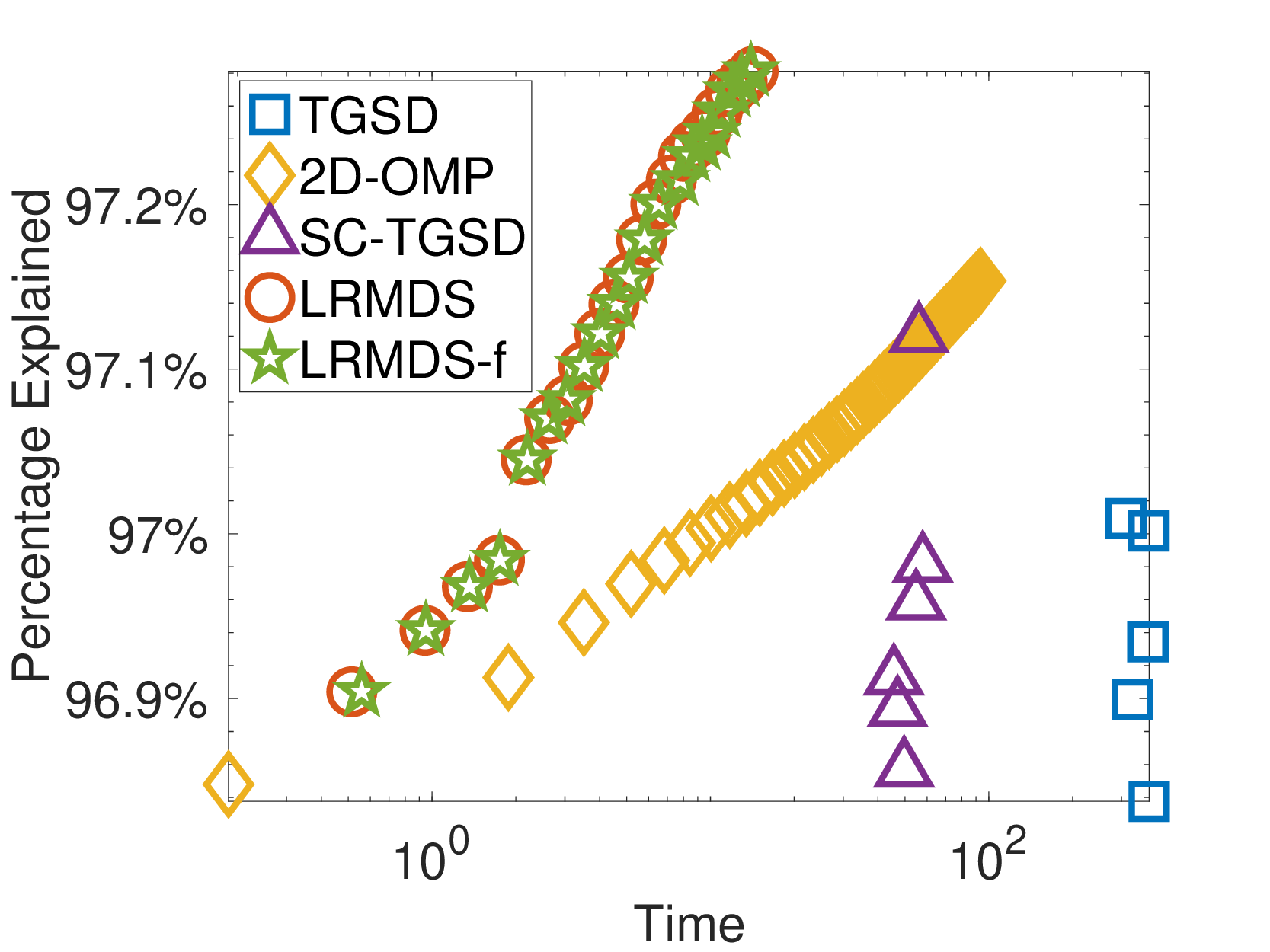}
        \label{fig:GR_rmse_vs_time}
    }
        \hspace{-0.2in}
    \subfigure [GW+R: RMSE vs Time]
    {
        \includegraphics[width=0.25\linewidth]{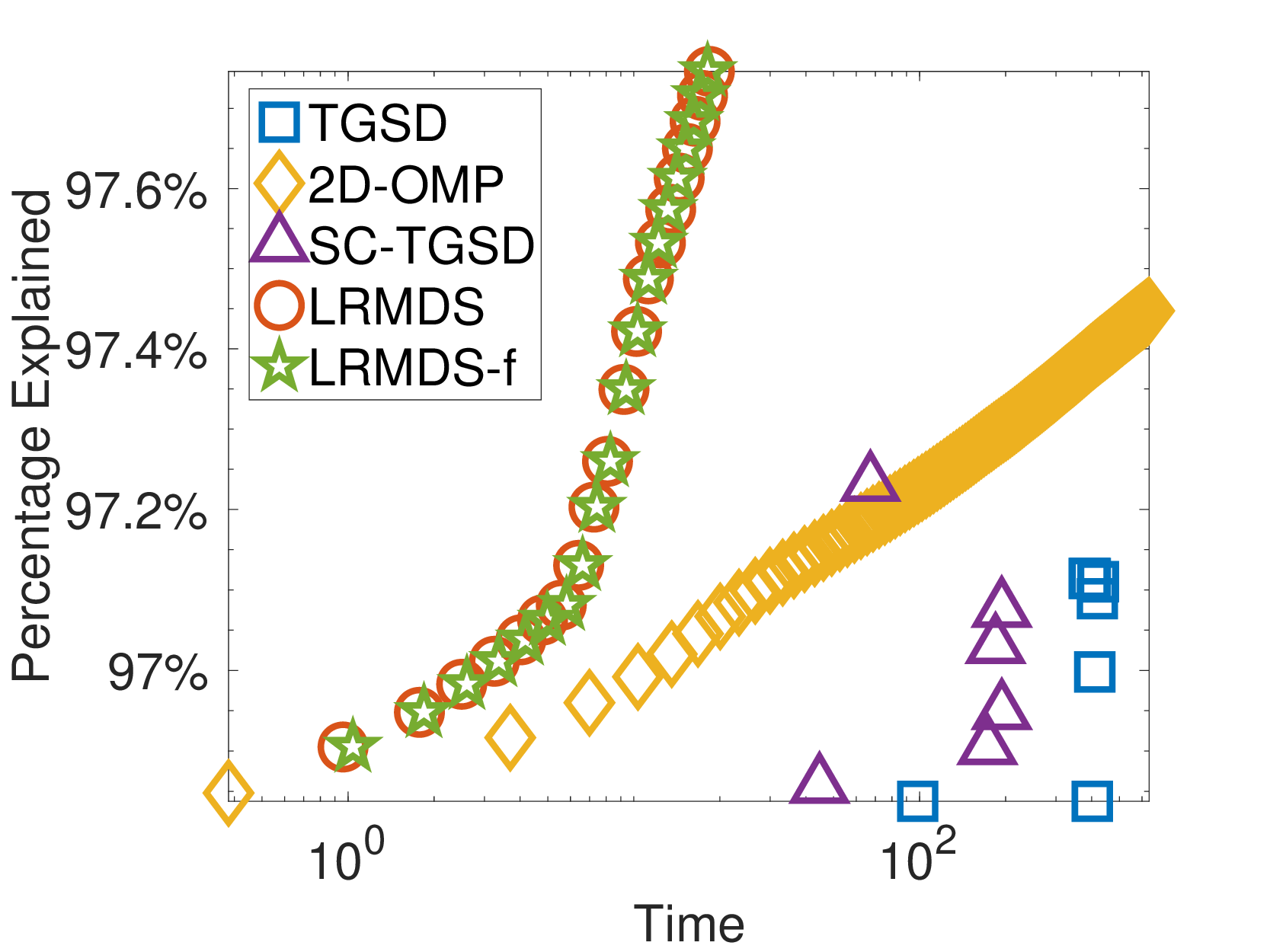}
        \label{fig:GWR_rmse_vs_time}
    }
        \hspace{-0.2in}
    \subfigure [GW+RS: RMSE vs Time]
    {
        \includegraphics[width=0.25\linewidth]{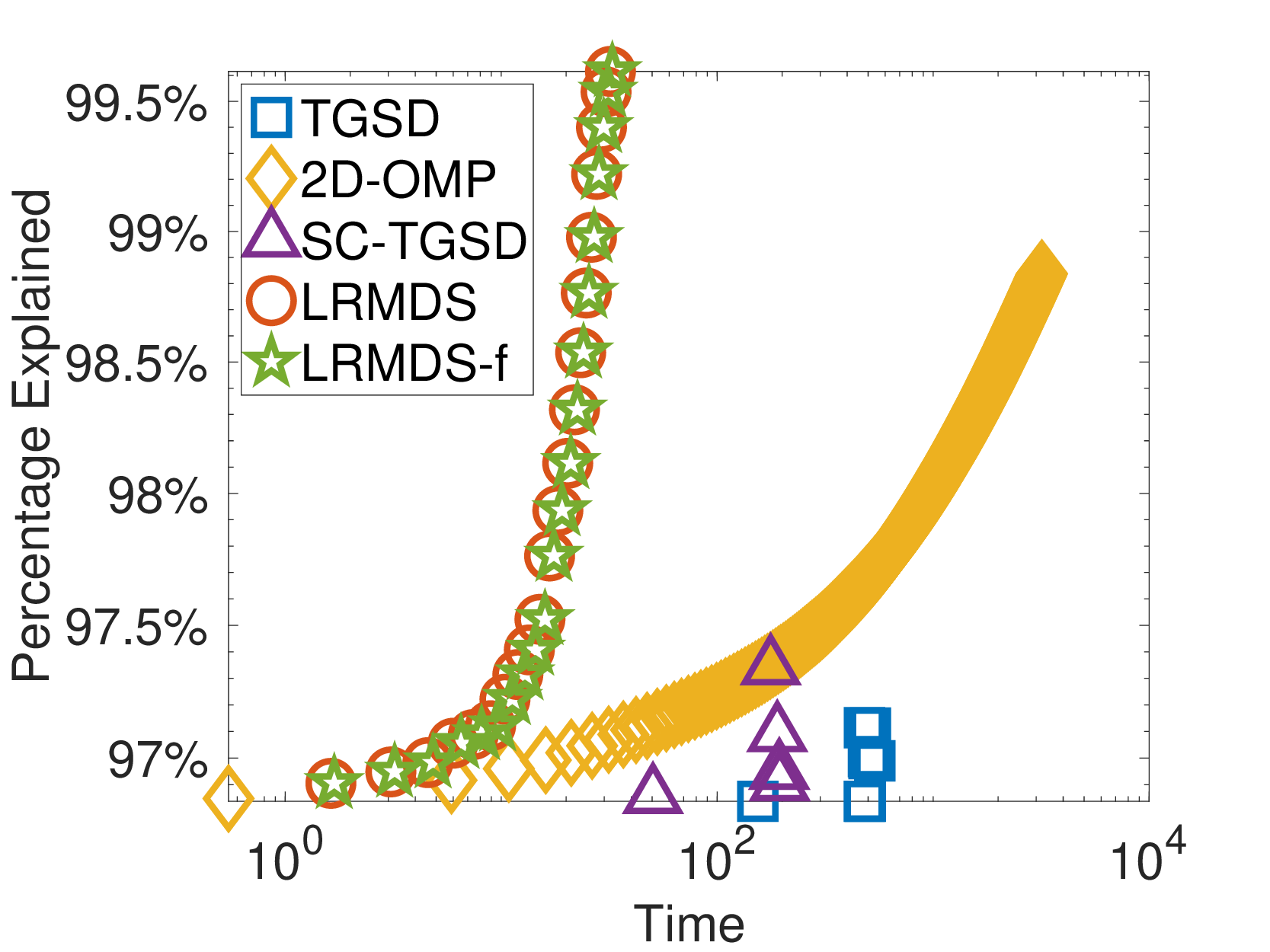}
        \label{fig:GWRS_rmse_vs_time}
    }
        \hspace{-0.2in}
            \subfigure [\scriptsize Ablation: RMSE vs Time]
    {
        \includegraphics[width=0.25\linewidth]{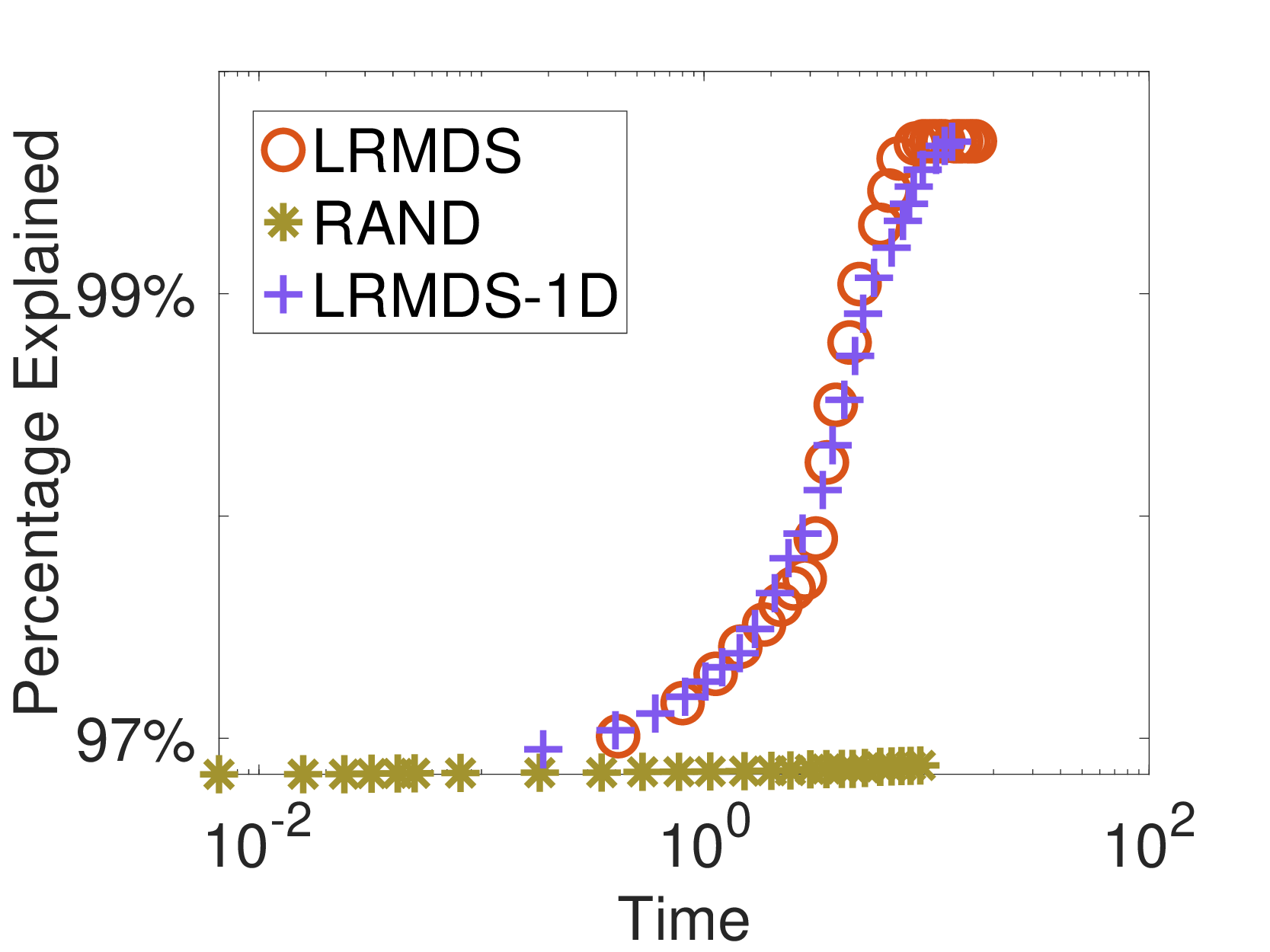}
        \label{fig:ab_rmse_vs_time}
    }
        \hspace{-0.2in}
    \subfigure [Twitch: RMSE v.s. Time]
    {
        \includegraphics[width=0.25\linewidth]{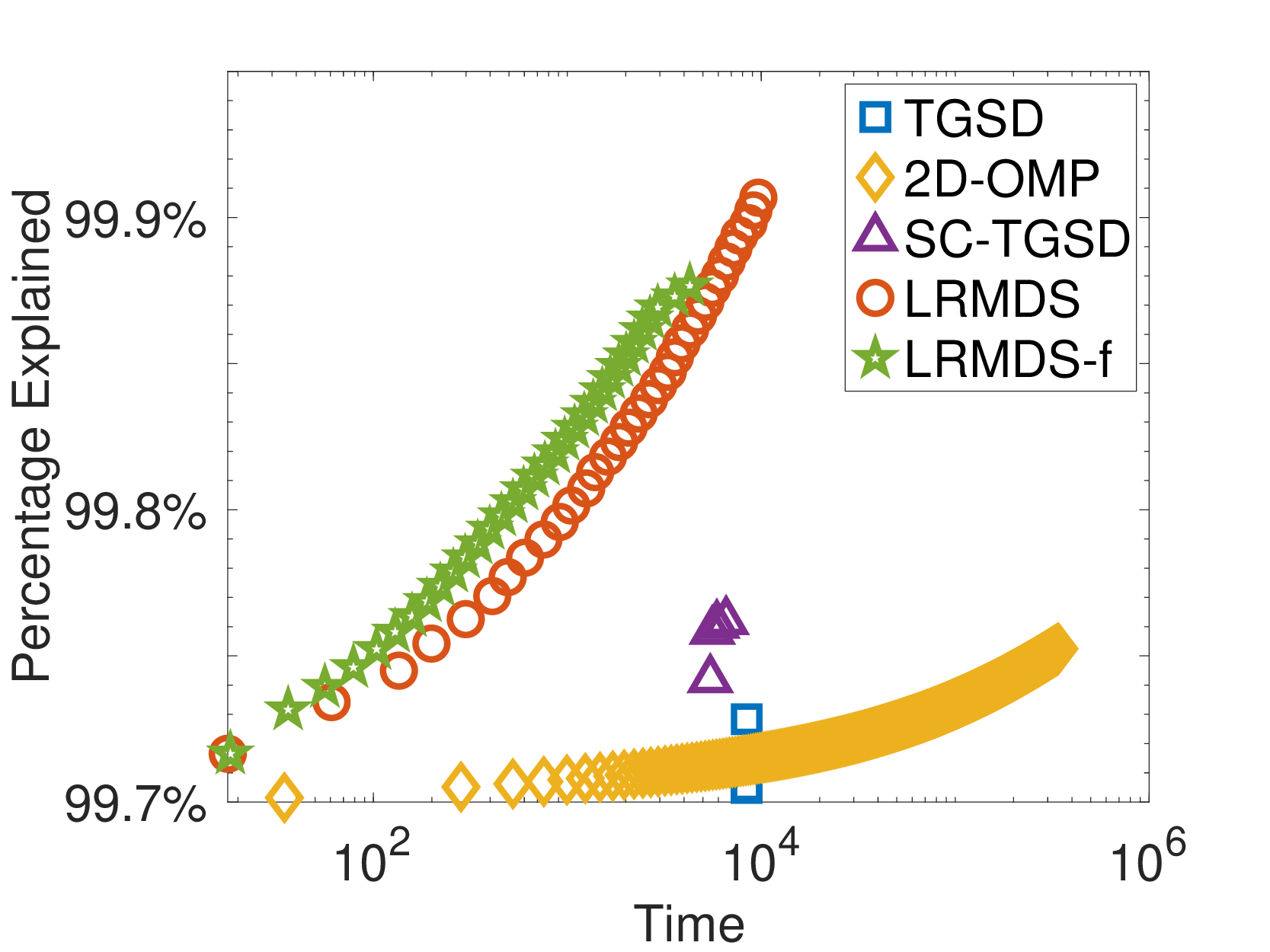}
        \label{fig:twitch_rmse_vs_time}
    }
    \hspace{-0.2in}
    \subfigure [Road: RMSE vs Time]
    {
        \includegraphics[width=0.25\linewidth]{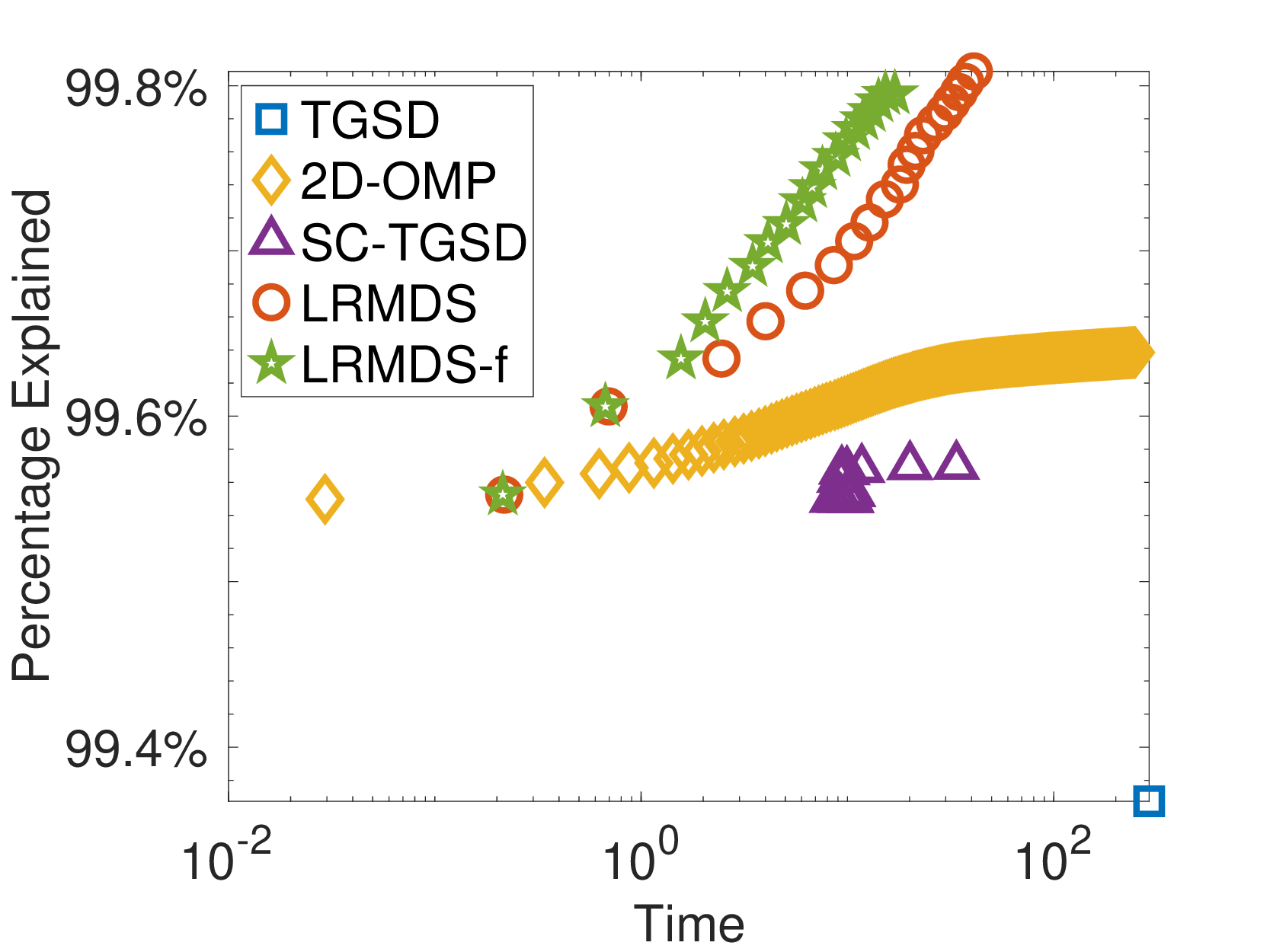}
        \label{fig:road_rmse_vs_time}
    }
    \hspace{-0.2in}
    \subfigure [Wiki: RMSE vs Time]
    {
        \includegraphics[width=0.25\linewidth]{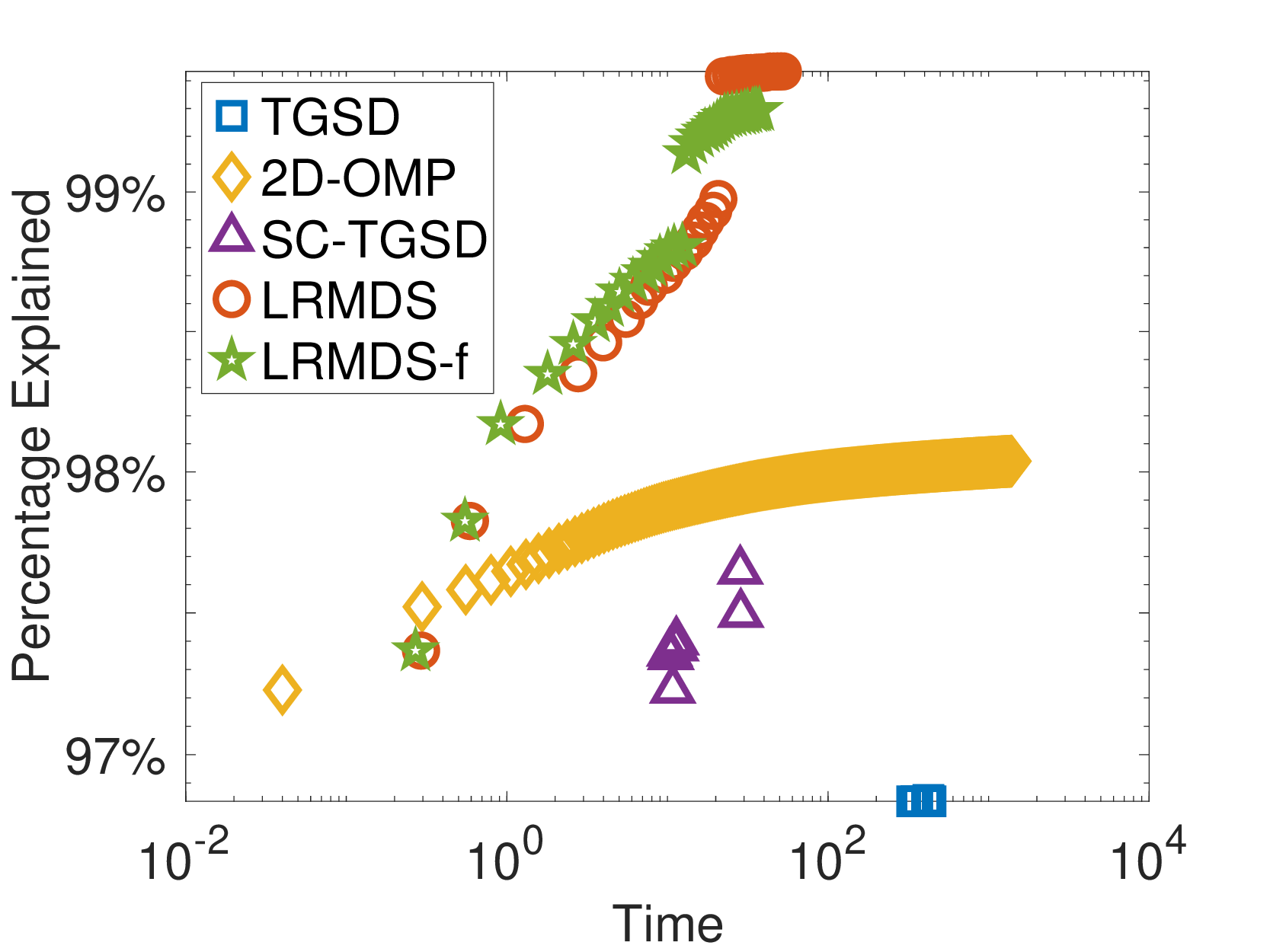}
        \label{fig:wiki_rmse_vs_time}
    }
    \hspace{-0.2in}
    \subfigure [Covid: RMSE vs Time]
    {
        \includegraphics[width=0.25\linewidth]{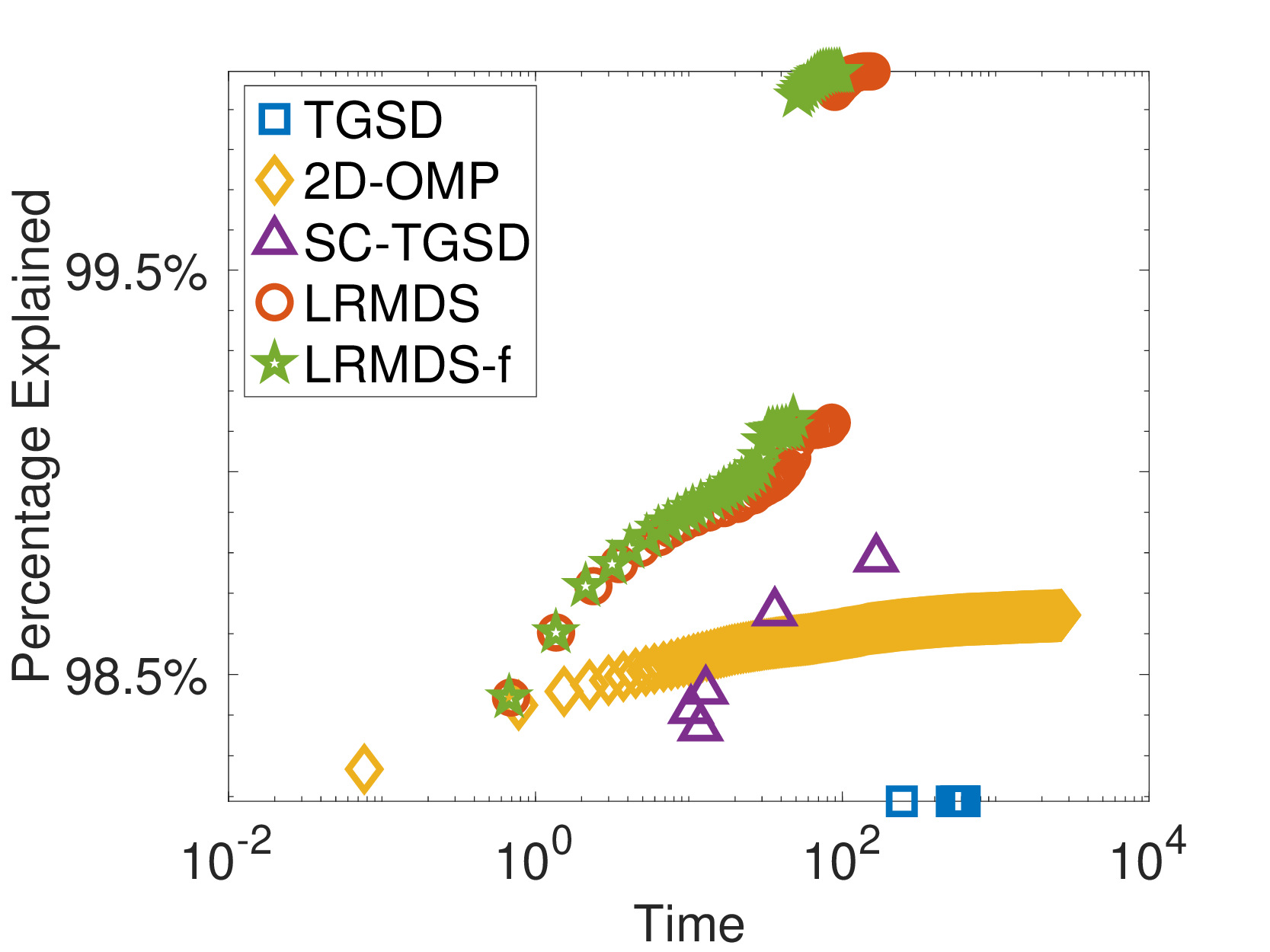}
        \label{fig:covid_rmse_vs_time}
    }

    \caption{ \footnotesize Comparison of competing techniques' percentage of input explained per second of run-time obtained on i) Synthetic data with varying dictionaries \subref{fig:GR_rmse_vs_time}-\subref{fig:GWRS_rmse_vs_time}; ii) Synthetic data as part of the ablation study \subref{fig:ab_rmse_vs_time}; and the real-world  datasets \subref{fig:twitch_rmse_vs_time}-\subref{fig:covid_rmse_vs_time} }.
    \label{fig:rmse_vs_time}
\end{figure*}

We also performed a more detailed analysis of the effect of the number of selected atoms ($k$) and determined that it controls a trade-off between quality and runtime. Given a total target (optimal but unknown) number of atoms $k^*$, when $k << k^*$ our algorithm requires more iterations to converge (involving multiple sparse coding fits with increasing dictionaries). On the other hand, a larger $k$ similar to $k^*$ will result in fewer iterations, however, the algorithm may require more than $k^*$ atoms to achieve the same RMSE. For example, if two ``good'' atoms are similar, then the projections on them will also be similar and high-valued and they will both be selected although potentially redundant. We chose a middle-ground k for our experiments.

\section{Additional Experiments and Figures} \label{sec:sup_add_exp}
We next include additional experiments with accompanying figures which shed light on the performance of \ourmeth and competitors under various settings. First we cover performance under various noise settings in Sec.~\ref{sec:var-noise}, and with various dictionary sizes in Sec.~\ref{sec:var-dic}. Finally, we conclude with plots for synthetic, ablation, and real world experiments which show the percentage of the data matrix explained as a function of running time for various methods.  This percentage is calculated as $1-\frac{||X-X'||_F}{||X||_F}$  where $X'$ is the reconstruction of the input matrix produced by any of the competing techniques at a given time point.

 \subsection{Additional synthetic data experiments: Varying noise level} \label{sec:var-noise}
To examine the sensitivity of competitors to noise in the input signal we vary the magnitude of noise term in our synthetic generation resulting in SNR values in the range $[20, 10, 5, 2, 1]$. Results are presented in Fig.~\ref{fig:snr_trend}. Naturally, as SNR decreases (more noise added), the problem becomes more challenging for all methods resulting in increasing RMSE (Fig.\ref{fig:syn_snr_rmse}). \ourmeth has the highest relative advantage for high SNR and the gap between methods generally decreases for more noisy settings. SC-TGSD exhibits the worst RMSE across regimes since depending on the $\lambda$ parameter used for screening, redundant atoms may survive and similarly some ground truth atoms may be pruned.

In terms of running time (Fig.~\ref{fig:syn_snr_time}), all methods are relatively stable for varying SNR.
TGSD's computational complexity is only dependent on the data input (both dictionaries and signal matrix) which are constant in this experiment. The running times for 2D-OMP and \ourmeth are dependent on these factors plus the number of iterations needed to perform atom selection. We set the number of selected atoms to be constant across regimes resulting in a relatively stable running time. Finally, the fastest method among baselines is SC-TGSD as it thresholds only once and then performs TGSD with the resulting small sub-dictionaries. Its accuracy in terms of atom selection, and thus its RMSE suffers due to its simplicity. Both versions of our method are more than an order of magnitude faster than competitors.

\subsection{Varying dictionary sizes} \label{sec:var-dic}
In Fig.~\ref{fig:dict_types_sup}, we show additional synthetic experiments for various dictionary combinations as well as figures comparing the run-time versus number of atoms selected not shown in the main paper due to space limitations. The trends for smaller dictionaries (non-composite) are largely the same as those found in the larger composite dictionaries with \ourmeth and  \ourmeth-f finding a more accurate representation faster than baselines. Moreover, if one compares the performance across different dictionary combinations, it is clear that \ourmeth only benefits from the inclusion of more dictionaries. As more of the generative atoms becomes available, \ourmeth's improvement curve only becomes steeper as the accuracy improves but running time stays relatively stable. This is in contrast to the next best preforming method 2D-OMP, whose running time increases substantially with each dictionary added. Although \ourmeth and 2D-OMP both exhibit similar atom selection strategy, \ourmeth selects multiple atoms at a time, and \ourmeth's coefficients updates do not require a large matrix inversion involved in 2D-OMP.



\subsection{Representation Quality vs Runtime} \label{sec:time_vs_qual}

In Fig.\ref{fig:rmse_vs_time}, we plot results comparing the representation quality of competing techniques as a function of their runtime. It is important to note that these plots do not represent a new experiment but a re-contextualization of the  results from  the main paper (i.e., Fig.\ref{fig:dict_size_trends}, Fig.\ref{fig:real-test} and Fig.\ref{fig:dict_types_sup}). Both variants of \ourmeth perform well across all datasets, and dictionary settings, but as expected, \ourmeth-f generally produces slightly faster results at the cost of reduced representation quality. 2D-OMP obtains its first representation of the data faster than others, however this representation is of the poorest quality among competitors. In its first iteration, 2D-OMP only selects and solves for a single coefficient making it relatively fast. \ourmeth in contrast, always selects multiple coefficients leading to a higher running time but good initial quality. When more time is allowed for coding/dictionary selection, the representation quality of \ourmeth increases at a much faster rate than that of all competitors. Fig.~\ref{fig:ab_rmse_vs_time} shows a similar comparison for the simplified alternatives of our method from the ablation study. Initially, RAND is the first among competitors to achieve a representation as there is no projection step. However, as time increases its representation barely improves indicating that random atoms are not representative of the data. The other two methods perform similarly, but \ourmeth is able to obtain higher quality representations more quickly as time progresses.

\subsection{Generalization to multi-way data}
In this work, we have included analysis of 2-way (matrix) data and have not conducted experiments with higher order datasets (i.e. 3-way tensors and higher). However, we believe that our framework can be extended to 3D tensor data. Specifically, instead of calculating the 2D projection of the data on two dictionaries, we can find a 3D projection on three dictionaries (a projection tensor), select atoms and update the residual in a similar manner. We plan to investigate the dictionary selection in multi-way scenarios as part of our future research.

\end{document}